\documentclass[letterpaper,11pt]{article}

\usepackage[utf8]{inputenc} 
\usepackage[T1]{fontenc}    
\usepackage{hyperref}         
\usepackage{url}                  
\usepackage{booktabs}       
\usepackage{amsfonts}       
\usepackage{nicefrac}         
\usepackage{microtype}      
\usepackage{xcolor}            

\usepackage{amssymb}

\usepackage{enumitem}
\usepackage{amsmath}
\usepackage{amsthm}     
\usepackage{cleveref}     
\usepackage{graphicx}
\usepackage{wrapfig}
\usepackage{subcaption}
\usepackage{multirow}
\usepackage{algorithm}
\usepackage[noend]{algpseudocode}

\newcommand{\newparagraph}[1]{\vspace{1.5ex}\noindent\textbf{#1\hspace{1em}}}

\newtheorem{theorem}{Theorem}[section]

\newtheorem{cor}{Corollary}[section] 

\newcommand{\mc}{\mathcal}

\newcommand{\mbb}{\mathbb}

\newcommand{\gen}{\eps_{\mathrm{gen}}}

\usepackage[normalem]{ulem}

\newcommand\redout{\bgroup\markoverwith{\textcolor{red}{\rule[.5ex]{2pt}{0.4pt}}}\ULon}

\newcommand{\E}{\mbb{E}}

\newcommand{\eps}{\epsilon}

\newcommand{\Vast}{\bBigg@{5}}

\usepackage{cite}
\usepackage{tabulary}
\usepackage{geometry}
\hypersetup{
    colorlinks=true,
    linkcolor=red,
    filecolor=magenta,      
    urlcolor=cyan,
    citecolor=blue,
}
\usepackage[hyphenbreaks]{breakurl} 
\allowdisplaybreaks
\usepackage[symbol]{footmisc}

\date{}
\title{
\vspace{-41pt}
\vspace{-0.5cm}Repeated Random Sampling for \\Minimizing the Time-to-Accuracy of Learning
}

\author{\vspace{-2cm}}

\begin{document}
\maketitle

\begin{center}
\vspace{-1.25cm}
\begin{tabular}
{>{\centering}m{5.8cm} >{\centering}m{5.8cm}}
Patrik Okanovic\footnotemark
\\\texttt{\small ETH Z{\"u}rich}
&
Roger Waleffe\footnote[1]{}\footnotemark
\\\texttt{\small University of Wisconsin-Madison}
\tabularnewline \multicolumn{2}{c}{\vspace{-6pt}} \tabularnewline
Vasilis Mageirakos
\\\texttt{\small ETH Z{\"u}rich}
&
Konstantinos E. Nikolakakis
\\ \texttt{\small Yale} 
\tabularnewline \multicolumn{2}{c}{\vspace{-6pt}} \tabularnewline
Amin Karbasi
\\ \texttt{\small Google Research, Yale} 
&
Dionysis Kalogerias
\\ \texttt{\small Yale}
\tabularnewline \multicolumn{2}{c}{\vspace{-6pt}} \tabularnewline
Nezihe Merve G{\"u}rel
\\ \texttt{\small TU Delft}
&
Theodoros Rekatsinas\footnotemark
\\ \texttt{\small ETH Z{\"u}rich}
\end{tabular}
\vspace{2.5bp}
\end{center}
\footnotetext[1]{Equal contribution}
\footnotetext[2]{Corresponding author: waleffe@wisc.edu}
\footnotetext[3]{Currently at Apple}

\begin{abstract}
Methods for carefully selecting or generating a small set of training data to learn from, i.e., data pruning, coreset selection, and data distillation, have been shown to be effective in reducing the ever-increasing cost of training neural networks. Behind this success are rigorously designed strategies for identifying informative training examples out of large datasets. However, these strategies come with additional computational costs associated with subset selection or data distillation before training begins, and furthermore, many are shown to even under-perform random sampling in high data compression regimes. As such, many data pruning, coreset selection, or distillation methods may not reduce `time-to-accuracy', which has become a critical efficiency measure of training deep neural networks over large datasets. In this work, we revisit a powerful yet overlooked random sampling strategy to address these challenges and introduce an approach called \emph{Repeated Sampling of Random Subsets} (RSRS or RS2), where we randomly sample the subset of training data for each epoch of model training. We test RS2 against thirty state-of-the-art data pruning and data distillation methods across four datasets including ImageNet. Our results demonstrate that RS2 significantly reduces time-to-accuracy compared to existing techniques. For example, when training on ImageNet in the high-compression regime (using less than 10\% of the dataset each epoch), RS2 yields accuracy improvements up to 29\% compared to competing pruning methods while offering a runtime reduction of 7$\times$. Beyond the above meta-study, we provide a convergence analysis for RS2 and discuss its generalization capability. The primary goal of our work is to establish RS2 as a competitive baseline for future data selection or distillation techniques aimed at efficient training.

\end{abstract}
\vspace{-0.1cm}\textbf{Keywords:} data pruning, coreset selection, dataset distillation, random sampling, time-to-accuracy

\section{Introduction}
\label{sec:intro}

Deep learning is continually achieving impressive results, from image classification~\cite{he2016deep, dosovitskiy2020image} to speech recognition~\cite{chiu2018state} and natural language processing~\cite{brown2020language, Radford2019LanguageMA, openai2023gpt4}. Much of this success can be attributed to training large neural networks over datasets with millions or billions of examples~\cite{ILSVRC15, Gokaslan2019OpenWeb, brown2020language, radford2021learning}. However, these network and dataset sizes lead to model training that requires weeks or months and yields significant monetary and computational costs~\cite{mindermann2022prioritized, brown2020language}. Such costs nearly prohibit further model refinement through hyperparameter search or neural architecture search. As a result, there has been an arms race to minimize the required training time to reach a given accuracy, i.e., time-to-accuracy.

To reduce time-to-accuracy, recent works focus on decreasing the amount of training data used for model learning during each epoch. More specifically, given a large, labeled dataset, these works aim to maximize end-model accuracy and minimize runtime when training for multiple rounds, where training within each round is performed only on a small set of examples equal in size to a fraction $r$ of the full dataset. The set of examples used for training at each round can be either chosen once before learning begins or periodically recomputed between rounds based on model updates (e.g., as in ~\cite{mirzasoleiman2020coresets, killamsetty2021glister}). Existing methods in this framework span two main categories: 1) \textit{data pruning} methods which aim to reduce time-to-accuracy by selecting a subset of the most informative examples for training~\cite{welling2009herding, bachem2015coresets, bateni2014distributed, chen2010super, killamsetty2021glister, paul2021deep, mirzasoleiman2020coresets, sorscher2022beyond}; 2) \textit{dataset distillation} methods which aim to achieve the same by generating small sets of synthetic examples that summarize training instances from the full dataset~\cite{wang2018dataset, yu2023dataset}. We review data pruning and distillation methods in Section~\ref{sec:prelim}.
Briefly, these methods have demonstrated strong competitiveness in minimizing the time-to-accuracy of learning and influenced a large number of subsequent works, including ours. However, several challenges persist in the pursuit of minimizing time-to-accuracy using these methods. One major challenge lies in the time efficiency aspect, as there is a notable overhead associated with subset selection. For example, many methods require pretraining an auxiliary model on the full dataset for a few epochs in order to select the subset, a task which we find can take roughly 250 minutes on ImageNet (Figure~\ref{fig:imagenet-time-to-acc}) while training itself with $r=1$ and 10\% takes only 40 and 400 minutes respectively. Even when efforts are made to address this issue, score-based data pruning has fallen short in achieving high accuracy, particularly in the high compression regime (small $r$)~\cite{ayed2023data}, a practical regime enabling ML practitioners to efficiently perform tasks such as hyperparameter selection and architecture search.

With these challenges in mind, in this work, we revisit the random sampling baseline for subset selection and introduce an intuitive and powerful extension of this method for optimizing the time-to-accuracy of learning. Random sampling is widely accepted as a competitive baseline, particularly for high compression regimes~\cite{ayed2023data, sorscher2022beyond}, whose success lies in the ability to select representative data examples for training, thus preventing overfitting. Typically, a static subset of the complete dataset is sampled once before the learning begins~\cite{guo2022deepcore, park2022active}. 
However, we believe that a better way to utilize this baseline is by repeatedly sampling data instances at each epoch, as this allows the learner to explore more previously unseen examples throughout training. Random exploration has already proven advantageous for data pruning methods, particularly in the high compression regime~\cite{ayed2023data}, allowing them to calibrate for distribution shift (caused by discarded examples). Moreover, adversarial training has also experienced time-to-accuracy reduction with random exploration~\cite{kaufmann2022efficient}. Surprisingly, however, this method has yet to be studied or evaluated for standard training of deep neural networks. 
Motivated by this gap, we introduce Repeated Sampling of Random Subsets (RSRS = RS2), where we sample a subset uniformly at random for each training epoch. The main contribution of this paper is an in-depth analysis of RS2. We demonstrate that RS2 surpasses state-of-the-art data pruning and distillation methods in accuracy and runtime across a wide range of subset sizes. In addition, we show that it outperforms all prior methods in the high compression regime, thus, posing a strong benchmark to beat for minimizing the time-to-accuracy of learning. In Section~\ref{sec:method}, we provide a detailed explanation of RS2, and in Section~\ref{sec:theory} we discuss its convergence rate and generalization error.

We extensively evaluate the time-to-accuracy of RS2 and compare it against twenty-two proposed data pruning and eight data distillation methods from the literature (Section~\ref{sec:exp}). We find that RS2 outperforms existing methods with respect to runtime and accuracy across varying subset selection sizes and datasets, including CIFAR10, CIFAR100, ImageNet30, and ImageNet itself. For example, when training a ResNet-18 in the high-compression regime (with $r = 10\%$) on ImageNet, RS2 yields a model with 66\% accuracy, 11 points higher than the next-best method and only 3.5 points less than training with the entire dataset every round. Yet, RS2 reaches this accuracy $9\times$ faster than standard full-dataset training. With $r = 1\%$, RS2 still reaches 47\% accuracy, while the next-best method achieves only 18\%. Finally, we present an extension of RS2 beyond supervised learning by evaluating its performance on self-supervised pretraining of GPT2~\cite{Radford2019LanguageMA}, a setting for which existing works have yet to study. We again find that using RS2 can reduce training time by 3-10$\times$ while yielding models within 2 and 5 points of the full-dataset accuracy and perplexity respectively.

\section{Preliminaries}
\label{sec:prelim}
We first present a unified framework for the problem of reducing time-to-accuracy by training on less data each epoch and then review existing data pruning and distillation methods.

\newparagraph{Problem Statement}
Given a large, labeled dataset $S = \{\mathbf{x}_i, y_i\}_{i=1}^N$, where each training example consists of an input feature vector $\mathbf{x}_i$ and a given ground truth label $y_i$, our goal is to minimize runtime and maximize accuracy when training for $X$ rounds, with the training of each round performed on a set of examples $S'$ with size $|S'| = r \cdot |S|$ for $r \in (0, 1]$.

We highlight two important points: First, it is generally assumed that $X$ is chosen such that training proceeds for the same number of rounds as when training on the full dataset, otherwise the computational benefits are reduced (e.g., $r=50$\% with $X=200$ is the same amount of computation as $r=100$\% with $X=100$). Second, note that the subset $S'$ may be static or vary across rounds. In fact, some existing methods periodically recompute $S'$ between rounds~\cite{mirzasoleiman2020coresets, killamsetty2021glister}, while others do not~\cite{sorscher2022beyond}. Given the primary goal of minimizing time-to-accuracy, either choice is valid so long as the time to generate the subset $S'$ at each round is included in the overall runtime.

\newparagraph{Related Work}
To minimize time-to-accuracy, \textit{data pruning} methods attempt to find a subset (also called a \textit{coreset}) of informative examples $S' \subset S$ such that a model trained on $S'$ achieves similar accuracy to a model trained on $S$~\cite{guo2022deepcore, park2022active}. Numerous metrics have been proposed to quantify importance: Uncertainty based methods such as Least Confidence, Entropy, and Margin~\cite{sachdeva2021svp} assume examples with lower confidence will have higher impact on training. 
Loss and error based methods operate on a similar principle. Forgetting Events~\cite{toneva2018empirical}, GraNd, EL2N~\cite{paul2021deep}, and others~\cite{bachem2015coresets, munteanu2018coresets, dasgupta2019teaching, liu2021just} are examples of these methods. Other techniques for subset selection, such as CRAIG~\cite{mirzasoleiman2020coresets} and GradMatch~\cite{pmlr-v139-killamsetty21a}, focus on gradient matching, where the goal is to construct a subset of examples such that a weighted sum of the model gradients on the subset matches the overall gradient on the full dataset. A different class of methods focuses on feature geometry for data subset selection.  
A number of geometry-based methods have been proposed, such as Herding~\cite{welling2009herding, chen2010super}, K-Center Greedy~\cite{sener2018active}, and prototypes~\cite{sorscher2022beyond}. Additional data pruning algorithms attempt to find the training examples closest to the decision boundary (e.g., Adversarial Deepfool~\cite{ducoffe2018adversarial} and Contrastive Active Learning~\cite{liu2021just}), pose subset selection as a bilevel optimization problem (e.g., Retrieve~\cite{killamsetty2021retrieve} and Glister~\cite{killamsetty2021glister}), or connect subset selection to maximization of a submodular function (e.g., GraphCut, Facility Location, and Log Determinant~\cite{iyer2021submodular}). Active learning methods (which aim to minimize labeling cost by selecting an informative subset given a large unlabeled dataset), can also be used in the presence of a labeled dataset when the goal is to reduce time-to-accuracy. In fact, active learning has been shown to outperform existing methods in this setting~\cite{park2022active}. We refer the reader to recent surveys~\cite{guo2022deepcore} for more detailed descriptions on the above methods and for comparisons between them.

In contrast to data pruning which assumes $S'$ to be a subset of $S$, \textit{dataset distillation} methods use $S$ to generate a small set of synthetic examples $S'$ that aims to summarize $S$. Dataset distillation methods can be split into three groups: 1) Performance matching methods~\cite{wang2018dataset, deng2022remember} aim to optimize the synthetic examples in $S'$ such that models trained on $S'$ achieve the lowest loss on the original data $S$. 2) Parameter matching techniques~\cite{zhao2021DC, lee2022dataset, kim2022dataset, cazenavette2022dataset} focus instead on matching the parameters of a network trained on $S'$ with those of a network trained on $S$ by training both models for a number of steps. 3) Finally, the distribution matching approach~\cite{zhao2023dataset} to dataset distillation attempts to obtain synthetic examples in $S'$ such that the distribution of $S'$ matches the distribution of $S$. We refer the reader to~\cite{yu2023dataset} for a detailed survey on dataset distillation. 

\section{RS2: Repeated Random Sampling for Reducing Time-to-Accuracy}
\label{sec:method}
We introduce the RS2 algorithm and discuss how it yields efficient training by reducing the amount of training data used at each round of model learning.

\newparagraph{Repeated Sampling of Random Subsets (RS2)} As discussed in Section~\ref{sec:prelim}, we assume access to a large, labeled dataset $S$ and aim to minimize runtime while maximizing accuracy by training for $X$ rounds. 
\textit{We define RS2 as follows: Rather than traininng a model on the full dataset $S$ for $X$ rounds (i.e., epochs), train the same model for $X$ rounds with the training of each round performed on a subset $S'$ (of size $|S'| = r \cdot |S|$) sampled randomly from $S$ (Algorithm~\ref{alg:general_rs2}).} Training within a round proceeds exactly the same when using $S$ or $S'$, except that $S'$ has fewer examples (e.g., mini batches). We next describe RS2 in more detail and discuss two variants of the sampling strategy and the importance of appropriate learning rate scheduling.

\begin{wrapfigure}{R}{0.5\textwidth}
\vspace{-0.3in}
\begin{minipage}{0.5\textwidth}
    \begin{algorithm}[H]
    \caption{RS2 General Algorithm}
    \label{alg:general_rs2}
    \begin{algorithmic}[1]
    \small
    \Require Dataset $S = \{\mathbf{x}_i, y_i\}_{i=1}^N$, selection ratio $r\in(0, 1]$, batch size $b$, initial model $w^0$, $X$ rounds
    \State $T \gets \lceil N/b \rceil$
    \State $t \gets 1$
    \For{round $j=1$ to $X$}
    \State $S' \gets randomly\_sample\_subset(S,\; r)$
    \For{$k=1$ to $r\cdot T$}
    \State batch $m \gets S'[(k-1) \cdot b:k \cdot b]$
    \State $w^{t} \gets train\_on\_batch(w^{t-1},\; m)$
    \State $t \gets t+1$
    \EndFor
    \EndFor
    \Return $w^t$
    \end{algorithmic}
    \end{algorithm}
\end{minipage}
\vspace{-0.1in}
\end{wrapfigure}

\newparagraph{RS2 With Replacement} The simplest version of RS2 samples $S'$ with replacement across rounds---sampling can be stratified. This means that examples included in the subset of previous rounds are replaced in $S$ and eligible to be resampled when constructing $S'$ for the current round, i.e., $S'$ is always constructed by sampling uniformly from all examples in $S$.

\newparagraph{RS2 Without Replacement} A second variant of RS2 samples $S'$ without replacement across rounds. That is, examples in $S$ that have been included in the subset during previous rounds are not considered when sampling $S'$ for the current round. This continues until all examples from $S$ have been included in $S'$ at some round, at which point all examples are once again eligible for subset selection and the process repeats. Observe that sampling $S'$ without replacement across $X$ rounds is equivalent to training on the full dataset $S$ for $r\cdot X$ rounds (i.e., equivalent to \textit{early stopping} after $r\cdot X$ rounds on the full dataset assuming the same randomn seed). 
In other words, RS2 without replacement can be implemented as follows: Given a random permutation of the full dataset $S=\{\mathbf{x}_i, y_i\}_{i=1}^N$, we generate mini batches of size $b$ for training by first traversing the full dataset. That is, we select $b$ sequential examples (first mini batch), then the next $b$ examples (second mini batch), and so on. After iterating over the full dataset, we generate a new random permutation and repeat the procedure (see Appendix \ref{section_pseudo_app}, Algorithm~\ref{alg:minibatch_sgd_no_momentum}). The algorithm early stops after $r \cdot T \cdot X$ gradient updates (with $T = N / b$).

\newparagraph{RS2 Hyperparameters} For both RS2 variants, we assume that training proceeds using the same hyperparameters (e.g., batch size, optimizer, etc.) as those used when training on the full dataset with one exception: the learning rate schedule. The reason for this is that state-of-the-art training procedures often slowly decay the learning rate after each SGD step. We show in Figure~\ref{fig:learning_rate} the common cosine annealing learning rate schedule~\cite{loshchilov2016sgdr} used to train ResNets on CIFAR10. We include a vertical line showing the point of early stopping when running RS2 without replacement for $r=10$\%. Observe that data pruning with RS2 (or any method) leads to fewer SGD iterations (even when training for the same $X$ rounds). Thus, if the data pruning method uses the same learning rate schedule as when training on the full dataset, the learning rate may not decay to a sufficiently small value to achieve high end-model accuracy. We refer to this setting as \textit{naive early stopping}. Instead, we train both RS2 variants with the same kind of learning rate schedule as when training the full dataset (e.g., cosine annealing), but we decay the learning rate faster with the decay rate inversely proportional to the subset size $r$ (e.g., green line in Figure~\ref{fig:learning_rate}). This is standard across existing data pruning methods~\cite{guo2022deepcore}.

\begin{wrapfigure}{R}{0.4\textwidth}
\vspace{-0.1in}
\begin{minipage}{0.4\textwidth}
\centering
    \includegraphics[width=1.0\linewidth]{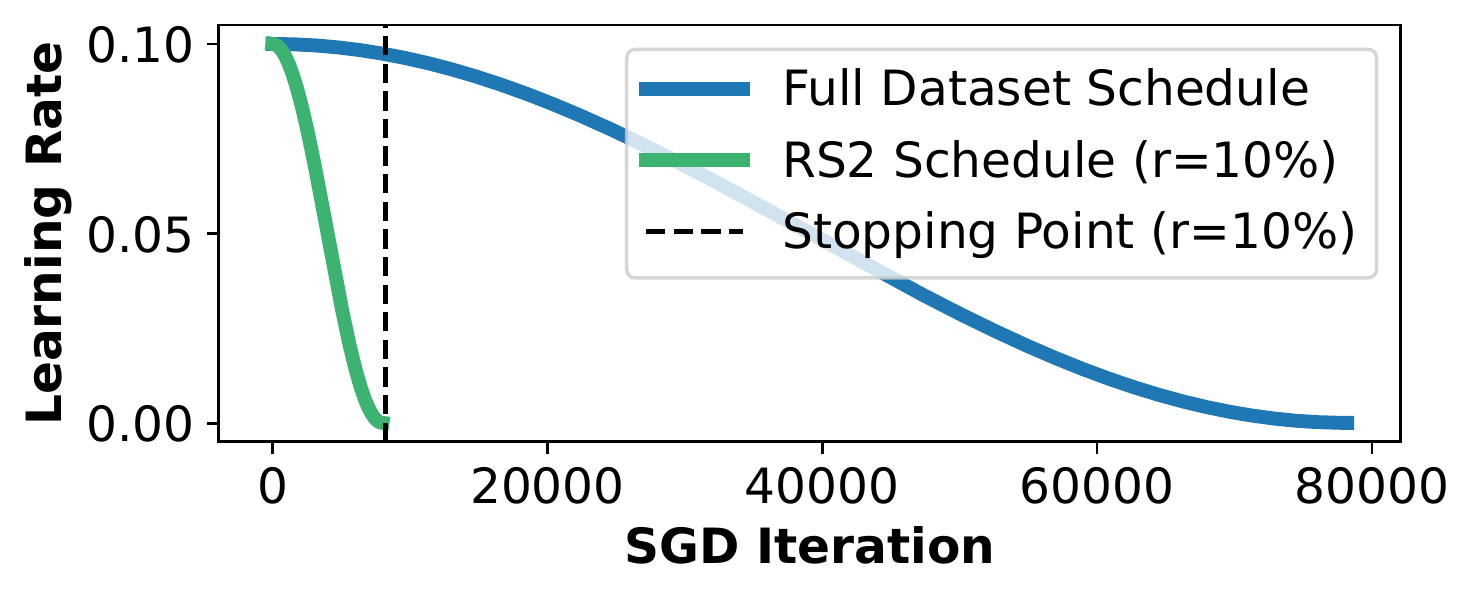}
    \caption{Learning rate schedules on CIFAR10 with and without data pruning.}
    \label{fig:learning_rate}
\end{minipage}
\end{wrapfigure}

Naive early stopping results in worse accuracy than RS2 with the modified learning rate schedule. For example, on CIFAR10 with $r=10\%$, the former achieves $83.9\%$ accuracy while RS2 with replacement reaches $89.7\%$ and RS2 without replacement (i.e., early stopping with faster learning rate decay) reaches $91.7\%$. While we generally see RS2 without replacement outperform RS2 with replacement in terms of accuracy, we do not claim that one variant is strictly better in this paper. We refer the reader to existing works which study this problem~\cite{haochen2019random, lu2022grab, NEURIPS2020_42299f06}. 

\section{Theoretical Analysis of RS2}
\label{sec:theory}
With time-to-accuracy in mind, we now study the convergence rate and generalization error of RS2.

\newparagraph{RS2 Convergence Rate}
We show that RS2 without replacement, under the assumptions described in Appendix~\ref{app:convergence}, converges at the same rate with respect to SGD iterations as standard training on the full dataset.
For the convergence rate analysis, we assume $train\_on\_batch(w^{t-1},\; m)$ in Algorithm~\ref{alg:general_rs2} uses Nesterov's accelerated gradient descent update~\cite{nesterov1983method} as shown in Algorithm~\ref{alg:minibatch_sgd}.
Following the analysis of recent works on the performance of accelerated mini-batch SGD~\cite{ghadimi2016accelerated}, we have:
\begin{cor}
\label{thm:convergence}
Suppose the loss $l(w)$ is nonconvex, has $\beta$-Lipschitz continuous gradients, and is bounded below. Let $g (w, \xi_t)$ at step $t$ represent the gradient estimate used when updating the model as in Algorithm~\ref{alg:minibatch_sgd} in the Appendix. Assume the gradient estimate satisfies $\mathbb{E} \left[ || g (w, \xi_t) - \nabla l (w)  ||^2 \right ] \leq \sigma ^2$, and $\mathbb{E} [g(w, \xi_t)] = \nabla l(w)$, where $\xi_t$ are random vectors whose distributions are supported on $\Xi_t \in \mathbb{R}^d$. 
With the previous assumptions, using a selection ratio $r \in (0,1]$ and mini batch of size $b$, RS2 produces an iterate $w$ after $X$ rounds, with $rT$ batches per round, such that:
\begin{equation}
\small
    \mathbb{E} \left[|| \nabla l(w) ||^2 \right ] \leq \mathcal{O} \left( \frac{\beta (l(w^0) - l(w^*))}{r \cdot T \cdot X} + \frac{\sigma\sqrt{\beta(l(w^0) - l(w^*))}}{\sqrt{b \cdot r \cdot T \cdot X}} \right).
\end{equation}
\end{cor}

We discuss the assumptions and prove Corollary~\ref{thm:convergence} in Appendix~\ref{app:convergence}. 
We find that the convergence rate of RS2 compared to the full dataset convergence rate~\cite{ghadimi2016accelerated} has a scaling factor $r$ in front of the total number of iterates, while the bound remains consistent with respect to all the other parameters; With $r=1$ we recover the results from previous work~\cite{ghadimi2016accelerated}. When $r<1$ the gradient bound after $X$ rounds increases compared to training with the full dataset for $X$ rounds, but this is intuitive as each round contains fewer mini batches ($rT$ with $r<1$ instead of $T$). If RS2 with $r < 1$ is instead allowed to train for more rounds, specifically $X_{new} = \frac{X}{r}$, then both training RS2 for $X_{new}$ rounds and training on the full dataset for $X$ rounds result in the same number of mini-batch iterations ($TX$). In this case, the gradient is bounded by the same value, implying that RS2 and training on the full dataset converge with respect to \textit{mini-batch iterations} at the same rate. Overall, \Cref{thm:convergence} ties the convergence behavior of RS2 directly to the amount of pruning $r$ and to training on the full dataset.

\newparagraph{RS2 Generalization Error} 
We now provide an upper bound on the generalization error of RS2. For this analysis, we relax the update rule from Algorithm~\ref{alg:general_rs2} to a standard gradient update without momentum. Recall that as $r$ decreases, RS2 results in a smaller total number of gradient steps after $X$ rounds compared to $r=1$. While this may lead to an increase in optimization error, the generalization error is expected to be smaller than that of the full dataset schedule (shorter training time gives a smaller generalization error). This phenomenon has been characterized rigorously in prior works~\cite{hardt2016train} for vanilla SGD with batch size $b=1$, however it does not directly apply to larger mini batch sizes and general selection rules. As such, we show an extension of known generalization error bounds that also holds for RS2 with mini batch size $b$. Before we proceed, we first introduce some notation for brevity. We define the training dataset $S\triangleq (z_1,z_2,\ldots,z_N)$, for which $z_i\triangleq (\mathbf{x}_i ,y_i)$ for $i\in \{1,\ldots , n\}$ and the (empirical) loss $l(w)\triangleq \frac{1}{N}\sum^N_{i=1} f(w , z_i ) $, where $f: \mbb{R}^d \times \mc{Z} \rightarrow \mbb{R}^+$. Let $z_1,z_2,\ldots,z_N,z$ be i.i.d random variables with respect to an unknown distribution $\mc{D}$. Then for any stochastic algorithm $A$ with input $S$, and output $A(S)$, the generalization error $\epsilon_{\mathrm{gen}}$ is defined as the difference between the empirical and population loss~\cite{hardt2016train}:
\vspace{-0.25em}
\begin{align}
\small
   \gen (f,\mc{D}, A) &\triangleq  \E_{S,A,z} \big[  f(A (S), z )\big] -  \E_{S,A} \Big[\frac{1}{N}\sum^N_{i=1} f(A (S), z_i ) \Big]. \label{eq"gen_def}
\end{align} 
We now proceed with an upper bound on the generalization error of RS2. The next result follows from recent work~\cite{nikolakakis2023select}, and applies to RS2 with no momentum and any batch size $b$.
\begin{theorem}[Generalization error of standard gradient RS2,~\cite{nikolakakis2023select} Theorem 8]\label{thm:generalization}
    Let the function $f$ be nonconvex, $L_f$-Lipschitz and $\beta_f$-smooth. Then the generalization error of the standard gradient RS2 algorithm with a decreasing step-size $\eta_t\leq C/t$ (for $C<1/\beta_f$), is bounded as: 
\begin{align}
\small
      |\gen (f, \mc{D} , \mathrm{RS2} ) |\leq \frac{1}{N} \cdot 2 C e^{C\beta_f} L_f^2 {(r \cdot T \cdot X)}^{C\beta_f }  \min\Big\{  1+ \frac{1}{C \beta_f} , \log(e\cdot r\cdot T\cdot X) \Big\}.\label{eq:gen_bound}
\end{align}
\end{theorem} 
\vspace{-0.25em}
The proof of Theorem \ref{thm:generalization} follows from very recent work~\cite{nikolakakis2023select} (Appendix~\ref{app:generalization_error}). Observe that, as above for the convergence rate, when comparing the generalization error of RS2 to that of the full dataset~\cite{nikolakakis2023select}, the dependence on all parameters remains the same except that the number of iterates for RS2 is scaled by $r$. 
The generalization error of RS2 relies on the fact that the batch at each iteration is selected \emph{non-adaptively and in a data-independent fashion}. However, most data pruning methods adopt data-dependent strategies, and deriving the relationship between the generalization error of RS2 and that of any arbitrary data-dependent mechanism can be challenging. A recent finding on the generalization of data-dependent pruning, particularly in the small $r$ regime, shows data-dependency may worsen the generalization~\cite{ayed2023data} due to the distribution shift caused by discarding a large number of data examples during training, thus leading to inferior performance compared to random sampling.  
\section{Evaluation}
\label{sec:exp}
We evaluate RS2 on four common benchmarks for supervised learning and compare against existing data pruning and distillation methods. We show that:
\begin{enumerate}[leftmargin=*, itemsep=-0.1ex, topsep=0.5ex]
    \item Across a wide range of pruning ratios, RS2 reaches higher accuracy than all existing methods.
    \item For a given pruning ratio, RS2 also trains the fastest, and thus has the fastest time-to-accuracy.
    \item In the presence of noisy labels, RS2 is the most robust data pruning method; It achieves the highest end-model accuracy and lowest relative drop in performance vs. training on the clean dataset.
\end{enumerate}
We also show that RS2 can extend beyond conventional supervised learning and reduce the self-supervised pretraining cost of GPT models with little model quality loss.

\subsection{Experimental Setup}
\label{subsec:setup}
We first discuss the setup used in the experiments. More details can be found in Appendix~\ref{appendix:setup}.

\newparagraph{Datasets, Models, and Metrics}
We benchmark RS2 against baseline methods using CIFAR10~\cite{krizhevsky2009learning}, CIFAR100~\cite{krizhevsky2009learning}, ImageNet30 (a subset of ImageNet)~\cite{hendrycks2019using}, and ImageNet~\cite{ILSVRC15} itself. We train ResNet models~\cite{he2016deep} representative of modern state-of-the-art convolutional neural networks. Beyond the standard supervised setting, we also evaluate RS2 when training GPT2~\cite{Radford2019LanguageMA} on the OpenWebText dataset~\cite{Gokaslan2019OpenWeb}. In this setting, we measure zero-shot final word prediction accuracy on the LAMBADA dataset~\cite{paperno2016lambada} and perplexity on LAMBADA and WikiText103~\cite{merity2016pointer}. For all experiments we measure subset selection overhead, overall training time (including the total time for subset selection across all rounds and the total training time on selected subsets), and end-model accuracy.

\newparagraph{Baselines}
We compare RS2 against 22 data pruning methods and eight data distillation methods from the literature. A full list and their abbreviations can be found in Appendix~\ref{appendix:setup}. We include data pruning methods from different classes, including uncertainty-based, loss-based, gradient matching, and geometry-based methods. All baselines are used for the smallest dataset (i.e., CIFAR10), but some methods do not scale to larger datasets (e.g., ImageNet). We utilize existing open source implementations and results of these methods where applicable~\cite{park2022active, guo2022deepcore} and implement RS2 within the same code for equal comparison. We compare to existing methods which sample a static subset once before learning begins, and to methods which repeatedly sample the subset each epoch.

\newparagraph{Training Details}
We use standard hyperparameters for each dataset from prior works known to achieve high accuracy. We use the same hyperparameters for all methods where applicable (e.g., batch size, initial learning rate, number of training rounds, etc.). Hyperparameters individual to each baseline are set based on the best known values from prior works~\cite{guo2022deepcore}. Full details are provided in Appendix~\ref{appendix:setup}.
We run image classification experiments on a university cluster with job isolation and NVIDIA RTX 3090 GPUs. For GPT2 training we utilize a machine with NVIDIA V100 GPUs.

\subsection{End-to-End Data Pruning Experiments}
\label{subsec:exp_e2e}
We discuss end-to-end comparisons of RS2 with data pruning baselines on supervised learning benchmarks. Results are shown in Figures~\ref{fig:acc}-\ref{fig:time-to-acc} and Table~\ref{tab:per_epoch}.

\newparagraph{Accuracy}
In Figure~\ref{fig:acc} we show the end-model accuracy of RS2 compared to existing methods on CIFAR10 and ImageNet for varying selection ratios. We use the combined baseline methods and setting from recent studies~\cite{guo2022deepcore, park2022active} together with newer prototype-based data pruning methods~\cite{sorscher2022beyond}. We include the exact accuracies for these Figures in Tables~\ref{tab:cifar10_acc} and~\ref{tab:imagenet_acc} in Appendix~\ref{appendix:results}. Here, we follow the setting proposed by these works: for all baselines we sample a static subset once before training starts. The repeated sampling of RS2 leads to accuracy improvements of at least 7\% in the high compression regime ($r\leq 10\%)$. For example, on CIFAR10 with 5\% of the training data each epoch, RS2 without replacement achieves 87.1\% accuracy while the next closest baseline reaches just 65.7\%. 
Similar results hold on CIFAR100 and ImageNet30 (Appendix Table~\ref{tab:acc_cfiar100_imagenet30}). Figure~\ref{fig:imagenet_acc} shows that RS2 also outperforms existing methods for the much larger ImageNet dataset. 
For example, RS2 end-model accuracy with $r=1$\% is 46.96\% while the next closest baseline trains to only 18.1\%. Moreover, the end-model accuracy of RS2 is actually on par with the training on the full dataset for non-trivial selection ratios (e.g., $r=10$\%), offering a potential practical solution to reduce the cost of training in some applications (e.g., neural architecture search) (see also runtime reductions below). 

\begin{figure}
  \centering
  \begin{subfigure}[t]{0.48\textwidth}
     \centering
     \includegraphics[width=\textwidth]{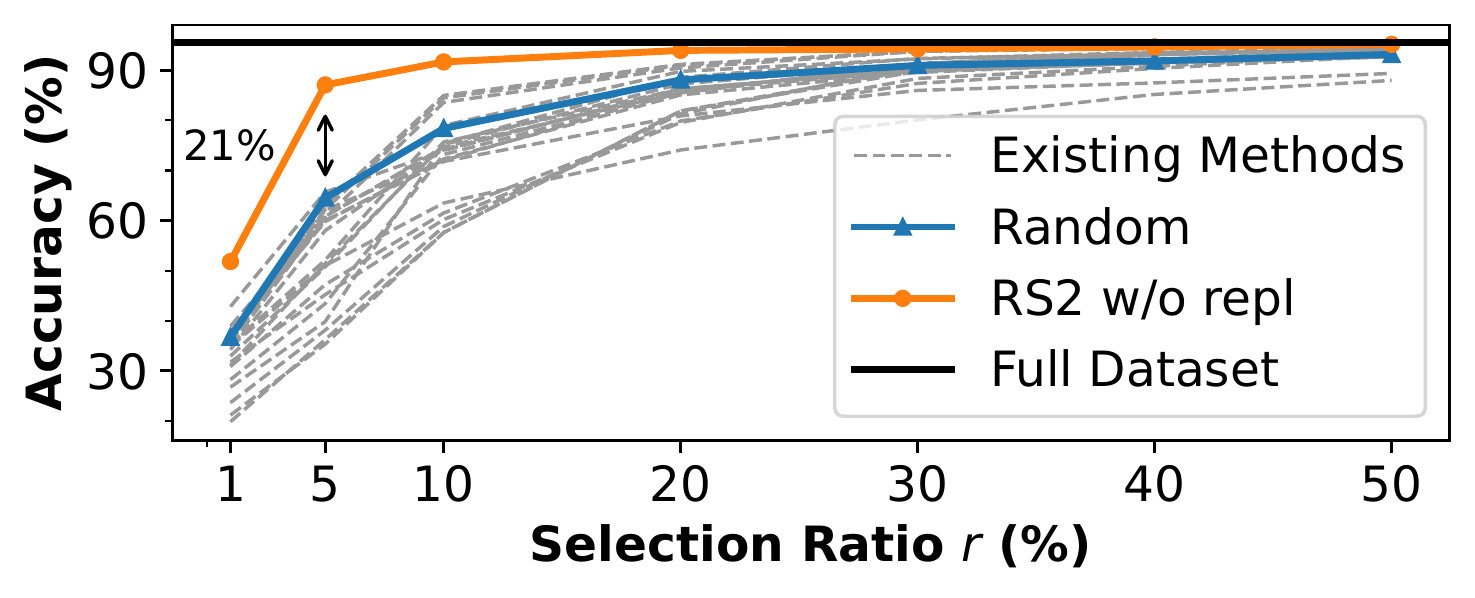}
     \caption{CIFAR10 accuracy}
     \label{fig:cifar10_acc}
  \end{subfigure}
  \hfill
  \begin{subfigure}[t]{0.48\textwidth}
     \centering
     \includegraphics[width=\textwidth]{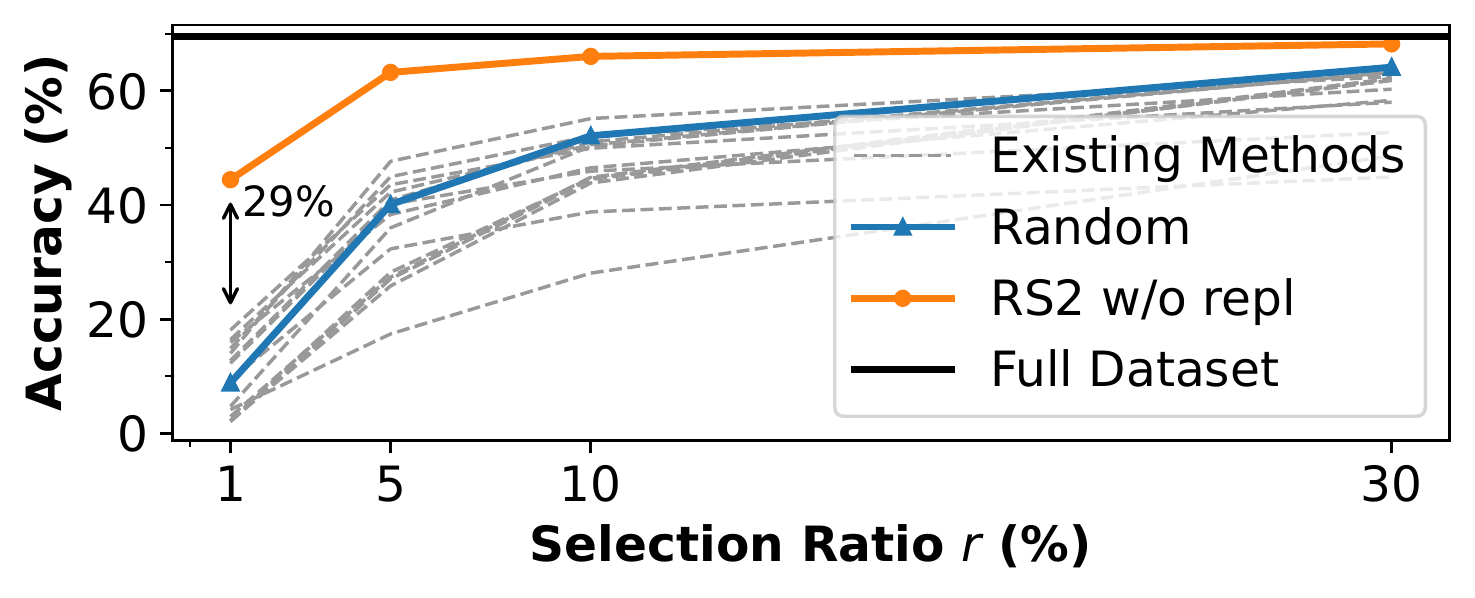}
     \caption{ImageNet accuracy}
     \label{fig:imagenet_acc}
  \end{subfigure}
  \caption{Accuracy achieved by data pruning methods when training ResNet-18 on CIFAR10 and ImageNet. Repeated Sampling of Random Subsets (RS2) outperforms existing methods.}
  \vspace{-0.20in}
  \label{fig:acc}
\end{figure}

Next, we extend baselines to also perform repeated sampling. Our goal is to examine if the prior observations are attributed only to the fact that RS2 performs repeated sampling while the above baselines do not. To do so, we take all baseline methods and where applicable modify them as follows: For a given method $M$ 1) if $M$ computes a numerical example importance for each training instance, we view these values as a probability distribution over the examples and \textit{resample} subsets according to this distribution after each round (called $M$-RS) and 2) if $M$ computes example importance based on model outputs or gradients, we \textit{recompute} example importance after each round using the current model and then choose the most important examples as the subset for the next round (called $M$-RC). Note that the latter set of methods (RC) are unlikely to be able to improve the efficiency of training as they require computing the model forward pass for every example between each round, but we include their accuracy for completeness. Results on CIFAR10 (for computational considerations we do not run these methods on ImageNet) are shown in Table~\ref{tab:per_epoch}. While updating the subset for existing methods each round improves their accuracy, RS2 still reaches the highest end-model accuracy for $r \leq 10$\%. This experiment highlights the importance of \textit{random} sampling, while the results above highlight the importance of \textit{repeated} sampling (e.g., Random vs RS2 in Figure~\ref{fig:acc}). 

\begin{table}[t]
\scriptsize
\caption{Accuracy achieved by data pruning methods with per-round sampling when training ResNet-18 on CIFAR10. The training subset is update for all methods after each round, either by resampling from a static example importance distribution (RS, left) or by recomputing example importance based on model updates (RC, right). Repeated Sampling of Random Subsets (RS2) outperforms repeated sampling based on example importance. Best method bolded; Next best underlined.}
\vspace{4pt}
\label{tab:per_epoch}
\centering
{
\resizebox{0.49\textwidth}{!}{
\begin{tabular}{llll}
\toprule
Selection Ratio ($r$) & 5\% & 10\% & 30\% \\
\midrule
CD-RS & - & - & - \\
Herding-RS & - & - & -  \\
K-Center Greedy-RS & - & - & - \\
Least Confidence-RS & 67.6$\pm$5.1 & 83.4$\pm$4.9 & 93.7$\pm$0.4 \\
Entropy-RS & 85.2$\pm$0.9 & 89.8$\pm$0.4 & \underline{94.4$\pm$0.3} \\
Margin-RS & 84.3$\pm$2.7 & \underline{90.4$\pm$1.0} & \underline{94.4$\pm$0.2} \\
Forgetting-RS & 81.9$\pm$3.1 & 88.3$\pm$2.4 & 94.0$\pm$0.1 \\
GraNd-RS & 86.2$\pm$2.1 & 90.1$\pm$0.9 & \textbf{94.5$\pm$0.1} \\
CAL-RS & 81.1$\pm$3.0 & 86.6$\pm$0.7 & 93.3$\pm$0.1 \\
Craig-RS & \underline{86.7$\pm$0.8} & 89.8$\pm$0.2 & 94.3$\pm$0.1 \\
Glister-RS & - & - & - \\
SP-Easy-RS & 84.0$\pm$4.3 & 88.4$\pm$0.1 & 93.6$\pm$0.3 \\
\midrule
RS2 w/o repl & \textbf{87.1$\pm$0.8} & \textbf{91.7$\pm$0.5} &  94.3$\pm$0.2 \\
\bottomrule
\end{tabular}
}
}
{
\resizebox{0.49\textwidth}{!}{
\begin{tabular}{llll}
\toprule
Selection Ratio ($r$) & 5\% & 10\% & 30\% \\
\midrule
CD-RC & 75.2$\pm$2.2 & \underline{83.1$\pm$0.7} & 87.5$\pm$0.2 \\
Herding-RC & 30.1$\pm$2.6 & 40.6$\pm$8.4 & 81.0$\pm$0.9 \\
K-Center Greedy-RC & 78.1$\pm$1.5 & 82.3$\pm$0.5 & 86.3$\pm$0.4\\
Least Confidence-RC & 44.8$\pm$12 & 76.7$\pm$3.9 &  \underline{88.3$\pm$0.3} \\
Entropy-RC & 41.4$\pm$6.9 & 78.4$\pm$2.9 &  86.9$\pm$0.1 \\
Margin-RC & \underline{79.7$\pm$1.4} & 82.8$\pm$1.4 & 86.8$\pm$0.2 \\
Forgetting-RC & 28.7$\pm$0.8 & 40.7$\pm$6.5 & 78.8$\pm$4.3 \\
GraNd-RC & 15.5$\pm$1.8 & 24.1$\pm$6.0 & 75.2$\pm$5.0 \\
CAL-RC & 66.7$\pm$1.7 & 74.5$\pm$0.8 & 84.8$\pm$0.4 \\
Craig-RC &  70.3$\pm$13 & 80.3$\pm$0.8 & 85.5$\pm$0.3 \\
Glister-RC & 72.5$\pm$0.6 & 81.4$\pm$0.7 & 86.6$\pm$0.5 \\
SP-Easy-RC & - & - & - \\
\midrule
RS2 w/o repl & \textbf{87.1$\pm$0.8} & \textbf{91.7$\pm$0.5} &  \textbf{94.3$\pm$0.2} \\
\bottomrule
\end{tabular}
}
}
\vspace{-0.15in}
\end{table}

\newparagraph{Training Time}
We now study the training time of RS2 compared to existing methods on CIFAR10 and ImageNet. In particular, we focus on time-to-accuracy to quantify efficient training. As runtime measurements have generally not been reported in the literature, we train all methods from scratch on NVIDIA 3090 GPUs for these experiments. We use all baseline methods from Figure~\ref{fig:cifar10_acc} and Figure~\ref{fig:imagenet_acc} for each dataset, respectively, which do not give GPU out-of-memory. We show the time-to-accuracy on CIFAR10 in Figure~\ref{fig:cifar10-time-to_acc} and on ImageNet in Figure~\ref{fig:imagenet-time-to-acc} using $r=10$\% for both datasets. 
We report the total time for subset selection on CIFAR10 for all methods in Appendix Table~\ref{tab:cifar10_subset_selection} and for baselines which utilize per-round sampling in Appendix Table~\ref{tab:cifar10_per_epoch_subset_selection}. We also include the time-to-accuracy measurements on CIFAR10 and ImageNet for different pruning ratios in Appendix~\ref{appendix:results}.

Figures~\ref{fig:cifar10-time-to_acc} and \ref{fig:imagenet-time-to-acc} show that RS2 provides the fastest time-to-accuracy when compared to previous data pruning methods. Note that the repeated subset selection in RS2 leads to negligible overhead compared to training on a static random subset (Figure~\ref{fig:time-to-acc}) and to the total training time: For example, the total subset selection time for RS2 on CIFAR10 with $r=10\%$ is less than one second, yet the total runtime is 750 seconds. Existing methods, however, are primarily limited by the fact that they require \textit{pretraining} an auxiliary model on the full dataset for a few epochs in order to rank example importance. 
For example, on ImageNet the fastest baseline begins training after 250 minutes, yet training itself only requires 400 minutes. With $r=1$\%, the training time drops to just 40 minutes; in this case the 250 minute overhead implies the fastest baseline is over 7$\times$ slower than RS2 (250+40=290 vs 40). Even if the pretraining overhead is amortized by fixing the subset for the remaining rounds, or by resampling from the importance distribution after each round, as in our `RS' baseline methods in Table~\ref{tab:per_epoch}, the initial overhead of these methods is still orders of magnitude higher than the total overhead of RS2 across all rounds (e.g., Table~\ref{tab:cifar10_subset_selection}-\ref{tab:cifar10_per_epoch_subset_selection}). Moreover, Figures~\ref{fig:cifar10-time-to_acc} and \ref{fig:imagenet-time-to-acc} highlight the practical potential of RS2 to reduce the computational cost of training high-accuracy models: For CIFAR10, RS2 reaches 91.7\% accuracy 4.3$\times$ faster than standard training on the full dataset, while for ImageNet, RS2 reaches 66\% accuracy 9$\times$ faster than standard full dataset training.

\begin{figure}
  \vspace{0.2in}
  \centering
  \begin{subfigure}[t]{0.48\textwidth}
     \centering
     \includegraphics[width=\textwidth]{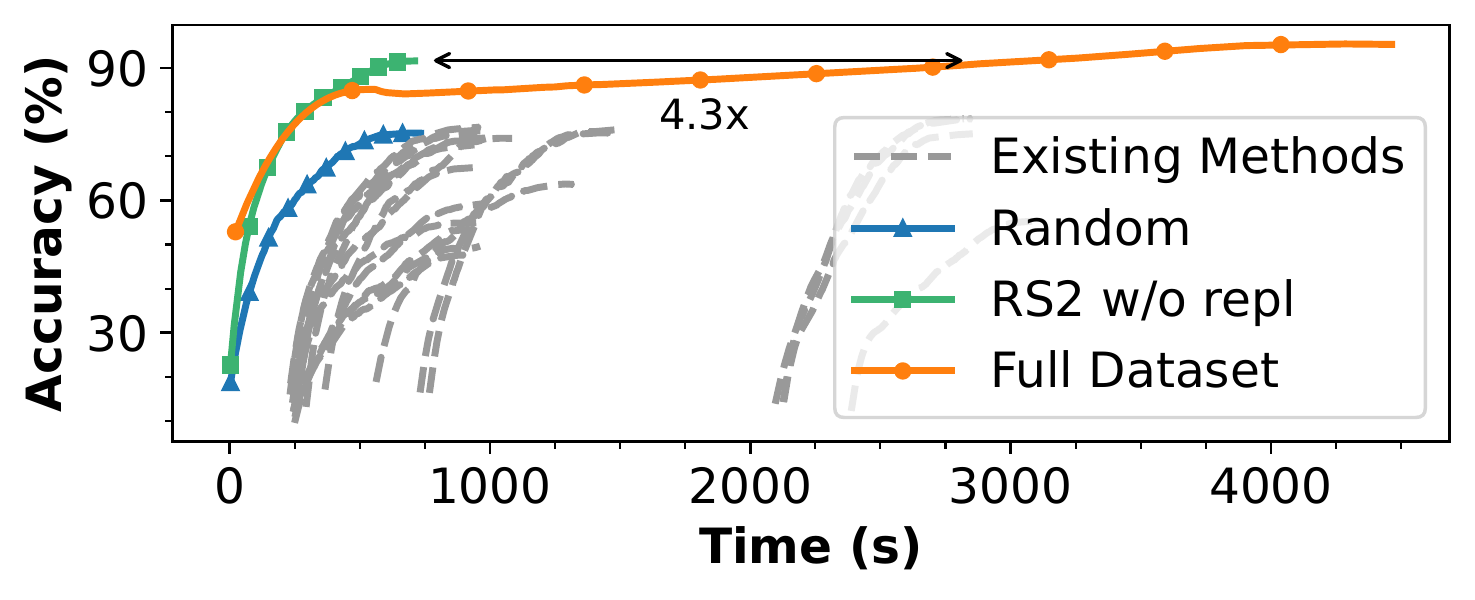}
     \caption{CIFAR10 time-to-accuracy}
     \label{fig:cifar10-time-to_acc}
  \end{subfigure}
  \hfill
  \begin{subfigure}[t]{0.48\textwidth}
     \centering
     \includegraphics[width=\textwidth]{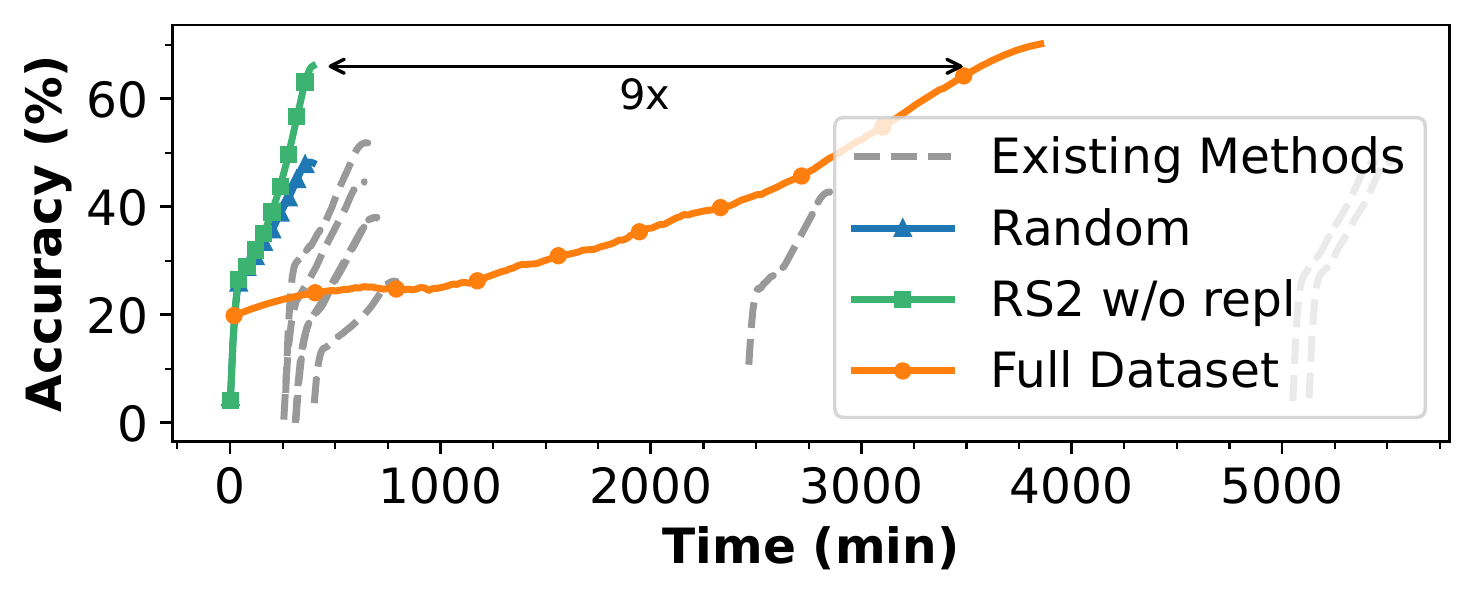}
     \caption{ImageNet time-to-accuracy}
     \label{fig:imagenet-time-to-acc}
  \end{subfigure}
  \caption{Time-to-accuracy for repeated sampling of random subsets (RS2) vs. existing data pruning methods, 
  a static random subset, and standard training on the full dataset. We use a selection ratio of $r=10$\%. RS2 is both the fastest and highest accuracy data pruning method.}
  \label{fig:time-to-acc}
  \vspace{-0.15in}
\end{figure}

\newparagraph{Takeaway}
The above results show that RS2 outperforms existing data pruning methods with respect to end-model accuracy by up to 29\%. RS2 also has the lowest subset selection overhead resulting in the best time-to-accuracy across small (CIFAR10) and large (ImageNet) datasets.

\subsection{Comparison to Dataset Distillation}
We also compare RS2 to dataset distillation methods which generate subsets of synthetic examples. We show the accuracy of RS2 with respect to these baselines on CIFAR10, CIFAR100, and Tiny ImageNet (rather than ImageNet30 or ImageNet for computational reasons) in Table~\ref{tab:dataset_distillation}. For these experiments, given the small selection ratios, we use no data augmentation when training with RS2. We also use a ConvNet model rather than a ResNet to be consistent with existing dataset distillation evaluations and because such methods can be to computationally expensive to run on complex architectures. While dataset distillation methods generally outperform the data pruning methods from the prior section (e.g., in Table~\ref{tab:dataset_distillation} a static random subset on CIFAR10 with $r=1\%$ reaches 43.4\% accuracy while data distillation methods reach up to 71.6\%), they have several drawbacks. First, subsets generated by these methods are model specific, i.e., the subset must be regenerated for every model one wishes to train. The most prominent issue, however, is the computation required to generate each subset. In fact, most methods are already too expensive to run on Tiny ImageNet, even when generating only a few examples per class. The best performing method, Trajectory Matching (TM), requires 133, 317, and 433 minutes to generate the subset with 50 images per class on CIFAR10, CIFAR100, and Tiny ImageNet, respectively. In comparison, RS2 requires just seven, 33, and 187 minutes for end-to-end training in these settings. Yet RS2 outperforms Trajectory Matching with respect to end-model accuracy for eight of the nine selection ratio/dataset combinations in Table~\ref{tab:dataset_distillation}. In the extreme compression regime ($r=0.2\%$) on Tiny ImageNet, RS2 outperforms TM by 14.7\%.

\begin{table}[t]
\caption{Accuracy achieved by dataset distillation methods, RS2, and Random data pruning when training a ConvNet model. We select the specified number of images per class (Img/Cls) corresponding to the given selection ratio on the full dataset. Best method bolded. Next best underlined.}
\vspace{4pt}
\label{tab:dataset_distillation}
\resizebox{\textwidth}{!}{
\begin{tabular}{cccccccccccccc}
\toprule
& \multirow{2}{*}{Img/Cls} & \multirow{2}{*}{Ratio \%} & \multicolumn{1}{c}{\multirow{2}{*}{Random}} & \multicolumn{8}{c}{Dataset Distillation Methods} & \multicolumn{1}{c}{\multirow{2}{*}{RS2 w/ repl}} & \multicolumn{1}{c}{\multirow{2}{*}{Full Dataset}} \\
& & & & DD & LD & DC  & DSA & DM  & CAFE & CAFE+DSA & TM & & \multicolumn{1}{c}{} \\ 
\midrule
\multirow{3}{*}{CIFAR10} 
& 1 & 0.02 & 14.4$\pm$2.0 & - & 25.7$\pm$0.7 & 28.3$\pm$0.5 & 28.8$\pm$0.7 & 26.0$\pm$0.8 & 30.3$\pm$1.1 & 31.6$\pm$0.8 & \underline{46.3$\pm$0.8} & \textbf{54.7$\pm$0.5} & \multirow{3}{*}{84.8$\pm$0.1} \\
& 10 & 0.2 & 36.8$\pm$1.2 & 36.8$\pm$1.2 & 38.3$\pm$0.4 & 44.9$\pm$0.5 & 52.1$\pm$0.5 & 48.9$\pm$0.6 & 46.3$\pm$0.6 & 50.9$\pm$0.5 & \underline{65.3$\pm$0.7} & \textbf{72.7$\pm$0.1} & \multicolumn{1}{c}{} \\
& 50 & 1 & 43.4$\pm$1.0 & - & 42.5$\pm$0.4 & 53.9$\pm$0.5 & 60.6$\pm$0.5 & 63.0$\pm$0.4 & 55.5$\pm$0.6 & 62.3$\pm$0.4 & \underline{71.6$\pm$0.2} & \textbf{76.5$\pm$0.3} & \multicolumn{1}{c}{} \\ 
\midrule
\multirow{3}{*}{CIFAR100} 
& 1 & 0.2 & 4.2$\pm$0.3 & - & 11.5$\pm$0.4 & 12.8$\pm$0.3 & 13.9$\pm$0.3 & 11.4$\pm$0.3 & 12.9$\pm$0.3 & 14.0$\pm$0.3 & \underline{24.3$\pm$0.3} & \textbf{37.4$\pm$0.4} & \multirow{3}{*}{56.2$\pm$0.3} \\
& 10 & 2 & 14.6$\pm$0.5 & - & - & 25.2$\pm$0.3 & 32.3$\pm$0.3 & 29.7$\pm$0.3 & 27.8$\pm$0.3 & 31.5$\pm$0.2 & \underline{40.1$\pm$0.4} & \textbf{43.9$\pm$0.4} & \\
& 50 & 10 & 30.0$\pm$0.4 & - & - & - & 42.8$\pm$0.4 & 43.6$\pm$0.4 & 37.9$\pm$0.3 & 42.9$\pm$0.2 & \textbf{47.7$\pm$0.2} & \underline{44.6$\pm$0.3} & \\ 
\midrule
\multirow{3}{*}{Tiny ImageNet} 
& 1 & 0.2 & 1.4$\pm$0.1 & - & - & - & - & 3.9$\pm$0.2  & - & - & \underline{8.8$\pm$0.3} & \textbf{23.5$\pm$0.2} & \multirow{3}{*}{37.6$\pm$0.4} \\
& 10 & 2 & 5.0$\pm$0.2 & - & - & - & - & 12.9$\pm$0.4 & - & - & \underline{23.2$\pm$0.2} & \textbf{27.4$\pm$0.1} & \\
& 50 & 10 & 15.0$\pm$0.4 & - & - & - & - & 24.1$\pm$0.3 & - & - & \underline{28.0$\pm$0.3} & \textbf{28.6$\pm$0.4} & \\ 
\bottomrule
\end{tabular}
}
\vspace{-0.10in}
\end{table}

\subsection{Beyond Standard Supervised Learning Benchmarks}
\label{subsec:beyond_sl}
We now consider two extensions of RS2 beyond standard supervised benchmarks: We examine 1) the robustness of data pruning methods and RS2 against noisy labels and 2) explore the benefits that RS2 can have on improving time-to-accuracy when training generative pretrained transformers.

\newparagraph{Robustness of RS2}
We evaluate the robustness of data pruning methods when the operate on a training dataset with noisy labels. To do so, we randomly flip some percentage $p$ of the labels in CIFAR10 and then run data pruning methods with these labels. We evaluate on the regular test set. Accuracy for RS2 and baseline methods when using a selection ratio of $r=10$\% and varying noise percentages $p$ is shown in Appendix Table~\ref{tab:robustness}. For each method we report end-model accuracy/raw accuracy drop compared to $p=0$/relative accuracy drop compared to $p=0$ (as a percentage of the $p=0$ accuracy). Just as for the results above on the noiseless datasets, Table~\ref{tab:robustness} shows that RS2 achieves higher end-model accuracy in the presence of noisy labels compared to existing data pruning methods. For example, with 30\% of the training examples mislabeled, RS2 without replacement achieves 74.4\% accuracy while the next closest baseline---our modified per-round version of supervised prototypes with easy examples (SP-Easy-RS)---achieves just 63.4\%. Moreover, RS2 is generally the most robust method in that it suffers the lowest relative drop in performance when presented with noisy labels. We discuss these results in more detail in Appendix~\ref{appendix:results}.

\newparagraph{RS2 for Language Model Pretraining}
One benefit of RS2 is that it can be easily generalized to settings beyond standard supervised learning. In this section, we use RS2 to reduce the cost of pretraining a large GPT2 language model. We extend RS2 to this setting as follows: we repeatedly sample random subsets of text from the dataset and use this data for next token prediction (the standard GPT2 pretraining task). We train RS2 for $r \cdot 600$k iterations for $r=[0.1, 0.3]$ and compare to training with the full dataset for 600k iterations (recall the connection between RS2 and early stopping). We also compare RS2 to random data pruning, i.e., training for $r \cdot 600$k iterations on a static fraction $r$ of the full dataset selected once before learning begins. We are not aware of existing data pruning methods being evaluated in this setting and present these results as an initial baseline.

Results of pretraining GPT2 on OpenWebText according to different strategies are shown Table~\ref{tab:gpt}. We report accuracy (higher is better) and perplexity (lower is better) on the LAMBADA~\cite{paperno2016lambada} benchmark as well as perplexity on WikiText103~\cite{merity2016pointer}. Observe that RS2 leads to better model quality but not cost compared to a static random sample. 
Moreover, we see that RS2 leads to near matching accuracy and perplexity compared to training using the full dataset for $r=30$\%. This result again highlights the practical potential of RS2 to enable faster and cheaper training, hyperparameter tuning, or neural arcitecture search for large language model pretraining, currently one of the most expensive and time consuming training paradigms in machine learning.

\begin{table}[t]
\tiny
\caption{Zero-shot results of GPT2 pretrained using RS2, a static random subset, and the full dataset. We report accuracy (ACC; higher is better) and perplexity (PPL; lower is better).}
\vspace{4pt}
\label{tab:gpt}
\centering
\begin{tabular}{cccccc}
\toprule
\multirow{2}{*}{Method} & \multirow{2}{*}{Selection Ratio ($r$)} & 2023 AWS & \multicolumn{3}{c}{Task} \\
& & Training Cost & LAMBADA (ACC $\uparrow$) & LAMBADA (PPL $\downarrow$) & WikiText103 (PPL $\downarrow$) \\ 
\midrule
Random & 10\% & \$520 & 43.37 & 45.02 & 53.36 \\
Random & 30\% & \$1,560 & 44.63 & 41.99 & 46.44 \\
RS2 w/ repl & 10\% & \$520 & 44.42 & 41.67 & 45.72 \\
RS2 w/ repl & 30\% & \$1,560 & 45.29 & 40.51 & 42.58 \\
\midrule
Full Dataset & - & \$5,200 & 46.61 & 40.30 & 40.55\\
\bottomrule
\end{tabular}
\vspace{-0.1in}
\end{table}
\section{Discussion and Conclusion}
\label{sec:discussion}
We end by discussing the limitations of RS2 before concluding and presenting future directions.

\newparagraph{RS2 Limitations}
In this work, we focused on minimizing time-to-accuracy when training over a large, labeled dataset. When the full dataset is not labeled, or when the primary metric of interest is something other than reducing time-to-accuracy, RS2 may not outperform existing methods. In particular, RS2 is likely to be a weaker baseline when the goal is to minimize the cost of labeling  examples for training by selecting a subset from a large, unlabeled dataset. We refer the reader to active learning based methods~\cite{park2022active, ren2021survey} for this regime. The assumption of labels for the full dataset is not unique to our work, however, as most recent methods for reducing time-to-accuracy utilize the labels of the full dataset $S$ to create the subset $S'$~\cite{toneva2018empirical, mirzasoleiman2020coresets, sachdeva2021svp, paul2021deep, pmlr-v139-killamsetty21a} for training.

\newparagraph{Conclusion and Future Work}
Through extensive experiments, we have shown that training on random subsets repeatedly sampled (RS2) from a large dataset results in reduced runtime and higher end-model accuracy when compared to existing data pruning and distillation methods. While the impressive performance of RS2 may provide a practical solution for reducing time-to-accuracy, e.g., for hyperparameter tuning or neural architecture search, we also hope that our findings serve as a baseline for future research to minimize time-to-accuracy through data subset selection. Specifically, we are excited for future work focused on the following question: How can we further close the gap between RS2 and training on the full dataset? Interesting sub directions to answering this question include: 1) further study of importance sampling-based methods for reducing time-to-accuracy and 2) improving the subset training procedure (independent of the method) to benefit the end-model accuracy. The key issue with the latter is that training on a subset results in fewer total SGD iterations when compared to training on the full dataset for the same number of rounds. Can we overcome this limitation of data pruning when training with SGD without eliminating the runtime benefits? We encourage new research into these questions to enable further reductions in time-to-accuracy.

\bibliographystyle{plainurl}
\bibliography{sections/ref}
\medskip

\newpage
\appendix
\section*{Appendix}
\label{appendix}

\section{A Motivating Experiment}
We have shown in the main body of the paper that Repeated Sampling of Random Subsets (RS2) allows for faster training and more accurate models when compared to existing data pruning and dataset distillation techniques. In this section, we discuss a simple experiment that helped motivate our work.

Existing data pruning methods are primarily based on the intuition that a small subset $S'$ of `difficult'~\cite{toneva2018empirical, paul2021deep} (or sometimes `easy'~\cite{sorscher2022beyond}) examples contained in the full dataset $S$ are close to (far from) the decision boundary and thus likely to be the most informative for learning. During our initial investigation into data pruning methods, we empirically studied this intuition. Calculating the distance between a training example and the decision boundary, however, can be challenging because the decision boundary is not known until training completes, and because the location of the decision boundary in high dimensional space can be computationally intensive to compute. Thus, we consider the following proxy measurement: To decide whether a training example $x$ is close to the decision boundary, we find the nearest neighbor (e.g., $L_2$ distance) from the full dataset and check whether it has the same label as $x$. If not, then the decision boundary in the input feature space must be between the two points (i.e., they are `close' to the boundary). 

We evaluated the above proxy measurement for all examples in the CIFAR10 dataset to decide whether each one was close to the decision boundary. Surprisingly, we found that the nearest neighbor for 65\% of the training examples had a different label than the example itself. In other words, in the raw feature space, this experiment provides some evidence that a majority of examples may be needed for learning the final decision boundary. This observation motivates RS2 as a strong data pruning baseline because it satisfies two desired properties: 1) it maximizes overall data coverage by periodically resampling the subset and 2) it provides representative examples from the dataset without overfitting. We remark that a majority of points are unlikely to be on the decision boundary if we first encode the input examples $x$ into a more semantically meaningful feature space. Learning such an encoding, however, requires first learning a decision boundary over the raw features and must be done during the model training itself. We leave a detailed study of this experiment, and the implications of this observation on selecting hard/easy examples for importance-sampling based data pruning to future work.

\section{Additional Experimental Setup}
\label{appendix:setup}
We expand on the experimental setup described in Section~\ref{subsec:setup} of the main body of the paper.

\subsection{Data Pruning Baselines}
We consider the following 22 data pruning baselines. We refer the reader to existing studies for more detailed descriptions of these methods~\cite{guo2022deepcore}.
\begin{enumerate}[itemsep=-0.1ex]
    \item Random: standard baseline; sample a static random subset of the dataset once before training
    \item Contextual Diversity (CD)~\cite{agarwal2020contextual}
    \item Herding~\cite{welling2009herding, chen2010super}
    \item K-Center Greedy~\cite{sener2018active}
    \item Least Confidence~\cite{sachdeva2021svp}
    \item Entropy~\cite{sachdeva2021svp}
    \item Margin~\cite{sachdeva2021svp}
    \item Forgetting~\cite{toneva2018empirical}
    \item GraNd~\cite{paul2021deep}
    \item Contrastive Active Learning (CAL)~\cite{liu2021just}
    \item Craig~\cite{mirzasoleiman2020coresets}
    \item GradMatch~\cite{pmlr-v139-killamsetty21a}
    \item Glister~\cite{killamsetty2021glister}
    \item Facility Location (FL)~\cite{iyer2021submodular}
    \item GraphCut~\cite{iyer2021submodular}
    \item Active Learning with confidence-based example informativeness (AL (Conf))~\cite{park2022active}
    \item Active Learning with loss-based example informativeness (AL (LL))~\cite{park2022active}
    \item Active Learning with margin-based example informativeness (AL (Margin))~\cite{park2022active}
    \item Self-supervised prototypes with easy examples (SSP-Easy)~\cite{sorscher2022beyond}
    \item Self-supervised prototypes with hard examples (SSP-Hard)~\cite{sorscher2022beyond}
    \item Supervised prototypes with easy examples (SP-Easy)~\cite{sorscher2022beyond}
    \item Supervised prototypes with hard examples (SP-Hard)~\cite{sorscher2022beyond}
\end{enumerate}

\subsection{Dataset Distillation Baselines}
We compare against the following eight dataset distillation methods.
\begin{enumerate}[itemsep=-0.1ex]
    \item Dataset Distillation (DD)~\cite{wang2018dataset}
    \item Flexible Dataset Distillation (LD)~\cite{bohdal2020flexible}
    \item Dataset Condensation (DC)~\cite{zhao2021DC}
    \item Differentiable Siamese Augmentation (DSA)~\cite{zhao2021DSA}
    \item Distribution Matching (DM)~\cite{zhao2023dataset}
    \item Aligning Features (CAFE)~\cite{wang2022cafe}
    \item Aligning Features + Differentiable Siamese Augmentation (CAFE+DSA)~\cite{wang2022cafe}
    \item Trajectory Matching (TM)~\cite{cazenavette2022dataset}
\end{enumerate}

\subsection{Additional Training Details}
For all experiments (except GPT2 due to cost considerations) we conduct three runs using different random seeds and report the average accuracy and runtime. We include additional details on the hyperparameters and hardware used below.

\newparagraph{Hyperparameters}
We use the following hyperparameters for our experiments: For CIFAR10 and CIFAR100 experiments, we use SGD as the optimizer with batch size 128, initial learning rate 0.1, a cosine decay learning rate schedule, momentum 0.9, weight decay 0.0005, and 200 training epochs. For data augmentation, we apply random cropping and horizontal flipping with four-pixel padding on the 32$\times$32 training images. For ImageNet30 and ImageNet, we use the same hyerparameters as above except for a larger batch size on ImageNet (256). We also use different data augmentation: training images are randomly resized and cropped to 224$\times$224 with random horizontal flipping. Further details can be found in the source code.

\newparagraph{Hardware Setup}
As described in the main body of the paper, we run image classification experiments on a university cluster with job isolation and NVIDIA RTX 3090 GPUs. We run GPT2 experiments using AWS P3 GPU instances with eight NVIDIA V100 GPUs (as GPT2 experiments require more compute power). Utilizing the former allows us to reduce the cost of our experiments (e.g., compared to training entirely using AWS), but introduces the potential for increased variance compared to training with completely dedicated hardware---Even though all experiments run with exclusive access to one GPU and a set of CPU cores, cluster load can influence runtime measurements. We observe small variance across multiple runs of the same experiment on small datasets (e.g., on CIFAR10 the three run standard deviation is generally less than one percent of the total runtime), but larger variance on ImageNet, likely do to an increased load on the shared file system and longer experiment runtimes. As such, we calculate the runtime of each method on ImageNet as follows: We calculate the minimum time per mini-batch using all runs across \textit{all} methods, and then use this value to compute individual method runtimes by multiplying by the total number of batches during training and adding any necessary overheads for subset selection. More specifically, we have: the total runtime of any method $T_{\text{total}} = T_{\text{total\_subset\_selection}} + T_{\text{total\_training\_time}}$ with $T_{\text{total\_training\_time}} = T_{\text{global\_minimum\_batch\_runtime}} \times \text{total\_number\_of\_batches}$. Note that this means runtimes differ only due to subset selection overhead as expected (once a subset has been selected, all methods train on the same number of examples per round using the same hardware, and thus should have the same per round training time). Furthermore, we calculate $T_{\text{total\_subset\_selection}}$ as the minimum subset selection time observed across three runs of each method. The above runtime calculation allows us to minimize the affect of cluster noise on our experiments and ensure a fair comparison for the ImageNet time-to-accuracy reported in the paper.

\section{Additional Experimental Results}
\label{appendix:results}
Here we include additional evaluation result comparing Repeated Sampling of Random Subsets (RS2) to existing data pruning and dataset distillation methods. These results extend those presented in Section~\ref{sec:exp} of the main paper. We briefly discuss each result (table) in turn and how it connects to the arguments made in Section~\ref{sec:exp}.

First, in Table~\ref{tab:robustness} we show the robustness of RS2 and existing data pruning methods against noisy labels. We include existing methods which sample static subsets, as well as our modified version of the recent prototype-based data pruning method which utilizes repeated subset selection between each round (SP-Easy-RS) (see Section~\ref{subsec:exp_e2e}). As discussed in Section~\ref{subsec:beyond_sl} in the main paper, we evaluate the robustness of data pruning methods as follows: We randomly flip some percentage $p$ of the labels in CIFAR10 and then run data pruning methods with these labels. We use a subset selection ratio $r=10$\% for all methods and evaluate on the regular test set. For each method we report end-model accuracy/raw accuracy drop compared to $p = 0$/relative accuracy drop compared to $p = 0$ (as a percentage of the $p = 0$ accuracy). Table~\ref{tab:robustness} shows that RS2 achieves higher end-model accuracy than existing data pruning methods in the presence of noisy labels. RS2 is also the most robust method (lowest relative accuracy drop) when the noise ratio is 10\% and 30\%. Interestingly, the GraNd baseline actually gets better as the noise ratio increases. While surprising, the overall end-model quality of GraNd is still limited, however, as the GraNd accuracy begins to decrease again as the noise increases beyond 50\% and all noise ratios result in lower accuracy than training on clean data. We leave a detailed study of these observations and robust data pruning methods for future work.

\begin{table}[t]
\scriptsize
\caption{Accuracy achieved by different data pruning methods when training ResNet-18 on the CIFAR10 dataset with $p$ percent of the train set labels randomly flipped (noise ratio). Data pruning methods use a selection ratio of 10\%. We test on the normal test set. We report raw end-model accuracy/accuracy drop compared to $p=0$/relative accuracy drop compared to $p=0$ (as a percentage of the $p=0$ accuracy). Best method (highest accuracy) bolded; Next best underlined. Most robust method (lowest relative accuracy drop) starred.}
\vspace{4pt}
\label{tab:robustness}
\centering
\begin{tabular}{cccc}
\toprule
Noise Ratio & 10\% & 30\% & 50\% \\ 
\midrule
Random & 50.7$\pm$1.0/27.7/35.4 & 40.9$\pm$2.3/37.5/47.8 &  37.2$\pm$1.1/41.2/52.6 \\
CD & 31.2$\pm$3.9/27.6/47.0 & 30.3$\pm$0.6/28.5/48.4 &  29.7$\pm$2.9/29.1/49.5 \\
Herding & 15.9$\pm$2.2/47.6/74.9 & 17.1$\pm$1.1/46.4/73.1 &  16.1$\pm$1.8/47.4/74.7 \\
K-Center Greedy & 41.6$\pm$1.7/33.6/44.6 &  33.1$\pm$0.7/42.1/56.0 & 31.1$\pm$0.4/44.1/58.6 \\
Least Confidence & 27.4$\pm$0.7/30.2/52.4 &  25.2$\pm$2.1/32.4/56.2 & 24.5$\pm$2.6/33.1/57.5 \\
Entropy & 26.9$\pm$4.2/30.7/53.2 & 24.9$\pm$3.7/32.7/56.8 &  23.3$\pm$2.0/34.3/59.5 \\
Margin & 29.2$\pm$2.2/44.0/60.1 & 27.4$\pm$2.2/45.8/62.6 &  28.5$\pm$3.3/44.7/61.0 \\
Forgetting & 47.5$\pm$1.6/31.5/39.9 & 48.3$\pm$1.7/30.7/38.9 &  49.0$\pm$1.6/30.0/37.9 \\
GraNd & 34.5$\pm$3.2/40.9/54.3 & \underline{49.4$\pm$5.6/26.0/34.5} & \underline{60.0$\pm$1.3/15.4/20.4$^*$} \\
CAL & 45.4$\pm$2.5/26.4/36.8 & 39.9$\pm$1.2/31.9/44.4 &  35.5$\pm$0.4/36.3/50.5 \\
Craig & 50.3$\pm$0.7/9.9/16.5 & 39.5$\pm$0.8/20.7/34.4 & 37.3$\pm$1.5/22.9/38.0 \\
Glister & 49.8$\pm$2.5/25.9/34.2 & 40.1$\pm$1.5/35.6/47.1 &  37.9$\pm$1.5/37.8/50.0 \\
GraphCut & 50.6$\pm$1.4/23.4/31.7 & 42.0$\pm$0.7/32.0/43.2 &  37.7$\pm$1.1/36.3/49.0 \\
FL & 50.4$\pm$2.1/24.3/32.5 & 41.6$\pm$0.8/33.1/44.3 &  36.8$\pm$0.7/37.9/50.7 \\
AL (Conf) & 53.5$\pm$3.1/30.1/36.0 & 47.1$\pm$0.9/36.5/43.7 &  37.2$\pm$1.4/46.4/55.5 \\
AL (LL) & 57.1$\pm$0.4/27.9/32.8 & 45.1$\pm$1.8/39.9/46.9 &  38.2$\pm$0.6/46.8/55.1 \\
AL (Margin) & \underline{57.6$\pm$0.5/26.9/31.8} & 46.1$\pm$1.2/38.4/45.4 &  36.9$\pm$1.0/47.6/56.3 \\
SSP-Easy & 51.1$\pm$1.5/20.9/29.0 & 40.7$\pm$1.9/31.3/43.4 &  36.7$\pm$1.0/35.3/49.1 \\
SSP-Hard & 50.4$\pm$0.3/23.9/32.2 & 40.9$\pm$1.5/33.4/44.9 &  36.3$\pm$1.9/38.0/51.2 \\
SP-Easy & 48.4$\pm$2.4/23.9/33.0 & 40.0$\pm$0.4/32.3/44.6 &  37.7$\pm$1.0/34.6/47.8 \\
SP-Hard & 47.5$\pm$2.4/26.6/35.9 & 39.7$\pm$1.3/34.4/46.4 &  34.3$\pm$2.0/39.8/53.7 \\ 
\midrule
SP-Easy-RS & \underline{74.2$\pm$0.6/14.2/16.0} & \underline{63.4$\pm$1.0/25.0/28.3} & 57.8$\pm$0.7/30.6/34.7 \\
\midrule
RS2 w/ repl & 77.5$\pm$1.0/12.2/13.6$^*$ & 69.9$\pm$0.4/19.8/22.1 & 64.6$\pm$1.5/25.1/28.0 \\
RS2 w/ repl (stratified) & 76.1$\pm$0.5/13.7/15.3 & 68.7$\pm$0.6/21.1/23.5 & 65.0$\pm$1.4/24.8/27.7 \\
RS2 w/o repl & \textbf{78.7$\pm$0.8/13.0/14.1} & \textbf{74.4$\pm$0.6/17.3/18.9$^*$} & \textbf{69.0$\pm$0.9/22.7/24.8} \\ 
\bottomrule
\end{tabular}
\end{table}

In Tables~\ref{tab:cifar10_acc}-\ref{tab:imagenet_acc} we show the end-model accuracy of RS2 and existing data pruning methods for varying selection ratios on CIFAR10 and ImageNet respectively. The numbers in these tables were used to create Figure~\ref{fig:acc} in the main body of the paper. Recall from the discussion of Figure~\ref{fig:acc} in Section~\ref{subsec:exp_e2e} that we use the combined baseline methods from recent studies~\cite{park2022active, guo2022deepcore} together with newer prototype-based data pruning methods~\cite{sorscher2022beyond}. Recall also that for these tables, we use the setting proposed by these works: for all baselines, we sample a static subset once before training starts. We use all baseline methods for CIFAR10, but some methods do not scale to the larger ImageNet dataset. We show in Table~\ref{tab:cifar10_subset_selection} that active learning already takes more than eight hours for subset selection in some settings on CIFAR10 and we are not aware of a scalable implementation of prototype-based methods that would allow for training on ImageNet. As in Figure~\ref{fig:acc}, Tables~\ref{tab:cifar10_acc}-\ref{tab:imagenet_acc} show that the repeated sampling of RS2 leads to accuracy improvements compared to existing data pruning methods which sample a static subset (see discussion in Section~\ref{subsec:exp_e2e}). Interestingly, while RS2 generally outperforms existing methods in the high compression regime ($r \leq 10$\%), for extreme compression ratios, like $r=0.1$\% on ImageNet, we find RS2 to be inferior to existing methods. We hypothesize that this occurs because in these extreme regimes, only a few examples are shown to the model for each class and these examples likely have large variance when using repeated random sampling coupled with data augmentation. In this setting, it may be best to select a static subset of only the easiest examples as highlighted in recent work~\cite{sorscher2022beyond}, however the significance of this regime is debatable given the low end-model accuracy of all methods. Improving the performance in these regimes is of interest for future work.

\begin{table}[t]
\scriptsize
\caption{Accuracy achieved by different data pruning methods when training ResNet-18 on CIFAR10 for different subset selection sizes. Best method bolded; Next best underlined.}
\vspace{4pt}
\label{tab:cifar10_acc}
\centering
\resizebox{\textwidth}{!}{
\begin{tabular}{ccccccccc}
\toprule
Selection Ratio ($r$) & 1\% & 5\% & 10\% & 20\% & 30\% & 40\% & 50\% & 100\% \\ 
\midrule
Random & 36.7$\pm$1.7 & 64.5$\pm$1.1 & 78.4$\pm$0.9 & 88.1$\pm$0.5 & 91.0$\pm$0.3 & 91.9$\pm$0.2 & 93.2$\pm$0.3 & 95.5$\pm$0.2 \\
CD & 23.6$\pm$1.9 & 38.1$\pm$2.2 & 58.8$\pm$2.0 & 81.3$\pm$2.5 & 90.8$\pm$0.5 & 93.3$\pm$0.4 & 94.3$\pm$0.2 & 95.5$\pm$0.2 \\
Herding & 34.8$\pm$3.3 & 51.0$\pm$3.1 & 63.5$\pm$3.4 & 74.1$\pm$2.5 & 80.1$\pm$2.2 & 85.2$\pm$0.9 & 88.0$\pm$1.1 & 95.5$\pm$0.2 \\
K-Center Greedy & 31.1$\pm$1.2 & 51.4$\pm$2.1 & 75.2$\pm$1.7 & 87.3$\pm$1.0 & 91.2$\pm$0.6 & 92.2$\pm$0.5 & 93.8$\pm$0.5 & 95.5$\pm$0.2 \\
Least Confidence & 19.8$\pm$2.2 & 36.2$\pm$1.9 & 57.6$\pm$3.1 & 81.9$\pm$2.2 & 90.3$\pm$0.4 & 93.1$\pm$0.5 & 94.5$\pm$0.1 & 95.5$\pm$0.2 \\
Entropy & 21.1$\pm$1.3 & 35.3$\pm$3.0 & 57.6$\pm$2.8 & 81.9$\pm$0.4 & 89.8$\pm$1.6 & 93.2$\pm$0.2 & 94.4$\pm$0.3 & 95.5$\pm$0.2\\
Margin & 28.2$\pm$1.0 & 43.4$\pm$3.3 & 73.2$\pm$1.3 & 85.5$\pm$0.9 & 91.3$\pm$0.5 & 93.6$\pm$0.3 & 94.5$\pm$0.2 & 95.5$\pm$0.2 \\
Forgetting & 35.2$\pm$1.6 & 52.1$\pm$2.2 & 79.0$\pm$1.0 & 89.8$\pm$0.9 & 92.3$\pm$0.4 & 93.6$\pm$0.4 & 93.8$\pm$0.3 & 95.5$\pm$0.2 \\
GraNd & 26.7$\pm$1.3 & 39.8$\pm$2.3 & 75.4$\pm$1.2 & 88.6$\pm$0.6 & 92.4$\pm$0.4 & 93.3$\pm$0.5 & 94.2$\pm$0.4 & 95.5$\pm$0.2 \\
CAL & 37.8$\pm$2.0 & 60.0$\pm$1.4 & 71.8$\pm$1.0 & 80.9$\pm$1.1 & 86.0$\pm$1.9 & 87.5$\pm$0.8 & 89.4$\pm$0.6 & 95.5$\pm$0.2 \\
Craig & 31.7$\pm$1.1 & 45.2$\pm$2.9 & 60.2$\pm$4.4 & 79.6$\pm$3.1 & 88.4$\pm$0.5 & 90.8$\pm$1.4 & 93.3$\pm$0.6 & 95.5$\pm$0.2 \\
GradMatch & 30.8$\pm$1.0 & 47.2$\pm$0.7 & 61.5$\pm$2.4 & 79.9$\pm$2.6 & 87.4$\pm$2.0 & 90.4$\pm$1.5 & 92.9$\pm$0.6 & 95.5$\pm$0.2 \\
Glister & 32.9$\pm$2.4 & 50.7$\pm$1.5 & 75.7$\pm$1.0 & 86.3$\pm$0.9 & 90.1$\pm$0.7 & 91.5$\pm$0.5 & 93.3$\pm$0.6 & 95.5$\pm$0.2 \\
FL & 38.9$\pm$1.4 & 60.8$\pm$2.5 & 74.7$\pm$1.3 & 85.6$\pm$1.9 & 91.4$\pm$0.4 & 93.2$\pm$0.3 & 93.9$\pm$0.2 & 95.5$\pm$0.2 \\
GraphCut & \underline{42.8$\pm$1.3} & \underline{65.7$\pm$1.2} & 74.0$\pm$1.5 & 86.3$\pm$0.9 & 90.2$\pm$0.5 & 91.5$\pm$0.4 & 93.8$\pm$0.5 & 95.5$\pm$0.2 \\
AL (Conf) & 35.2$\pm$1.5 & 60.6$\pm$3.1 & 83.6$\pm$0.7 & 90.5$\pm$0.4 & 93.8$\pm$0.4 & 94.8$\pm$0.3 & 95.1$\pm$0.3 & 95.5$\pm$0.2 \\
AL (LL) & 37.5$\pm$4.3 & 63.1$\pm$2.0 & \underline{85.0$\pm$0.9} & \underline{91.2$\pm$0.7} & 93.8$\pm$0.6 & 94.4$\pm$0.5 & 95.0$\pm$0.4 & 95.5$\pm$0.2 \\
AL (Margin) & 36.7$\pm$0.8 & 62.2$\pm$1.1 & 84.5$\pm$0.7 & 91.0$\pm$0.5 & \underline{93.9$\pm$0.4} & \underline{94.5$\pm$0.3} & \textbf{95.3$\pm$0.2} & 95.5$\pm$0.2 \\
SSP-Easy & 35.6$\pm$1.7 & 62.1$\pm$1.2 & 72.0$\pm$0.8 & 85.9$\pm$0.4 & 90.0$\pm$0.2 & 91.5$\pm$0.4 & 92.7$\pm$0.0 & 95.5$\pm$0.2 \\
SSP-Hard & 34.2$\pm$1.1 & 58.0$\pm$2.4 & 74.3$\pm$1.7 & 86.1$\pm$1.3 & 90.3$\pm$0.4 & 91.9$\pm$0.3 & 93.3$\pm$0.2 & 95.5$\pm$0.2 \\
SP-Easy & 37.1$\pm$1.4 & 59.8$\pm$0.5 & 72.3$\pm$2.9 & 85.1$\pm$1.0 & 89.6$\pm$0.2 & 91.6$\pm$0.2 & 92.7$\pm$0.2 & 95.5$\pm$0.2 \\
SP-Hard & 35.0$\pm$0.7 & 60.9$\pm$1.8 & 74.1$\pm$1.1 & 86.3$\pm$0.3 & 89.8$\pm$0.6 & 91.5$\pm$0.3 & 93.0$\pm$0.3 & 95.5$\pm$0.2 \\ 
\midrule
RS2 w/ repl & 51.1$\pm$3.5 & 86.7$\pm$0.8 & 89.7$\pm$0.2 & 93.5$\pm$0.3 & 94.2$\pm$0.1 & 94.6$\pm$0.2 & 95.1$\pm$0.2 & 95.5$\pm$0.2 \\
RS2 w/ repl (stratified) & 51.1$\pm$4.5 & 86.6$\pm$0.5 & 89.8$\pm$0.4 & 93.4$\pm$0.1 & \textbf{94.5$\pm$0.1} & \textbf{94.8$\pm$0.1} & 95.1$\pm$0.3 & 95.5$\pm$0.2   \\ 
RS2 w/o repl & \textbf{51.8$\pm$2.0} & \textbf{87.1$\pm$0.8} & \textbf{91.7$\pm$0.5} & \textbf{94.0$\pm$0.5} & 94.3$\pm$0.2 & 94.7$\pm$0.1 & \underline{95.2$\pm$0.1} & 95.5$\pm$0.2 \\
\bottomrule
\end{tabular}
}
\end{table}

\begin{table}[t]
\scriptsize
\caption{Accuracy achieved by different data pruning methods when training ResNet-18 on ImageNet for different subset selection sizes. Repeatedly Sampling Random Subsets (RS2) considerably outperforms existing methods for realistic selection ratios. Best method bolded; Next best underlined.}
\vspace{4pt}
\label{tab:imagenet_acc}
\centering
\resizebox{\textwidth}{!}{
\begin{tabular}{cccccccc}
\toprule
Select Ratio ($r$) & 0.1\% & 0.5\% & 1\% & 5\% & 10\% & 30\% & 100\% \\ 
\midrule
Random & 0.76$\pm$0.01 & 3.78$\pm$0.14 & 8.85$\pm$0.46 & 40.09$\pm$0.21 & 52.1$\pm$0.22 & \underline{64.11$\pm$0.05} & 69.52$\pm$0.45 \\
CD & 0.76$\pm$0.01 & 1.18$\pm$0.06 & 2.16$\pm$0.18 & 25.82$\pm$2.02 & 43.84$\pm$0.12 & 62.13$\pm$0.45 & 69.52$\pm$0.45 \\
Herding & 0.34$\pm$0.01 & 1.7$\pm$0.13  & 4.17$\pm$0.26 & 17.41$\pm$0.34 & 28.06$\pm$0.05 & 48.58$\pm$0.49 & 69.52$\pm$0.45 \\
K-Center Greedy & 0.76$\pm$0.01 & 1.57$\pm$0.09 & 2.96$\pm$0.24 & 27.36$\pm$0.08 & 44.84$\pm$1.03 & 62.12$\pm$0.46 & 69.52$\pm$0.45 \\
Least Confidence & 0.29$\pm$0.04 & 1.03$\pm$0.25 & 2.05$\pm$0.38 & 27.05$\pm$3.25 & 44.47$\pm$1.42 & 61.8$\pm$0.33 & 69.52$\pm$0.45 \\
Entropy& 0.31$\pm$0.02 & 1.01$\pm$0.17 & 2.26$\pm$0.3 & 28.21$\pm$2.83 & 44.68$\pm$1.54 & 61.82$\pm$0.31 & 69.52$\pm$0.45 \\
Margin & 0.47$\pm$0.02 & 1.99$\pm$0.29 & 4.73$\pm$0.64 & 35.99$\pm$1.67 & 50.29$\pm$0.92 & 63.62$\pm$0.15 & 69.52$\pm$0.45 \\
Forgetting & 0.76$\pm$0.01 & 4.69$\pm$0.17 & 14.02$\pm$0.13 & \underline{47.64$\pm$0.03} & \underline{55.12$\pm$0.13} & 62.49$\pm$0.11 & 69.52$\pm$0.45 \\
GraNd & 1.04$\pm$0.04 & 7.02$\pm$0.05 & \underline{18.1$\pm$0.22} & 43.53$\pm$0.19 & 49.92$\pm$0.21 & 57.98$\pm$0.17 & 69.52$\pm$0.45 \\
CAL & \textbf{1.29$\pm$0.09} & 7.5$\pm$0.26 & 15.94$\pm$1.3 & 38.32$\pm$0.78 & 46.49$\pm$0.29 & 58.31$\pm$0.32 & 69.52$\pm$0.45 \\
Craig & 1.13$\pm$0.08 & 5.44$\pm$0.52 & 9.4$\pm$1.69 & 32.3$\pm$1.24 & 38.77$\pm$0.56 & 44.89$\pm$3.72 & 69.52$\pm$0.45 \\
GradMatch & 0.93$\pm$0.04 & 5.2$\pm$0.22 & 12.28$\pm$0.49 & 40.16$\pm$2.28 & 45.91$\pm$1.73 & 52.69$\pm$2.16 & 69.52$\pm$0.45 \\
Glister & 0.98$\pm$0.06 & 5.91$\pm$0.42 & 14.87$\pm$0.14 & 44.95$\pm$0.28 & 52.04$\pm$1.18 & 60.26$\pm$0.28 & 69.52$\pm$0.45 \\
FL & \underline{1.23$\pm$0.03} & 5.78$\pm$0.08 & 12.72$\pm$0.21 & 40.85$\pm$1.25 & 51.05$\pm$0.59 & 63.14$\pm$0.03 & 69.52$\pm$0.45 \\
GraphCut & 1.21$\pm$0.09 & \underline{7.66$\pm$0.43} & 16.43$\pm$0.53 & 42.23$\pm$0.6 & 50.53$\pm$0.42 & 63.22$\pm$0.26 & 69.52$\pm$0.45 \\ 
\midrule
RS2 w/ repl & 0.17$\pm$0.03 & 16.35$\pm$0.56 & 44.45$\pm$0.07 & 45.4$\pm$7.18 & 64.87$\pm$0.10 & 68.23$\pm$0.07 & 69.52$\pm$0.45 \\
RS2 w/ repl (stratified) & 0.18$\pm$0.02 & \textbf{33.66$\pm$0.13} & \textbf{46.96$\pm$0.13} & 62.32$\pm$0.08  & 64.92$\pm$0.10 & \textbf{68.24$\pm$0.08}  & 69.52$\pm$0.45 \\
RS2 w/o repl & 0.19$\pm$0.02 & 18.2$\pm$0.35 & 44.42$\pm$0.04 & \textbf{63.2$\pm$0.07} & \textbf{66.0$\pm$0.18} & 68.19$\pm$0.06 & 69.52$\pm$0.45 \\
\bottomrule
\end{tabular}
}
\end{table}

In Table~\ref{tab:acc_cfiar100_imagenet30}, we include additional end-model accuracy results for RS2 and existing data pruning methods on two datasets, CIFAR100 and ImageNet30, not included in the main paper due to space considerations. For these experiments, we include a representative set of baseline methods which sample static subsets, together with our modified version of the recent prototype-based data pruning method which utilizes repeated subset selection between each round (SP-Easy-RS) (see Section~\ref{subsec:exp_e2e}). Thus, Table~\ref{tab:acc_cfiar100_imagenet30} extends the end-model accuracy results presented previously for CIFAR10 and ImageNet in Figure~\ref{fig:acc} and Tables~\ref{tab:per_epoch},~\ref{tab:cifar10_acc}, and~\ref{tab:imagenet_acc}. Observe that RS2 also outperforms existing methods on these datasets. For example, in the high compression regime ($r=10$\%), RS2 without replacement reaches 73\% accuracy on CIFAR100, while the best baseline method, our per-round prototype-based data pruning method reaches only 66\%. Existing methods which sample static subsets only once before training begins reach just 36\% in this setting. 

\begin{table}[t]
\caption{Accuracy achieved by select data pruning methods when training ResNet-18 on CIFAR100 and ImageNet30. Best method bolded; Next best underlined.}
\vspace{4pt}
\label{tab:acc_cfiar100_imagenet30}
\resizebox{\columnwidth}{!}{
\begin{tabular}{lcllllllllll}
\toprule
Dataset & Select Ratio ($r$) & \multicolumn{1}{c}{10\%} & \multicolumn{1}{c}{20\%} & \multicolumn{1}{c}{30\%} & \multicolumn{1}{c}{40\%} & \multicolumn{1}{c}{50\%} & \multicolumn{1}{c}{60\%} & \multicolumn{1}{c}{70\%} & \multicolumn{1}{c}{80\%} & \multicolumn{1}{c}{90\%} & \multicolumn{1}{c}{100\%} \\ 
\midrule
\multicolumn{1}{l}{\multirow{17}{*}{CIFAR100}} & Random & 32.0$\pm$0.9 & 53.6$\pm$0.6 & 63.6$\pm$0.5 & 67.2$\pm$0.5 & 71.0$\pm$0.3 & 73.1$\pm$0.4 & 75.2$\pm$0.2 & 76.1$\pm$0.3 & 77.5$\pm$0.2 & 78.7$\pm$0.2 \\
\multicolumn{1}{l}{} & K-Center Greedy & 33.9$\pm$1.5 & 56.2$\pm$0.9 & 64.5$\pm$0.6 & 69.8$\pm$0.4 & 72.1$\pm$0.5 & 74.3$\pm$0.4 & 75.8$\pm$0.3 & 77.2$\pm$0.2 & 77.8$\pm$0.2 & 78.7$\pm$0.2 \\  
\multicolumn{1}{l}{} & Margin & 18.7$\pm$2.1 & 38.2$\pm$1.6 & 58.1$\pm$0.8 & 65.1$\pm$0.6 & 70.1$\pm$0.5 & 73.3$\pm$0.3 & 75.4$\pm$0.3 & 76.9$\pm$0.4 & \underline{78.5$\pm$0.2} & 78.7$\pm$0.2 \\
\multicolumn{1}{l}{} & Forgetting & 35.4$\pm$1.0 & 54.7$\pm$0.9 & 64.6$\pm$0.7 & 68.6$\pm$0.8 & 71.5$\pm$0.4 & 73.7$\pm$0.5 & 75.5$\pm$0.3 & 76.1$\pm$0.3 & 76.9$\pm$0.3 & 78.7$\pm$0.2 \\
\multicolumn{1}{l}{}& GraNd & 30.8$\pm$1.9 & 49.4$\pm$1.0 & 62.8$\pm$0.9 & 68.1$\pm$0.6 & 70.5$\pm$0.3 & 72.5$\pm$0.4 & 74.5$\pm$0.3 & 76.4$\pm$0.2 & 77.8$\pm$0.2 & 78.7$\pm$0.2 \\
\multicolumn{1}{l}{} & Glister & \underline{36.4$\pm$1.0} & 55.5$\pm$1.0 & 63.9$\pm$0.8 & 69.1$\pm$0.7 & 71.2$\pm$0.6 & 73.5$\pm$0.4 & 75.0$\pm$0.3 & 76.9$\pm$0.2 & 77.6$\pm$0.2 & 78.7$\pm$0.2 \\
\multicolumn{1}{l}{} & GraphCut & 36.3$\pm$1.1 & 56.0$\pm$0.8 & 65.5$\pm$0.6 & 69.5$\pm$0.4 & 71.1$\pm$0.4 & 73.8$\pm$0.4 & 75.4$\pm$0.2 & 76.4$\pm$0.2 & 78.0$\pm$0.2 & 78.7$\pm$0.2 \\
\multicolumn{1}{l}{} & AL (Conf) & 36.1$\pm$1.6 & 55.7$\pm$1.0  & 65.8$\pm$0.7 & \underline{70.6$\pm$0.5} & \underline{73.7$\pm$0.4} & \underline{76.1$\pm$0.5} & 77.1$\pm$0.3 & 78.0$\pm$0.2 & 78.4$\pm$0.2 & 78.7$\pm$0.2 \\
\multicolumn{1}{l}{} & AL (LL) & 33.1$\pm$1.9 & 55.3$\pm$1.3 & 64.9$\pm$0.8 & 70.3$\pm$0.7 & 73.1$\pm$0.5 & 75.9$\pm$0.5 & 77.0$\pm$0.3 & \underline{78.2$\pm$0.3} & \underline{78.5$\pm$0.2} & 78.7$\pm$0.2 \\
\multicolumn{1}{l}{} & AL (Margin) & 36.0$\pm$1.0 & \underline{57.3$\pm$0.5} & \underline{66.0$\pm$0.6} & 70.4$\pm$0.5 & 73.6$\pm$0.5& \underline{76.1$\pm$0.4} & \underline{77.2$\pm$0.3} & \underline{78.2$\pm$0.3} & \underline{78.5$\pm$0.2} & 78.7$\pm$0.2 \\
\multicolumn{1}{l}{} & SSP-Easy & 32.8$\pm$2.0 & 50.0$\pm$1.5 & 62.5$\pm$1.5 & 67.9$\pm$0.3 & 70.2$\pm$0.2 & 73.4$\pm$0.3 & 75.0$\pm$0.7 & 76.3$\pm$0.6 & 77.4$\pm$0.1 & 78.7$\pm$0.2 \\
\multicolumn{1}{l}{} & SSP-Hard & 29.7$\pm$1.5 & 53.3$\pm$0.6 & 63.2$\pm$0.5 & 67.8$\pm$0.2 & 71.3$\pm$0.2 & 72.9$\pm$0.2 & 74.8$\pm$0.1 & 75.9$\pm$0.8 & 77.1$\pm$0.2 & 78.7$\pm$0.2 \\
\multicolumn{1}{l}{} & SP-Easy & 33.6$\pm$0.9 & 53.0$\pm$2.0 & 63.0$\pm$1.0 & 67.4$\pm$1.0 & 70.5$\pm$0.3 & 73.3$\pm$0.2 & 74.9$\pm$0.2 & 76.3$\pm$0.6 & 76.9$\pm$0.3 & 78.7$\pm$0.2 \\
\multicolumn{1}{l}{} & SP-Hard & 31.2$\pm$2.7 & 53.6$\pm$0.4 & 63.0$\pm$0.6 & 68.0$\pm$0.8 & 71.1$\pm$0.3 & 73.0$\pm$0.4 & 74.6$\pm$0.8 & 75.8$\pm$0.9 & 77.4$\pm$0.4 & 78.7$\pm$0.2 \\ 
\cmidrule(lr){2-12}
\multicolumn{1}{l}{} & SP-Easy-RS & 66.1$\pm$1.8 & 72.7$\pm$0.6 & 74.6$\pm$0.5 & 75.5$\pm$0.2 & 76.3$\pm$0.3 & 76.9$\pm$0.4 & 77.6$\pm$0.1 & 78.0$\pm$0.1 & 78.3$\pm$0.3 & 78.7$\pm$0.2 \\ 
\cmidrule(lr){2-12}
\multicolumn{1}{l}{} & RS2 w/ repl & 68.8$\pm$1.5 & 74.4$\pm$0.1 & \textbf{76.1$\pm$0.3} & 76.8$\pm$0.1 & \textbf{77.6$\pm$0.2} & 77.7$\pm$0.0 & \textbf{78.3$\pm$0.3} & \textbf{78.4$\pm$0.2} & \textbf{78.7$\pm$0.1} & 78.7$\pm$0.2 \\
\multicolumn{1}{l}{} & RS2 w/ repl (stratified) & 68.6$\pm$2.1 & 74.6$\pm$0.7 & 75.9$\pm$0.2 & 76.7$\pm$0.2 & 77.5$\pm$0.1 & 77.7$\pm$0.1 & 78.1$\pm$0.3 & 78.2$\pm$0.2 & 78.3$\pm$0.3 & 78.7$\pm$0.2 \\ 
\multicolumn{1}{l}{} & RS2 w/o repl & \textbf{73.0$\pm$0.3} &  \textbf{74.9$\pm$0.7} &  \textbf{76.1$\pm$0.5} &  \textbf{77.1$\pm$0.1} &  77.5$\pm$0.4 &  \textbf{78.0$\pm$0.1} &  \textbf{78.3$\pm$0.2} &  78.3$\pm$0.2 &  78.4$\pm$0.3 & 78.7$\pm$0.2 \\
\midrule
\multicolumn{1}{l}{\multirow{17}{*}{ImageNet30}} & Random & 69.3$\pm$0.7 & 83.7$\pm$0.5 & 86.9$\pm$0.4 & 90.3$\pm$0.3 & 92.2$\pm$0.3 & 93.0$\pm$0.2 & 94.6$\pm$0.3 & 95.2$\pm$0.2 & 95.4$\pm$0.2 & 96.1$\pm$0.1 \\
\multicolumn{1}{l}{} & K-Center Greedy & 69.7$\pm$0.9 & 84.1$\pm$0.5 & 88.9$\pm$0.4 & 91.6$\pm$0.3 & 93.4$\pm$0.2 & 94.4$\pm$0.3 & 95.1$\pm$0.2 & 95.3$\pm$0.2 & 95.6$\pm$0.2 & 96.1$\pm$0.1 \\ 
\multicolumn{1}{l}{} & Margin & 56.9$\pm$1.1 & 77.3$\pm$0.7 & 83.7$\pm$0.5 & 90.5$\pm$0.4 & 92.9$\pm$0.2 & 94.4$\pm$0.3 & 95.1$\pm$0.2 & \textbf{95.8$\pm$0.2} & \textbf{96.0$\pm$0.1} & 96.1$\pm$0.1 \\
\multicolumn{1}{l}{} & Forgetting & 64.1$\pm$0.9 & 85.4$\pm$0.7 & 87.3$\pm$0.5 & 90.9$\pm$0.3 & 93.6$\pm$0.4 & 94.8$\pm$0.2 & 94.9$\pm$0.2 & 95.1$\pm$0.2 & 95.3$\pm$0.2 & 96.1$\pm$0.1 \\
\multicolumn{1}{l}{} & GraNd & 69.3$\pm$0.9 & 85.7$\pm$0.5 & 90.0$\pm$0.5 & 92.4$\pm$0.4 & 93.6$\pm$0.3 & 94.7$\pm$0.4 & 95.1$\pm$0.2 & 95.5$\pm$0.2 & 95.7$\pm$0.1 & 96.1$\pm$0.1 \\
\multicolumn{1}{l}{} & Glister & \underline{72.4$\pm$0.7} & 82.9$\pm$0.5 & 87.0$\pm$0.4 & 91.2$\pm$0.3 & 92.7$\pm$0.3 & 93.3$\pm$0.3 & 94.2$\pm$0.2 & 95.0$\pm$0.2 & 95.8$\pm$0.2 & 96.1$\pm$0.1 \\
\multicolumn{1}{l}{} & GraphCut & 71.9$\pm$0.6 & 83.0$\pm$0.3 & 88.5$\pm$0.3 & 91.2$\pm$0.3 & 92.9$\pm$0.2           & 93.7$\pm$0.3 & 94.4$\pm$0.2 & 95.3$\pm$0.2 & 95.6$\pm$0.2 & 96.1$\pm$0.1 \\
\multicolumn{1}{l}{} & AL (Conf) & 70.7$\pm$1.1 & \underline{87.0$\pm$0.5} & \underline{90.3$\pm$0.5} & 93.1$\pm$0.4 & 94.3$\pm$0.3 & 95.1$\pm$0.2 & 95.5$\pm$0.4 & \underline{95.7$\pm$0.2} & \textbf{96.0$\pm$0.1} & 96.1$\pm$0.1 \\
\multicolumn{1}{l}{} & AL (LL) & 68.4$\pm$1.5 & 85.5$\pm$0.7 & 89.3$\pm$0.6 & 93.1$\pm$0.5 & \underline{94.7$\pm$0.2} & \underline{95.3$\pm$0.2} & \underline{95.6$\pm$0.3} & \textbf{95.8$\pm$0.2} & \textbf{96.0$\pm$0.2} & 96.1$\pm$0.1 \\
\multicolumn{1}{l}{} & AL (Margin) & 71.9$\pm$0.9 & 86.7$\pm$0.5 & 90.1$\pm$0.4 & \underline{93.3$\pm$0.4} & 94.5$\pm$0.3 & 95.1$\pm$0.2 & \underline{95.6$\pm$0.3} & \textbf{95.8$\pm$0.2} & \textbf{96.0$\pm$0.2} & 96.1$\pm$0.1 \\
\multicolumn{1}{l}{} & SSP-Easy & 71.3$\pm$0.5 & 81.5$\pm$2.0 & 87.4$\pm$0.7 & 90.2$\pm$0.3 & 92.0$\pm$0.5 & 93.1$\pm$0.4 & 94.2$\pm$0.2 & 94.9$\pm$0.1 & 95.3$\pm$0.2 & 96.1$\pm$0.1 \\
\multicolumn{1}{l}{} & SSP-Hard & 70.4$\pm$1.7 & 83.0$\pm$0.7 & 87.4$\pm$0.3 & 91.1$\pm$0.2 & 92.9$\pm$0.2 & 93.2$\pm$0.7 & 94.5$\pm$0.2 & 94.9$\pm$0.4 & 95.2$\pm$0.2 & 96.1$\pm$0.1 \\
\multicolumn{1}{l}{} & SP-Easy & 70.0$\pm$1.5 & 82.4$\pm$0.3 & 87.1$\pm$1.3 & 89.9$\pm$0.6& 92.0$\pm$0.4 & 93.4$\pm$0.2 & 94.3$\pm$0.3 & 94.6$\pm$0.2 & 95.4$\pm$0.0 & 96.1$\pm$0.1 \\
\multicolumn{1}{l}{} & SP-Hard & 68.0$\pm$1.2 & 81.6$\pm$0.3 & 87.6$\pm$0.7 & 90.8$\pm$0.8 & 92.7$\pm$0.6 & 93.7$\pm$0.3 & 94.3$\pm$0.2 & 94.8$\pm$0.4 & 95.3$\pm$0.2 & 96.1$\pm$0.1 \\ 
\cmidrule(lr){2-12}
\multicolumn{1}{l}{} & SP-Easy-RS & 89.1$\pm$0.9 & 92.3$\pm$0.1 & 93.2$\pm$0.5 & 93.8$\pm$0.4 & 94.5$\pm$0.3 & 95.0$\pm$0.3 & 94.9$\pm$0.2 & 95.4$\pm$0.3 & 95.7$\pm$0.2 & 96.1$\pm$0.1 \\ 
\cmidrule(lr){2-12}
\multicolumn{1}{l}{} & RS2 w/ repl & 91.7$\pm$0.6 & 93.7$\pm$0.2 & 94.2$\pm$0.6 & \textbf{94.9$\pm$0.2} & \textbf{95.3$\pm$0.2} & 95.0$\pm$0.1 & 95.4$\pm$0.4 & 95.6$\pm$0.2 & \underline{95.9$\pm$0.1} & 96.1$\pm$0.1 \\
\multicolumn{1}{l}{} & RS2 w/ repl (stratified) & 91.7$\pm$0.6 & 93.6$\pm$0.3 & \textbf{94.8$\pm$0.4} & \textbf{94.9$\pm$0.3} & 95.2$\pm$0.1 & \textbf{95.4$\pm$0.3} & 95.5$\pm$0.0 & 95.6$\pm$0.3 & \underline{95.9$\pm$0.2} & 96.1$\pm$0.1 \\
\multicolumn{1}{l}{} & RS2 w/o repl & \textbf{92.0$\pm$0.4} &  \textbf{94.0$\pm$0.4} &  94.6$\pm$0.2 &  94.5$\pm$0.3 &  95.2$\pm$0.3 &  95.3$\pm$0.2 &  \textbf{95.7$\pm$0.1} &  \textbf{95.8$\pm$0.3} &  95.8$\pm$0.1 & 96.1$\pm$0.1 \\
\bottomrule
\end{tabular}
}
\end{table}

We now focus on additional results to accompany the runtime and time-to-accuracy results presented in the main body of the paper. Specifically, in Table~\ref{tab:cifar10_subset_selection}, we show the total time needed for subset selection on CIFAR10 across all rounds for RS2 and compare to the total time needed for subset selection for existing data pruning methods which sample a static subset once before learning begins. In Table~\ref{tab:cifar10_per_epoch_subset_selection} we show the same measurement for our baseline methods which utilize repeated sampling between each round. Note that the differences presented in these tables are the dominant factor leading to differences in end-to-end runtime between methods: Once a subset has been selected for training at each round, all methods train on the same number of examples, and thus have the same per-round training time (assuming there is no noise). Thus the method with the lowest subset selection overhead will also be the fastest method for end-to-end training.

Table~\ref{tab:cifar10_subset_selection} shows that sampling a static random subset once before training leads to the lowest total subset selection time, but that repeated random sampling (RS2) also has low subset selection overhead, i.e., generally less than one second on CIFAR10. The subset selection overhead of RS2 is orders-of-magnitude less than existing methods, even though they sample the subset only once at the beginning of training. For example, most existing methods require over 200 seconds for subset selection because they require pretraining an auxiliary model on the full dataset for a few epochs in order to rank example importance. Some methods, however, require even more time for subset selection; Active Learning based methods can require more than 32,000 seconds to select a subset with $r=50$\%. Once example importance has been calculated, Table~\ref{tab:cifar10_per_epoch_subset_selection} shows that this information can be used to resample the subset for training between each round (our -RS baseline methods, see Section~\ref{subsec:exp_e2e}) with little additional overhead. All such methods, however, still require orders of magnitude more time for subset selection compared to RS2 due to the initial pretraining
\footnote{We note that the pretraining overhead of GraNd in Table~\ref{tab:cifar10_subset_selection} uses the default hyperparameters from~\cite{guo2022deepcore} in which the results from 10 pretrained auxiliary models are averaged, but for GraNd-RS in Table~\ref{tab:cifar10_per_epoch_subset_selection} we use only one model for consistency across all -RS methods.}. 
On the other hand, recomputing the most  important examples between each round (our -RC methods), leads to increased subset selection overhead. The reason for this is that reranking example importance requires computing the model forward pass for all training examples between each round. Thus, such methods generally are unable to significantly reduce the end-to-end runtime compared to simply training on the full dataset each round; Even with a selection ratio of 5\%, the fastest -RC method requires more than 3500 seconds for subset selection, yet end-to-end training, each round on the full dataset, requires only 4500 seconds.

\begin{table}[t]
\caption{Comparison of the total time needed for subset selection for different data pruning methods when training on CIFAR10. Time reported in seconds. The overhead of repeated random sampling is considerably less than existing data pruning methods. For reference, training on the full dataset for 200 epochs takes roughly 4500 seconds. Best method bolded; Next best underlined.}
\vspace{4pt}
\label{tab:cifar10_subset_selection}
\resizebox{\textwidth}{!}{
\begin{tabular}{cccccccccccc}
\toprule
Select Ratio ($r$) & 1\% & 5\% & 10\% & 20\% & 30\% & 40\% & 50\%\\
\midrule
Random & \textbf{0.001$\pm$0.0} & \textbf{0.001$\pm$0.0} & \textbf{0.001$\pm$0.0} &\textbf{ 0.001$\pm$0.0} & \textbf{0.001$\pm$0.0} & \textbf{0.001$\pm$0.0} & \textbf{0.001$\pm$0.0}\\
CD & 237.78$\pm$3.08 &    243.73$\pm$6.06 &  247.01$\pm$13.72 &     244.39$\pm$3.58 &     243.42$\pm$2.18 &     254.46$\pm$9.87 &     254.72$\pm$2.28\\
Herding & 238.29$\pm$3.84 &    241.31$\pm$2.37 &    253.0$\pm$5.84 &    258.49$\pm$12.35 &      255.08$\pm$0.8 &     268.91$\pm$8.13 &     263.16$\pm$2.84 \\
K-Center Greedy & 238.42$\pm$2.75 &    243.16$\pm$0.16 &    243.12$\pm$5.1 &     246.71$\pm$5.54 &     252.07$\pm$3.21 &      260.44$\pm$5.2 &     259.34$\pm$1.68\\
Least Confidence & 238.61$\pm$2.6 &    238.08$\pm$2.59 &   241.96$\pm$5.89 &     239.92$\pm$3.72 &     237.07$\pm$1.64 &     239.47$\pm$6.53 &     240.01$\pm$5.08 \\
Entropy & 239.44$\pm$0.91 &    242.48$\pm$1.41 &   239.57$\pm$5.57 &     242.79$\pm$7.49 &     235.59$\pm$1.44 &     240.67$\pm$2.19 &      239.68$\pm$2.0\\
Margin & 241.71$\pm$2.58 &    245.28$\pm$5.89 &   246.17$\pm$4.65 &      240.83$\pm$1.5 &     243.12$\pm$1.24 &      241.4$\pm$4.05 &     243.45$\pm$1.32\\
Forgetting & 235.34$\pm$1.82 &    238.09$\pm$6.87 &   237.93$\pm$5.57 &     235.44$\pm$1.96 &      234.9$\pm$4.17 &     235.97$\pm$7.32 &     234.78$\pm$1.57 \\
GraNd & 2372.95$\pm$22.89 &  2406.41$\pm$79.95 &  2384.34$\pm$13.3 &   2377.09$\pm$16.19 &   2396.31$\pm$27.83 &   2375.14$\pm$17.84 &   2389.62$\pm$34.32 \\
CAL & 559.68$\pm$1.43 &   562.32$\pm$22.51 &   558.97$\pm$1.96 &     557.83$\pm$10.9 &    568.95$\pm$10.64 &     559.37$\pm$5.04 &     553.13$\pm$9.83 \\
Craig & 296.27$\pm$2.94 &    322.16$\pm$3.83 &   362.8$\pm$10.12 &    438.21$\pm$13.49 &    506.38$\pm$10.17 &     572.06$\pm$2.95 &    642.26$\pm$12.05 \\
Glister & 244.66$\pm$3.99 &    242.02$\pm$4.24 &   247.44$\pm$3.41 &     248.03$\pm$8.07 &     254.79$\pm$1.62 &     259.26$\pm$5.81 &     259.58$\pm$5.36 \\
FL & 330.79$\pm$20.94 &   587.43$\pm$16.07 &  764.18$\pm$86.51 &   1261.9$\pm$165.55 &  1863.98$\pm$241.04 &  2151.46$\pm$435.54 &  2722.44$\pm$145.36 \\
GraphCut & 325.66$\pm$9.37 &   551.81$\pm$46.75 &  728.66$\pm$45.93 &  1251.72$\pm$187.75 &  1601.92$\pm$202.89 &  2335.62$\pm$495.53 &   2672.69$\pm$643.3 \\
AL (Conf) & 408.3$\pm$8.4 & 908.1$\pm$9.8 & 2152.8$\pm$23.7 & 6694.8$\pm$80.6 & 13358.8$\pm$184.5 & 22120.2$\pm$329.5 & 32940.6$\pm$418.7 \\
AL (LL) & 398.5$\pm$5.4 & 879.1$\pm$9.8 & 2087.2$\pm$21.0 & 6592.5$\pm$43.9 & 13206.8$\pm$172.9 & 21933.8$\pm$362.6 & 32763.1$\pm$596.6 \\
AL (Margin) & 396.3$\pm$19.9 & 875.3$\pm$36.2 & 2107.2$\pm$73.6 & 6634.9$\pm$149.1 & 13298.5$\pm$241.1 & 22062.8$\pm$275.7 & 32871.6$\pm$425.9 \\
SSP-Easy & 265.67$\pm$5.87 &    269.89$\pm$8.48 &   268.44$\pm$8.51 &     263.84$\pm$6.49 &     264.85$\pm$5.58 &     268.21$\pm$5.82 &     269.39$\pm$4.78 \\
SSP-Hard & 285.91$\pm$9.25 &    288.47$\pm$8.07 &   290.3$\pm$26.43 &    284.52$\pm$27.16 &    293.41$\pm$21.92 &    287.28$\pm$25.74 &     271.59$\pm$6.09 \\
SP-Easy & 229.04$\pm$2.46 &    231.09$\pm$3.94 &   231.47$\pm$4.36 &     233.61$\pm$6.34 &     233.86$\pm$3.38 &     231.58$\pm$5.65 &      233.6$\pm$4.32 \\
SP-Hard & 227.68$\pm$1.65 &     234.39$\pm$1.2 &   230.85$\pm$3.65 &     227.67$\pm$2.23 &     231.66$\pm$2.64 &      230.97$\pm$3.1 &     233.12$\pm$5.94 \\
\midrule
RS2 w/ repl & \underline{0.16$\pm$0.01} & \underline{0.16$\pm$0.01} & \underline{0.16$\pm$0.01} & \underline{0.16$\pm$0.01} & \underline{0.16$\pm$0.01} & \underline{0.16$\pm$0.01} & \underline{0.16$\pm$0.01} \\
RS2 w/ repl (stratified) & 0.68$\pm$0.02 & 0.72$\pm$0.03 & 0.75$\pm$0.01 & 0.84$\pm$0.03 & 0.93$\pm$0.05 & 1.0$\pm$0.03 & 1.09$\pm$0.03 \\
RS2 w/o repl &  0.09$\pm$0.01 &  0.1$\pm$0.01 &  0.11$\pm$0.01 &  0.12$\pm$0.01 &  0.14$\pm$0.01 &  0.15$\pm$0.01 &  0.19$\pm$0.01 \\
\bottomrule
\end{tabular}
    }
\end{table}

\begin{table}[t]
\caption{Comparison of the total time (in seconds) needed for subset selection for our dynamic data pruning methods when training on CIFAR10. The training subset is update for all methods after each round, either by resampling from a static example importance distribution (RS, left) or by recomputing example importance based on model updates (RC, right). For reference, training on the full dataset for 200 epochs takes roughly 4500 seconds. Best method bolded; Next best underlined.}
\vspace{4pt}
\label{tab:cifar10_per_epoch_subset_selection}
\centering
\resizebox{\textwidth}{!}{
{
\begin{tabular}{llll}
\toprule
Selection Ratio ($r$) & 5\% & 10\% & 30\% \\
\midrule
CD-RS & - & - & - \\
Herding-RS & - & - & -  \\
K-Center Greedy-RS & - & - & - \\
Least Confidence-RS & 238.9$\pm$2.66 & 244.12$\pm$3.5 &         243.75$\pm$9.7 \\
Entropy-RS & 241.26$\pm$2.63 & 240.41$\pm$1.33 &        247.27$\pm$3.19 \\
Margin-RS & 243.46$\pm$3.29 & 239.53$\pm$4.26 &        239.27$\pm$3.36 \\
Forgetting-RS & \underline{236.3$\pm$5.61} & \underline{239.38$\pm$2.36} & \underline{236.24$\pm$4.24} \\
GraNd-RS & 397.92$\pm$0.56 & 409.93$\pm$9.32 &        406.89$\pm$4.74 \\
CAL-RS & 555.75$\pm$20.87 & 549.93$\pm$10.61 &        547.82$\pm$3.33 \\
Craig-RS & 1025.14$\pm$70.46 & 987.05$\pm$16.13 &      1021.92$\pm$13.77 \\
Glister-RS & - & - & - \\
SP-Easy-RS & 326.76 $\pm$ 63.77 & 371.23 $\pm$ 63.71  & 497.39 $\pm$ 8.69 \\ 
\midrule
RS2 w/ repl (stratified) & 0.72 $\pm$ 0.03 & 0.75 $\pm$ 0.01 & 0.93 $\pm$ 0.05 \\
RS2 w/o repl & \textbf{0.1$\pm$0.01} & \textbf{0.11$\pm$0.01} & \textbf{0.14$\pm$0.01} \\
\bottomrule
\end{tabular}
}
{
\begin{tabular}{llll}
\toprule
Selection Ratio ($r$) & 5\% & 10\% & 30\% \\
\midrule
CD-RC & \underline{3581.05$\pm$61.15} & 3860.4$\pm$31.36 &      4860.18$\pm$55.95 \\
Herding-RC & 3851.82$\pm$37.15 & 4332.73$\pm$31.25 &       6578.17$\pm$9.45 \\
K-Center Greedy-RC & 3854.89$\pm$39.84 & 4384.02$\pm$38.72 &      6282.79$\pm$35.03 \\
Least Confidence-RC & 3698.25$\pm$47.31 & \underline{3674.73$\pm$37.97} & \underline{3630.66$\pm$30.48} \\
Entropy-RC & 3651.39$\pm$15.2 & 3677.18$\pm$32.94 &      3690.08$\pm$27.57 \\
Margin-RC & 3715.31$\pm$75.59 & 3686.48$\pm$24.21 &      3760.33$\pm$91.99 \\
Forgetting-RC & 3756.12$\pm$25.54 & 3732.81$\pm$33.9 &      3723.22$\pm$39.52 \\
GraNd-RC & 38035.57$\pm$1212.62 &  37390.35$\pm$939.82 &  29134.04$\pm$16123.62 \\
CAL-RC & 69994.0$\pm$200.65 &  66947.73$\pm$2645.88 &   67086.71$\pm$1213.57 \\
Craig-RC & 20517.31$\pm$955.04 & 27497.62$\pm$359.82 &    55305.63$\pm$988.66 \\
Glister-RC & 4358.65$\pm$56.65 & 4966.57$\pm$20.48 & 6393.25$\pm$33.59  \\
SP-Easy-RC & - & - & - \\
\midrule
RS2 w/ repl (stratified) & 0.72 $\pm$ 0.03 & 0.75 $\pm$ 0.01 & 0.93 $\pm$ 0.05 \\
RS2 w/o repl & \textbf{0.1$\pm$0.01} & \textbf{0.11$\pm$0.01} & \textbf{0.14$\pm$0.01} \\
\bottomrule
\end{tabular}
}
}
\end{table}

Finally, as our primary focus is on reducing time-to-accuracy, we include in Tables~\ref{tab:time_to_acc_cifar_p1}-\ref{tab:time_to_acc_imagenet_p2} the time for select baseline methods and RS2 to reach a set of accuracy targets when training with varying pruning ratios on CIFAR10 and ImageNet. For the active learning time-to-accuracy results in these tables, we report the runtime of the smallest selection ratio that reached the given accuracy. This prevents active learning time-to-accuracies from being dominated by large subset selection overheads as the selection ratio increases (e.g., Table~\ref{tab:cifar10_subset_selection}), when these selection ratios are not strictly needed to reach the desired accuracy. As shown in the main body of the paper, RS2 provides the fastest time-to-accuracy compared to existing methods. Dashes indicate that the given method and pruning ratio failed to reach the target accuracy. We leave a detailed study of these results for future work. In particular, an interesting question is how to decide what pruning ratio $r$ one should use in order to minimize runtime to reach a desired accuracy.

\begin{table}[t]
\centering
\caption{The total time required for RS2 and baseline data pruning methods to reach a target accuracy (time-to-accuracy) when training with varying pruning ratios on CIFAR10. Time is reported in seconds. Part 1/3. The best method(s) is bolded.}
\vspace{4pt}
\label{tab:time_to_acc_cifar_p1}
\resizebox{0.95\columnwidth}{!}{
\begin{tabular}{ccccccccccccc}
\toprule
Target                     & Select Ratio ($r$) & 1\%          & 5\%         & 10\%        & 20\%        & 30\%        & 40\%        & 50\%        & 60\%        & 70\%        & 80\%        & 90\%        \\ \midrule
\multirow{25}{*}{30\% acc} & Random             & 73           & \textbf{16} & 28          & \textbf{23} & \textbf{15} & \textbf{20} & \textbf{13} & \textbf{14} & \textbf{16} & \textbf{18} & \textbf{20} \\
                        & CD                 & -            & 332         & 372         & 289         & 297         & 284         & 290         & 283         & 286         & 294         & 283         \\
                           & Herding            & -            & 295         & 296         & 275         & 283         & 279         & 273         & 278         & 281         & 285         & 291         \\
                           & K-Center Greedy        & 353          & 284         & 304         & 273         & 267         & 280         & 283         & 281         & 283         & 290         & 296         \\
                           & Least Confidence    & -            & -           & 393         & 330         & 299         & 278         & 275         & 265         & 259         & 262         & 264         \\
                           & Entropy            & -            & 484         & 398         & 346         & 289         & 270         & 274         & 267         & 256         & 259         & 262         \\
                           & Margin             & -            & 324         & 381         & 297         & 297         & 271         & 279         & 264         & 272         & 259         & 263         \\
                           & Forgetting         & 302          & 263         & 299         & 264         & 258         & 264         & 258         & 249         & 257         & 252         & 253         \\
                           & GraNd              & -            & 2485        & 2461        & 2422        & 2428        & 2404        & 2425        & 2397        & 2376        & 2471        & 2387        \\
                           & CAL                & -            & 596         & 597         & 574         & 385         & 579         & 565         & 605         & 605         & 591         & 595         \\
                           & Craig              & 365          & 342         & 395         & 461         & 521         & 583         & 668         & 740         & 816         & 909         & 952         \\
                           & Glister            & -            & 280         & 274         & 264         & 270         & 278         & 283         & 278         & 280         & 291         & 293         \\
                           & FL                 & 377          & 605         & 803         & 1279        & 1878        & 2179        & 2734        & 3239        & 3671        & 4696        & 4513        \\
                           & GraphCut                & 371          & 567         & 771         & 1268        & 1617        & 2355        & 2684        & 3462        & 3771        & 4427        & 4458        \\
                           & AL (Conf)           & 395          & 395         & 395         & 395         & 395         & 395         & 395         & 395         & 395         & 395         & 395         \\
                           & AL (LL)             & 386          & 386         & 386         & 386         & 386         & 386         & 386         & 386         & 386         & 386         & 386         \\
                           & AL (Margin)         & 383          & 383         & 383         & 383         & 383         & 383         & 383         & 383         & 383         & 383         & 383         \\
                           & SSP-Easy           & 320          & 295         & 300         & 280         & 288         & 287         & 281         & 282         & 282         & 287         & 289         \\
                           & SSP-Hard           & 403          & 325         & 330         & 307         & 309         & 308         & 283         & 283         & 286         & 293         & 289         \\
                           & SP-Easy            & 295          & 246         & 258         & 249         & 246         & 254         & 240         & 243         & 249         & 243         & 253         \\
                           & SP-Hard            & 315          & 260         & 262         & 250         & 247         & 250         & 257         & 250         & 251         & 248         & 252         \\
                           \cmidrule{2-13}
                           & SP-Easy-RS         & 284          & 249         & 256         & 262         & 249         & 250         & 257         & 249         & 250         & 249         & 252         \\
                           \cmidrule{2-13}
                           & RS2 w/ repl               & 67           & \textbf{19} & 36          & 21 & \textbf{15} & \textbf{19} & \textbf{12} & \textbf{14} & \textbf{16} & \textbf{18} & \textbf{20} \\
                           & RS2 w/ repl (stratified)             & \textbf{59}  & \textbf{20} & 43          & 27 & \textbf{15} & \textbf{19} & \textbf{12} & \textbf{14} & \textbf{15} & \textbf{18} & \textbf{20} \\
                           & RS2 w/o repl              & 83           & \textbf{17} & \textbf{14} & \textbf{11} & \textbf{15} & \textbf{19} & \textbf{11} & \textbf{14} & \textbf{16} & \textbf{18} & \textbf{20} \\ \midrule
\multirow{25}{*}{50\% acc} & Random             & -            & 78 & 109         & 65 & 55 & 69          & 49          & 56          & 48 & \textbf{54} & 61          \\
                            & CD                 & -            & -           & 617         & 354         & 366         & 384         & 361         & 325         & 318         & 349         & 304         \\
                           & Herding            & -            & -           & 564         & 399         & 345         & 319         & 322         & 320         & 314         & 323         & 311         \\
                          & K-Center Greedy        & -            & 390         & 397         & 322         & 314         & 309         & 318         & 323         & 331         & 327         & 317         \\
                          & Least Confidence    & -            & -           & -           & 438         & 383         & 357         & 347         & 321         & 325         & 335         & 306         \\
                           & Entropy            & -            & -           & -           & 513         & 412         & 339         & 345         & 324         & 321         & 314         & 303         \\
                           & Margin             & -            & -           & 755         & 439         & 407         & 359         & 351         & 305         & 305         & 314         & 304         \\
                           & Forgetting         & -            & -           & 455         & 366         & 350         & 333         & 306         & 305         & 306         & 306         & 295         \\
                           & GraNd              & -            & -           & 2811        & 2550        & 2531        & 2484        & 2498        & 2453        & 2408        & 2509        & 2428        \\
                           & CAL               & -            & 881         & 764         & 652         & 633         & 630         & 612         & 648         & 638         & 628         & 635         \\
                           & Craig              & -            & 458         & 523         & 537         & 568         & 633         & 704         & 782         & 885         & 946         & 993         \\
                           & Glister            & -            & 393         & 372         & 313         & 308         & 319         & 307         & 334         & 312         & 328         & 313         \\
                           & FL                 & -            & 695         & 909         & 1341        & 1917        & 2219        & 2771        & 3267        & 3704        & 4733        & 4535        \\
                           & GraphCut                 & -            & 623         & 876         & 1321        & 1664        & 2395        & 2720        & 3506        & 3805        & 4463        & 4478        \\
                           & AL (Conf)           & -            & 895         & 895         & 895         & 895         & 895         & 895         & 895         & 895         & 895         & 895         \\
                           & AL (LL)             & -            & 867         & 867         & 867         & 867         & 867         & 867         & 867         & 867         & 867         & 867         \\
                           & AL (Margin)         & -            & 862         & 862         & 862         & 862         & 862         & 862         & 862         & 862         & 862         & 862         \\
                           & SSP-Easy           & -            & 375         & 394         & 336         & 333         & 336         & 316         & 323         & 314         & 324         & 329         \\
                           & SSP-Hard           & -            & 432         & 433         & 365         & 357         & 358         & 332         & 326         & 335         & 330         & 329         \\
                           & SP-Easy            & -            & 340         & 362         & 306         & 301         & 284         & 287         & 271         & 282         & 261         & 293         \\
                           & SP-Hard            & -            & 277         & 354         & 312         & 301         & 290         & 306         & 294         & 284         & 285         & 291         \\
                            \cmidrule{2-13}
                           & SP-Easy-RS       & -            & 308         & 355         & 335         & 289         & 298         & 304         & 290         & 282         & 267         & 273         \\
                           \cmidrule{2-13}
                           & RS2 w/ repl               & 279          & 77 & 137         & 66 & 61 & \textbf{48} & 47          & 56          & 48 & \textbf{55} & 61          \\
                           & RS2 w/ repl (stratified)             & \textbf{250} & 87 & 123         & 83          & 69          & 60          & 72          & 57          & 64          & \textbf{54} & \textbf{40} \\
                           & RS2 w/o repl              & -            & \textbf{64} & \textbf{44} & \textbf{46} & \textbf{46} & 58          & \textbf{35} & \textbf{41} & \textbf{32} & \textbf{53} & 60          \\ \bottomrule
\end{tabular}
}
\end{table}

\begin{table}[h]
\centering
\caption{The total time required for RS2 and baseline data pruning methods to reach a target accuracy (time-to-accuracy) when training with varying pruning ratios on CIFAR10. Time is reported in seconds. Part 2/3. The best method(s) is bolded.}
\vspace{4pt}
\label{tab:time_to_acc_cifar_p2}
\resizebox{0.95\columnwidth}{!}{
\begin{tabular}{ccccccccccccc}
\toprule
Target                     & Select Ratio ($r$)   & 1\% & 5\%          & 10\%         & 20\%         & 30\%         & 40\%         & 50\%         & 60\%         & 70\%         & 80\%         & 90\%         \\ \midrule
\multirow{25}{*}{70\% acc} & Random          & -   & -            & 364          & 216          & 184          & 139 & 136          & 139          & 130          & 126 & 143          \\
                        & CD              & -   & -            & -            & 584          & 504          & 473          & 456          & 396          & 382          & 421          & 406          \\
                           & Herding         & -   & -            & -            & 926          & 760          & 456          & 406          & 418          & 395          & 415          & 393          \\
                           & K-Center Greedy     & -   & -            & 552          & 420          & 384          & 397          & 414          & 380          & 412          & 493          & 379          \\
                           & Least Confidence & -   & -            & -            & 911          & 514          & 455          & 453          & 420          & 374          & 372          & 369          \\
                           & Entropy         & -   & -            & -            & 954          & 565          & 457          & 464          & 423          & 371          & 387          & 385          \\
                           & Margin          & -   & -            & -            & 705          & 564          & 479          & 473          & 387          & 402          & 387          & 366          \\
                           & Forgetting      & -   & -            & -            & 549          & 473          & 430          & 366          & 376          & 374          & 360          & 356          \\
                           & GraNd           & -   & -            & -            & 2727         & 2689         & 2573         & 2583         & 2523         & 2489         & 2584         & 2489         \\
                           & CAL             & -   & -            & -            & 1025         & 801          & 763          & 720          & 748          & 719          & 720          & 695          \\
                           & Craig           & -   & -            & 775          & 664          & 693          & 724          & 803          & 854          & 966          & 1020         & 1075         \\
                           & Glister         & -   & -            & 601          & 429          & 409          & 418          & 390          & 406          & 393          & 383          & 396          \\
                           & FL              & -   & -            & 1144         & 1467         & 2009         & 2299         & 2855         & 3326         & 3785         & 4807         & 4597         \\
                           & GraphCut              & -   & -            & 1136         & 1465         & 1734         & 2484         & 2804         & 3578         & 3889         & 4536         & 4538         \\
                           & AL (Conf)        & -   & -            & 2884         & 2884         & 2884         & 2884         & 2884         & 2884         & 2884         & 2884         & 2884         \\
                           & AL (LL)          & -   & -            & 2804         & 2804         & 2804         & 2804         & 2804         & 2804         & 2804         & 2804         & 2804         \\
                           & AL (Margin)      & -   & -            & 2833         & 2833         & 2833         & 2833         & 2833         & 2833         & 2833         & 2833         & 2833         \\
                           & SSP-Easy        & -   & -            & 716          & 445          & 440          & 404          & 412          & 407          & 410          & 399          & 390          \\
                           & SSP-Hard        & -   & -            & 689          & 532          & 461          & 438          & 392          & 411          & 416          & 404          & 411          \\
                           & SP-Easy         & -   & -            & 705          & 431          & 409          & 403          & 370          & 355          & 363          & 332          & 373          \\
                           & SP-Hard         & -   & -            & 586          & 428          & 379          & 398          & 402          & 380          & 365          & 357          & 372          \\
                            \cmidrule{2-13}
                           & SP-Easy-RS      & -   & 440          & 508          & 415          & 383          & 374          & 377          & 358          & 378          & 339          & 354          \\
                           \cmidrule{2-13}
                           & RS2 w/ repl            & -   & \textbf{218} & 271          & 205          & 162          & 147          & 130          & 154          & \textbf{112} & 128          & 142          \\
                           & RS2 w/ repl (stratified)           & -   & \textbf{213} & 260          & 179          & 155          & 151          & 144          & \textbf{126} & 145          & 126          & \textbf{121} \\
                           & RS2 w/o repl           & -   & \textbf{201} & \textbf{168} & \textbf{128} & \textbf{116} & \textbf{127} & \textbf{105} & \textbf{125} & \textbf{112} & \textbf{107} & \textbf{122} \\ \midrule
\multirow{25}{*}{80\% acc} & Random          & -   & -            & -            & 553          & 470          & 395          & 303          & \textbf{266}          & 277          & \textbf{199} & 306          \\
                            & CD              & -   & -            & -            & 917          & 713          & 654          & 601          & 494          & 495          & 567          & 508          \\
                           & Herding         & -   & -            & -            & -            & -            & 1522         & 1222         & 909          & 686          & 598          & 536          \\
                           & K-Center Greedy     & -   & -            & -            & 692          & 695          & 635          & 606          & 535          & 557          & 511          & 563          \\
                           & Least Confidence & -   & -            & -            & -            & 738          & 622          & 621          & 587          & 537          & 518          & 534          \\
                           & Entropy         & -   & -            & -            & -            & 912          & 675          & 584          & 521          & 516          & 498          & 529          \\
                           & Margin          & -   & -            & -            & -            & 975          & 678          & 606          & 538          & 548          & 496          & 509          \\
                           & Forgetting      & -   & -            & -            & 719          & 641          & 616          & 511          & 518          & 541          & 451          & 538          \\
                           & GraNd           & -   & -            & -            & 2984         & 2831         & 2671         & 2681         & 2636         & 2585         & 2716         & 2650         \\
                           & CAL             & -   & -            & -            & -            & 1378         & 1132         & 995          & 978          & 968          & 921          & 816          \\
                           & Craig           & -   & -            & -            & 1057         & 1063         & 1075         & 986          & 1111         & 1116         & 1242         & 1218         \\
                           & Glister         & -   & -            & -            & 788          & 650          & 687          & 614          & 604          & 586          & 494          & 559          \\
                           & FL              & -   & -            & -            & 1765         & 2331         & 2525         & 3122         & 3528         & 3997         & 4936         & 4742         \\
                           & GraphCut              & -   & -            & -            & 1779         & 2068         & 2681         & 3049         & 3782         & 4106         & 4735         & 4697         \\
                           & AL (Conf)        & -   & -            & 2884         & 2884         & 2884         & 2884         & 2884         & 2884         & 2884         & 2884         & 2884         \\
                           & AL (LL)          & -   & -            & 2804         & 2804         & 2804         & 2804         & 2804         & 2804         & 2804         & 2804         & 2804         \\
                           & AL (Margin)      & -   & -            & 2833         & 2833         & 2833         & 2833         & 2833         & 2833         & 2833         & 2833         & 2833         \\
                           & SSP-Easy        & -   & -            & -            & 814          & 712          & 639          & 638          & 629          & 553          & 583          & 592          \\
                           & SSP-Hard        & -   & -            & -            & 840          & 786          & 739          & 584          & 625          & 626          & 570          & 573          \\
                           & SP-Easy         & -   & -            & -            & 839          & 702          & 699          & 617          & 466          & 461          & 493          & 595          \\
                           & SP-Hard         & -   & -            & -            & 766          & 673          & 605          & 558          & 578          & 496          & 503          & 514          \\
                        \cmidrule{2-13}
                           & SP-Easy-RS     & -   & 564          & 634          & 637          & 585          & 501          & 590          & 567          & 490          & 538          & 516          \\
                           \cmidrule{2-13}
                           & RS2 w/ repl            & -   & \textbf{331} & 383          & 334          & 308          & 343          & 274          & 293          & 337          & 312          & 324          \\
                           & RS2 w/ repl (stratified)           & -   & \textbf{331} & 339          & 320          & 317          & 341          & 289          & 295          & 275          & 272          & {303}        \\
                           & RS2 w/o repl           & -   & \textbf{333} & \textbf{278} & \textbf{257} & \textbf{278} & \textbf{264} & \textbf{212} & \textbf{264} & \textbf{241} & 270          & \textbf{264} \\ \bottomrule
\end{tabular}
}
\end{table}

\begin{table}[h]
\centering
\caption{The total time required for RS2 and baseline data pruning methods to reach a target accuracy (time-to-accuracy) when training with varying pruning ratios on CIFAR10. Time is reported in seconds. Part 3/3. The best method(s) is bolded.}
\vspace{4pt}
\label{tab:time_to_acc_cifar_p3}
\resizebox{0.95\columnwidth}{!}{
\begin{tabular}{ccccccccccccc}
\toprule
Target                     & Select Ratio ($r$)   & 1\% & 5\% & 10\%         & 20\%         & 30\%          & 40\%          & 50\%          & 60\%          & 70\%          & 80\%          & 90\%          \\ \midrule
\multirow{25}{*}{90\% acc} & Random          & -   & -   & -            & -            & 1473          & 1462          & 1722          & 1906          & 2037          & 2149          & 2403          \\
                        & CD              & -   & -   & -            & -            & 1327          & 1615          & 1806          & 1917          & 2002          & 2181          & 2461          \\
                           & Herding         & -   & -   & -            & -            & -             & -             & 2736          & 2801          & 2810          & 2885          & 2839          \\
                           & K-Center Greedy     & -   & -   & -            & -            & 1412          & 1662          & 1913          & 2093          & 2310          & 2433          & 2433          \\
                           & Least Confidence & -   & -   & -            & -            & 1446          & 1546          & 1746          & 1940          & 2080          & 2268          & 2423          \\
                           & Entropy         & -   & -   & -            & -            & -             & 1542          & 1783          & 1892          & 2024          & 2255          & 2468          \\
                           & Margin          & -   & -   & -            & -            & -             & 1554          & 1749          & 1877          & 2020          & 2151          & 2498          \\
                           & Forgetting      & -   & -   & -            & -            & 1359          & 1569          & 1699          & 1969          & 2228          & 2334          & 2562          \\
                           & GraNd           & -   & -   & -            & -            & 3530          & 3645          & 3845          & 3889          & 4042          & 4465          & 4531          \\
                           & CAL             & -   & -   & -            & -            & -             & 2201          & 2292          & 2459          & 2658          & 2692          & 3042          \\
                           & Craig           & -   & -   & -            & -            & 1798          & 2024          & 2381          & 2613          & 2982          & 3134          & 3030          \\
                           & Glister         & -   & -   & -            & -            & 1718          & 1734          & 1885          & 2222          & 2368          & 2485          & 2639          \\
                           & FL              & -   & -   & -            & -            & 3203          & 3591          & 4454          & 5229          & 5715          & 6903          & 6707          \\
                           & GraphCut              & -   & -   & -            & -            & 2846          & 3782          & 4428          & 5387          & 5860          & 6459          & 6782          \\
                           & AL (Conf)        & -   & -   & -            & 7845         & 7845          & 7845          & 7845          & 7845          & 7845          & 7845          & 7845          \\
                           & AL (LL)          & -   & -   & -            & 7740         & 7740          & 7740          & 7740          & 7740          & 7740          & 7740          & 7740          \\
                           & AL (Margin)      & -   & -   & -            & 7787         & 7787          & 7787          & 7787          & 7787          & 7787          & 7787          & 7787          \\
                           & SSP-Easy        & -   & -   & -            & -            & 1534          & 1719          & 1932          & 2152          & 2329          & 2582          & 2546          \\
                           & SSP-Hard        & -   & -   & -            & -            & 1740          & 1760          & 1961          & 2155          & 2273          & 2452          & 2499          \\
                           & SP-Easy         & -   & -   & -            & -            & -             & 1699          & 1936          & 2126          & 2324          & 2323          & 2647          \\
                           & SP-Hard         & -   & -   & -            & -            & -             & 1712          & 1906          & 2157          & 2206          & 2407          & 2650          \\
                            \cmidrule{2-13}
                           & SP-Easy-RS     & -   & -   & -            & 1080         & 1284          & 1505          & 1721          & 1962          & 2173          & 2365          & 2619          \\
                           \cmidrule{2-13}
                           & RS2 w/ repl            & -   & -   & -            & 777 & 1028 & \textbf{1220} & {1435}        & 1645 & 1893 & \textbf{1979} & 2378          \\
                           & RS2 w/repl (stratified)           & -   & -   & -            & 785 & 995  & 1267 & {1483}        & \textbf{1577} & \textbf{1786} & 2099 & \textbf{2127} \\
                           & RS2 w/o repl           & -   & -   & \textbf{566} & \textbf{723} & \textbf{953}  & \textbf{1211} & \textbf{1291} & 1637 & 1866 & 2106          & 2357          \\ \midrule
\multirow{25}{*}{95\% acc} & Random          & -   & -   & -            & -            & -             & -             & -             & -             & -             & -             & 3582          \\
                            & CD              & -   & -   & -            & -            & -             & -             & -             & -             & 3124          & 3416          & 3784          \\
                           & Herding         & -   & -   & -            & -            & -             & -             & -             & -             & -             & -             & 3815          \\
                           & K-Center Greedy     & -   & -   & -            & -            & -             & -             & -             & -             & -             & -             & 3952          \\
                           & Least Confidence & -   & -   & -            & -            & -             & -             & -             & -             & 3084          & 3432          & 3838          \\
                           & Entropy         & -   & -   & -            & -            & -             & -             & -             & 2900          & 3171          & 3464          & 3822          \\
                            & Margin          & -   & -   & -            & -            & -             & -             & -             & -             & 3057          & 3522          & 3797          \\
                           & Forgetting      & -   & -   & -            & -            & -             & -             & -             & -             & -             & -             & 3835          \\
                           & GraNd           & -   & -   & -            & -            & -             & -             & -             & -             & 5133          & 5721          & 5956          \\
                           & CAL             & -   & -   & -            & -            & -             & -             & -             & -             & -             & -             & 4153          \\
                           & Craig           & -   & -   & -            & -            & -             & -             & -             & -             & -             & -             & -             \\
                           & Glister         & -   & -   & -            & -            & -             & -             & -             & -             & -             & -             & 4012          \\
                           & FL              & -   & -   & -            & -            & -             & -             & -             & -             & -             & -             & 8221          \\
                           & GraphCut              & -   & -   & -            & -            & -             & -             & -             & -             & -             & -             & 8028          \\
                           & AL (Conf)        & -   & -   & -            & -            & -             & -             & -             & 48708         & 48708         & 48708         & 48708         \\
                           & AL (LL)          & -   & -   & -            & -            & -             & -             & -             & 48549         & 48549         & 48549         & 48549         \\
                           & AL (Margin)      & -   & -   & -            & -            & -             & -             & -             & 48326         & 48326         & 48326         & 48326         \\
                           & SSP-Easy        & -   & -   & -            & -            & -             & -             & -             & -             & -             & -             & -             \\
                           & SSP-Hard        & -   & -   & -            & -            & -             & -             & -             & -             & -             & -             & 3908          \\
                           & SP-Easy         & -   & -   & -            & -            & -             & -             & -             & -             & -             & -             & 4214          \\
                           & SP-Hard         & -   & -   & -            & -            & -             & -             & -             & -             & -             & -             & 3922          \\
                            \cmidrule{2-13}
                           & SP-Easy-RS      & -   & -   & -            & -            & -             & -             & -             & -             & 3180          & 3505          & 3810          \\
                           \cmidrule{2-13}
                           & RS2 w/ repl            & -   & -   & -            & -            & -             & -             & 2296 & \textbf{2498} & \textbf{2856} & 3245 & 3696          \\
                           & RS2 w/repl (stratified)           & -   & -   & -            & -            & -             & -             & -             & 2585 & 3003          & 3300          & \textbf{3501} \\
                           & RS2 w/o repl           & -   & -   & -            & -            & -             & -             & \textbf{2153} & 2633 & 2943 & \textbf{3147} & 3569 \\ \bottomrule
\end{tabular}
}
\end{table}

\begin{table}[h]
\centering
\caption{The total time required for RS2 and select baseline data pruning methods to reach a target accuracy (time-to-accuracy) when training with varying pruning ratios on ImageNet. Time is reported in seconds. Part 1/2. The best method(s) is bolded.}
\vspace{4pt}
\label{tab:time_to_acc_imagenet_p1}
\resizebox{.30\paperheight}{!}{
\begin{tabular}{ccccc}
\toprule
Target                     & Select Ratio ($r$)    & 1\%           & 5\%           & 10\%          \\ \midrule
\multirow{12}{*}{5\% acc}  & Random          & 919           & \textbf{289}           & \textbf{231}           \\
                            & Herding         & 22222         & 27794         & 24243         \\
                            & Least Confidence & -             & 15292         & 19362         \\
                           & Entropy         & -             & 19535         & 19285         \\
                           & Margin          & 19494         & 14121         & 15785         \\
                            & Forgetting      & 19081         & 18394         & 15978         \\
                           & GraNd           & 146059        & 143640        & 147931        \\
                           & FL              & 67907         & 67789         & 220647        \\
                           & GraphCut              & 66336         & 224755        & 303318        \\
                            \cmidrule{2-5}
                           & RS2 w/ repl            & 542           & 347  & \textbf{231}  \\
                           & RS2 w/ repl (stratified)           & \textbf{403}  & \textbf{291}           & 350           \\
                           & RS2 w/o repl           & 530           & 347  & 347           \\ \midrule
\multirow{12}{*}{10\% acc} & Random          & 2003          & \textbf{462}           & 579           \\
                            & Herding         & -             & 29468         & 24822         \\
                            & Least Confidence & -             & 16043         & 20056         \\
                           & Entropy         & -             & 20517         & 19864         \\
                           & Margin          & -             & 14698         & 16364         \\
                            & Forgetting      & 19799         & 18625         & 16094         \\
                           & GraNd           & 14743         & 143755        & 148047        \\
                           & FL              & 69144         & 220878        & 308052        \\
                           & GraphCut              & 66831         & 224928        & 303434        \\
                            \cmidrule{2-5}
                           & RS2 w/ repl            & 931           & 635  & \textbf{463}  \\
                           & RS2 w/ repl (stratified)           & \textbf{686}  & 638           & 581           \\
                           & RS2 w/o repl           & 907           & 635  & \textbf{463}  \\ \midrule
\multirow{12}{*}{20\% acc} & Random          & -             & 1906          & \textbf{1157}          \\
                            & Herding         & -             & -             & 36974         \\
                            & Least Confidence & -             & 22684         & 23065         \\
                           & Entropy         & -             & 27158         & 22757         \\
                           & Margin          & -             & 19029         & 17984         \\
                            & Forgetting      & -             & 20877         & 16557         \\
                           & GraNd           & -             & 146123        & 148741        \\
                           & FL              & -             & 221802        & 308630        \\
                           & GraphCut              & -             & 225794        & 304012        \\
                            \cmidrule{2-5}
                           & RS2 w/ repl            & 1355          & 1502          & 1273          \\
                           & RS2 w/ repl (stratified)           & \textbf{1145} & \textbf{1388}          & \textbf{1160}          \\
                           & RS2 w/o repl           & 1308          & \textbf{1386} & \textbf{1157} \\ \midrule
\multirow{12}{*}{30\% acc} & Random          & -             & 7219          & 4282          \\
                            & Herding         & -             & -             & -             \\
                            & Least Confidence & -             & -             & 32555         \\
                           & Entropy         & -             & -             & 32131         \\
                           & Margin          & -             & 43750         & 25159         \\
                            & Forgetting      & -             & 25959         & 18293         \\
                           & GraNd           & -             & 151840        & 158347        \\
                           & FL              & -             & 227519        & 313954        \\
                           & GraphCut              & -             & 307600        & 307600        \\
                            \cmidrule{2-5}
                           & RS2 w/ repl            & 1626          & 3927          & 4514          \\
                           & RS2 w/ repl (stratified)           & \textbf{1546} & \textbf{3583} & 3937 \\
                           & RS2 w/o repl           & 1673          & 4100          & \textbf{3166}          \\ \bottomrule
\end{tabular}
}
\end{table}

\begin{table}[h]
\centering
\caption{The total time required for RS2 and select baseline data pruning methods to reach a target accuracy (time-to-accuracy) when training with varying pruning ratios on ImageNet. Time is reported in seconds. Part 2/2. The best method(s) is bolded.}
\vspace{4pt}
\label{tab:time_to_acc_imagenet_p2}
\resizebox{.30\paperheight}{!}{
\begin{tabular}{ccccc}
\toprule
Target                     & Select Ratio ($r$)    & 1\%           & 5\%            & 10\%           \\ \midrule
\multirow{12}{*}{40\% acc} & Random          & -             & -              & 14814          \\
                        & Herding         & -             & -              & -              \\
                        & Least Confidence & -             & -              & -              \\
                           & Entropy         & -             & -              & -              \\
                           & Margin          & -             & -              & 32682          \\
                        & Forgetting      & -             & -              & 28824          \\
                           & GraNd           & -             & -              & 166680         \\
                           & FL              & -             & -              & 322518         \\
                           & GraphCut              & -             & 234399         & 317553         \\
                            \cmidrule{2-5}
                           & RS2 w/ repl            & 2015          & 6815           & \textbf{11689} \\
                           & RS2 w/ repl (stratified)           & \textbf{1864} & \textbf{6586}  & 11923          \\
                           & RS2 w/o repl           & 2015          & 6699           & 12152          \\ \midrule
\multirow{12}{*}{50\% acc} 
                            & Random          & -             & -              & -              \\
                            & Herding         & -             & -              & -              \\
                            & Least Confidence & -             & -              & -              \\
                           & Entropy         & -             & -              & -              \\
                           & Margin          & -             & -              & -              \\
                            & Forgetting      & -             & -              & 35305          \\
                           & GraNd           & -             & -              & -              \\
                           & FL              & -             & -              & -              \\
                           & GraphCut              & -             & -              & -              \\
                             \cmidrule{2-5}
                           & RS2 w/ repl            & -             & 8605           & 16550          \\
                           & RS2 w/ repl (stratified)           & -             & \textbf{8492}  & 16552          \\
                           & RS2 w/o repl           & -             & 8605           & \textbf{16434} \\ \midrule
\multirow{12}{*}{60\% acc} & Random          & -             & -              & -              \\
                            & Herding         & -             & -              & -              \\
                            & Least Confidence & -             & -              & -              \\
                           & Entropy         & -             & -              & -              \\
                           & Margin          & -             & -              & -              \\
                            & Forgetting      & -             & -              & -          \\
                           & GraNd           & -             & -              & -              \\
                           & FL              & -             & -              & -              \\
                           & GraphCut              & -             & -              & -              \\
                            \cmidrule{2-5}
                           & RS2 w/ repl            & -             & 10337          & 20021          \\
                           & RS2 w/ repl (stratified)           & -             & \textbf{10282}          & 20024          \\
                           & RS2 w/o repl           & -             & \textbf{10280} & \textbf{19790} \\ \midrule
\multirow{12}{*}{65\% acc} & Random          & -             & -              & -              \\
                            & Herding         & -             & -              & -              \\
                            & Least Confidence & -             & -              & -              \\
                           & Entropy         & -             & -              & -              \\
                           & Margin          & -             & -              & -              \\
                            & Forgetting      & -             & -              & -          \\
                           & GraNd           & -             & -              & -              \\
                           & FL              & -             & -              & -              \\
                           & GraphCut              & -             & -              & -              \\
                            \cmidrule{2-5}
                           & RS2 w/ repl            & -             & -              & 22105          \\
                           & RS2 w/ repl (stratified)           & -             & -              & 22107          \\
                           & RS2 w/o repl           & -             & -              & \textbf{21873}          \\ \bottomrule
\end{tabular}
}
\end{table}

\clearpage
\section{Additional RS2 Pseudocode}\label{section_pseudo_app}
In this section, we include additional RS2 pseudocode algorithms to accompany Algorithm~\ref{alg:general_rs2} presented in the main body of the paper and to present additional details useful for the RS2 theoretical analysis.
 
In Algorithm~\ref{alg:minibatch_sgd} we describe RS2 without replacement when training with accelerated mini batch SGD~\cite{ghadimi2016accelerated,nesterov1983method}. Nesterov's accelerated gradient introduces three different sets of parameters that are updated at each iteration $t$. We denote them as $w^t, w^t_{ag},$ and $w^t_{md}$. Furthermore, the algorithm introduces learning rate parameters $\alpha_t, \beta_t,$ and $\lambda_t$. In later sections, we specialize the learning rate parameters for obtaining a convergence rate bound. Finally, $g (w, \xi_t; m)$ at step $t$ represents the gradient estimate on a batch of data $m$ that is used for updating the model. $\xi_t$ are random vectors whose distributions are supported on $\Xi_t \in \mathbb{R}^d$.

\begin{algorithm}
\caption{RS2 w/o Replacement With Accelerated Mini batch SGD}
\label{alg:minibatch_sgd}
\begin{algorithmic}[1]
\Require Dataset $S = \{x_i, y_i\}_{i=1}^N$, selection ratio $r\in(0, 1]$, batch size $b$, initial model $w^0$, $X$ rounds, learning rate parameters $\{\alpha_t\} \: \text{s.t.} \: \alpha_1 = 1, \alpha_t \in (0,1), \forall t \geq 2$, $\{\beta_t > 0\}$, $\{\lambda_t > 0\}$, gradient estimate function for batch $m$ and parameters $w$ with noise $\xi$: $g (w, \xi; m)$
\State $T \gets \lceil N/b \rceil$
\State $t \gets 1$
\State $w^0 _{ag} = w^0$
\For{round $j=1$ to $X$}
\If{$t\%T == 0$} \Comment{Shuffle after full dataset has been seen}
\State $shuffle(S)$
\EndIf
\State $S' \gets S[(j-1) \cdot rN : j \cdot rN]$ \Comment{Select the subset across rounds without replacement}
\For{$k=1$ to $r\cdot T$}
\State batch $m \gets S'[(k-1) \cdot b:k \cdot b]$

\State $w^t_{md} \gets (1 - \alpha_t) w^{t-1}_{ag} + \alpha_t w^{t-1}$
\State $w^t \gets w^{t-1} - \lambda_t g(w^t_{md}, \xi_t; m)$
\hspace*{1em}\rlap{\smash{$\left.\begin{array}{@{}c@{}}\\{}\\{}\\\end{array}\right\}\begin{tabular}{l} \Comment{$train\_on\_batch$ for Nesterov mini batch SGD}\end{tabular}$}}
\State $w^t_{ag} \gets w^t_{md} - \beta_t g(w^t_{md}, \xi_t; m)$

\State $t \gets t+1$
\EndFor
\EndFor
\Return $w^t_{md}$
\end{algorithmic}
\end{algorithm}

In Algorithm~\ref{alg:minibatch_sgd_no_momentum}, we also show RS2 without replacement, but using standard mini batch SGD. We also write this algorithm using a different perspective: instead of iterating over rounds, selecting the training subset for each round, and then iterating over batches in the selected subset, RS2 without replacement can be equivalently implemented by iterating directly over batches from the full dataset, as long as these batches are correctly selected and the full dataset is shuffled as necessary. This perspective can be more useful for understanding the generalization error of RS2 without replacement as it more closely matches the common algorithms in related works analyzing SGD~\cite{ghadimi2016accelerated,nikolakakis2023select}.

\begin{algorithm}
\caption{RS2 w/o Replacement With Mini batch SGD; Single For Loop Perspective }\label{alg:minibatch_sgd_no_momentum}
\begin{algorithmic}[1]
\Require Dataset $S = \{x_i, y_i\}_{i=1}^N$, selection ratio $r\in(0, 1]$, batch size $b$, initial model $w^0$, $X$ rounds, learning rate $\eta_t$
\State $T \gets \lceil N/b \rceil$
\For{iterate $t=1$ to $r\cdot T\cdot X$}
\If{$t\%T == 0$} \Comment{Shuffle after full dataset has been seen}
\State $shuffle(S)$
\EndIf
\State batch $m \gets S[(t-1) \% T \cdot b:t \% T \cdot b]$
\State $w^t \gets w^{t-1} - \frac{\eta_t}{b} \sum _{(x,y) \in m} \nabla f(w^{t-1}; x,y)$ \Comment{$train\_on\_batch$ for mini batch SGD}
\EndFor
\Return $w^t$
\end{algorithmic}
\end{algorithm}

\section{RS2 Convergence Rate}
\label{app:convergence}
Performance of accelerated mini batch SGD has been well studied for convex functions~\cite{lan2012optimal,dekel2012optimal,cotter2011better}. It has been shown that mini batch SGD using batch of size $b$, after $X$ rounds with $T$ batches per round returns a solution $w$ satisfying
\begin{equation}
    \mathbb{E} [l(w) - l(w^*)] \leq \mathcal{O} \left( \frac{L || w^0 - w^*||^2}{T^2 X^2}  + \frac{\sigma || w^0 - w^* ||}{\sqrt{b T X}}\right). 
\end{equation}
Furthermore,~\cite{ghadimi2016accelerated} have analyzed the convergence rate of mini batch SGD for nonconvex $\beta$-smooth functions. After $TX$ mini batch steps of size $b$ the algorithm guarantees a solution $w$ such that
\begin{equation}
\label{eq:convex_statement}
    \mathbb{E} || \nabla l(w) ||^2 \leq \mathcal{O} \left( \frac{\beta (l(w^0) - l(w^*))}{TX} + \frac{\sigma \sqrt{\beta(l(w^0) - l(w^*))}}{\sqrt{bTX}} \right).
\end{equation}

We provide convergence analysis for accelerated mini batch SGD using RS2 without replacement as shown in Algorithm~\ref{alg:minibatch_sgd} following the analysis of ~\cite{ghadimi2016accelerated}.

\begin{cor}
\label{thm:convergence_app}
Suppose the loss $l(w)$ is nonconvex, has $\beta$-Lipschitz continuous gradients, and is bounded below. Let $g (w, \xi_t)$ at step $t$ represent the gradient estimate used when updating the model as in Algorithm~\ref{alg:minibatch_sgd} in the Appendix. Assume the gradient estimate satisfies $\mathbb{E} \left[ || g (w, \xi_t) - \nabla l (w)  ||^2 \right ] \leq \sigma ^2$, and $\mathbb{E} [g(w, \xi_t)] = \nabla l(w)$, where $\xi_t$ are random vectors whose distributions are supported on $\Xi_t \in \mathbb{R}^d$. 
With the previous assumptions, using a selection ratio $r \in (0,1]$ and mini batch of size $b$, RS2 produces an iterate $w$ after $X$ rounds, with $rT$ batches per round, such that:
\begin{equation}
\small
    \mathbb{E} \left[|| \nabla l(w) ||^2 \right ] \leq \mathcal{O} \left( \frac{\beta (l(w^0) - l(w^*))}{r \cdot T \cdot X} + \frac{\sigma\sqrt{\beta(l(w^0) - l(w^*))}}{\sqrt{b \cdot r \cdot T \cdot X}} \right).
\end{equation}
Furthermore, assuming that $l(w)$ is convex it holds that
\begin{equation}
    \mathbb{E} [l(w) - l(w^*)] \leq \mathcal{O} \left( \frac{\beta || w^0 - w^*||^2}{r^2 \cdot T^2 \cdot  X^2}  + \frac{\sigma || w^0 - w^* ||}{\sqrt{b \cdot r \cdot T \cdot X}}\right).
\end{equation}
\end{cor}

\begin{proof}
Each round we use Nesterov's accelerated method to update the gradient:
\begin{equation}
    w^t_{md} \gets (1 - \alpha_t) w^{t-1}_{ag} + \alpha_t w^{t-1}
\label{eq:2.2}
\end{equation}
\begin{equation}
    w^t \gets w^{t-1} - \lambda_t g(w^t_{md}, \xi_t; m)
\label{eq:3.2}
\end{equation}
\begin{equation}    
    w^t_{ag} \gets w^t_{md} - \beta_t g(w^t_{md}, \xi_t; m)
\label{eq:3.3}
\end{equation}
where $g(w^t_{md}, \xi_t)$ represents the gradient on a batch of data $m$. We assume that the following holds:
\begin{equation}
    \mathbb{E} g(w, \xi_t) = \nabla l (w)
\label{eq:bias_gradient}
\end{equation}
\begin{equation}
    \mathbb{E} || g(w, \xi_t) - \nabla l(w)||^2 = \sigma ^2,
\label{eq:variance_gradient}
\end{equation}
where $\xi_t$ are random vectors whose distributions are supported on $\Xi_t \in \mathbb{R}^d$; These are the source of randomness when estimating the full data gradient.

We repeat the procedure for $X$ rounds. When using full data each round we have $T$ batches, resulting in total $TX$ gradient updates. If we perform RS2 w/o replacement. each round we will contain $rT$ updates per round, resulting in a total of $rTX$ iterations.
Assume a relaxation for the learning rate parameters. For this part of the proof assume they are chosen $\{\alpha_t\} \: \text{s.t.} \: \alpha_1 = 1, \alpha_t \in (0,1), \forall t \geq 2$, $\{\beta_t > 0\}$, $\{\lambda_t > 0\}$, such that the following holds:
\begin{equation}
  \Gamma^t =
    \begin{cases}
      1 & t=1\\
      (1-\alpha_t)\Gamma^{t-1} & t\geq 2
    \end{cases}
\label{eq:def_gamma}
\end{equation}
\begin{equation}
    C^t := 1 - \beta\lambda_t - \frac{\beta (\lambda_t - \beta_t)^2}{2 \alpha_t \Gamma^t \lambda_t} \left( \sum _{\tau=t} ^{rTX} \Gamma^\tau \right) > 0
\end{equation}

\begin{equation}
    p_t = \frac{\lambda_t C^t}{\sum _{t=1} ^{rTX} \lambda_t C^t}, \: t=1,...,rTX.
\end{equation}
Furthermore, let $R$ represent an index chosen randomly in all the iterate updates from 1 to $rTX$, chosen such that $Prob\{R=t\} = p_t$.

First we want to show the following holds:
\begin{equation}
    \mathbb{E} || \nabla l (w^R _{md}) ||^2 \leq \frac{1}{\sum _{t=1} ^{rTX} \lambda_t C^t} \left[ l(w^0) - l(w^*) + \frac{\beta\sigma^2}{2} \sum _{t=1} ^{rTX} \lambda_t ^2 \left(  1 + \frac{(\lambda_t - \beta_t)^2}{\alpha_t \Gamma^t \lambda^2 _t} \sum _{\tau=t} ^{rTX} \Gamma^\tau\right) \right].
\label{eq:theorem3}
\end{equation}
Let us define, for ease of writing, the following: $\delta_t := g(w^t_{md}, \xi_t) - \nabla l(w^t _{md})$ and $\Delta^t := \nabla l(w^{t-1}) - \nabla l(w^t _{md})$. 
Since $l(w)$ is bounded from below and a differentiable nonconvex $\beta$-smooth function it holds that (see~\cite{nesterov2003introductory}):
\begin{equation}
    |l(y) - l(x) - \langle \nabla l(x), y-x \rangle| \leq \frac{\beta}{2} || y - x || ^2 \quad \forall x, y \in \mathbb{R}^n.
\label{eq:2.1}
\end{equation}
We start from the assumption that the loss function $l$ is $\beta$-smooth:
\begin{align}
    l(w^t) &\leq l(w^{t-1}) +  \langle \nabla l (w^{t-1}),  w^t - w^{t-1} \rangle + \frac{\beta}{2} || w^t - w^{t-1} ||.
\end{align}
Then, using the update step \cref{eq:3.2} and the definitions of $\delta_t, \Delta^t$:
\begin{align}
    l(w^t) &\leq  l(w^{t-1}) + \langle \Delta^t + \nabla l(w^t_{md}), -\lambda_t [\nabla l(w^t_{md}) + \delta^t] \rangle + \frac{\beta\lambda_t ^2}{2} || \nabla l(w^t_{md}) + \delta^t||^2 \nonumber \\
    &= l(w^{t-1}) + \langle \Delta^t + \nabla l(w^t_{md}), -\lambda_t \nabla l(w^t_{md}) \rangle - \lambda_t \langle \nabla l(w^{t-1}), \delta^t \rangle + \frac{\beta\lambda_t ^2}{2} || \nabla l(w^t_{md}) + \delta^t||^2.
\label{eq:starting_l}
\end{align}
Now using the inequality \cref{eq:2.1} we get: 
\begin{align}
    l(w^t) &\leq l(w^{t-1}) - \lambda_t \left( 1 - \frac{\beta\delta_t}{2} \right) || \nabla l(w^t_{md}) || ^2 + \lambda_t || \Delta^t || || \nabla l(w^t_{md}) || + \frac{\beta\lambda^2_t}{2} ||\delta^t||^2 \nonumber \\
    & \hspace{3mm} - \lambda_t \langle \nabla l(w^{t-1}) - \beta \lambda_t \nabla l (w^t_{md}), \delta^t\rangle.
\end{align}
Since $l$ is $\beta$-smooth and by the update rule \cref{eq:2.2} we have:
\begin{equation}
    ||\Delta^t|| = ||\nabla l(w^{t-1}) - \nabla l(w^t_{md})|| \leq \beta || w^{t-1} - w^t_{md}|| = \beta (1-\alpha_t) || w^t_{ag} - w^{t-1}||.
\label{eq:delta_trick}
\end{equation}

Continuing from \cref{eq:starting_l} and inserting \cref{eq:delta_trick}:
\begin{align}
    l(w^t) &\leq l(w^{t-1}) - \lambda_t \left( 1 - \frac{\beta\delta_t}{2} \right) || \nabla l(w^t_{md}) || ^2 + \lambda_t \beta (1-\alpha_t) || w^t_{ag} - w^{t-1}|| || \nabla l (w^t_{md}) || \nonumber \\
    &\hspace{3mm}+ \frac{\beta\lambda_t^2}{2}||\delta^t||^2 - \lambda_t \langle \nabla l(w^{t-1}) - \beta \lambda_t \nabla l (w^t_{md}), \delta^t\rangle.
\end{align}

Using the general fact that $xy \leq \frac{(x^2 + y^2)}{2}$ holds, we bound the previous inequality:
\begin{align}
    l(w^t) &\leq  l(w^{t-1}) - \lambda_t \left( 1 - \beta\lambda_t \right)|| \nabla l(w^t_{md}) || ^2 + \frac{\beta(1-\alpha_t)^2}{2}||w^{t-1}_{ag} - w^{t-1}||^2 + \frac{\beta\lambda_t ^2}{2}||\delta^t||^2 \nonumber \\
    &\hspace{3mm}- \lambda_t \langle \nabla l(w^{t-1}) - \beta \lambda_t \nabla l (w^t_{md}), \delta^t\rangle.
\label{eq:secondly_l}
\end{align}
Now we take a small digression from the main flow of the proof. We want to show that the following inequality holds:
\begin{align}
    ||w^{t-1}_{ag} - w^{t-1}||^2 &\leq \Gamma^{t-1} \sum_{\tau=1}^{t-1} \frac{(\lambda_{\tau} - \beta_{\tau})^2}{\Gamma^{\tau}\alpha_{\tau}}||\nabla l (w^{\tau}_{md}) + \delta^t||^2 \nonumber \\
    &=\Gamma^{t-1} \sum_{\tau=1}^{t-1} \frac{(\lambda_{\tau} - \beta_{\tau})^2}{\Gamma^{\tau}\alpha_{\tau}} \left[ ||\nabla l (w^{\tau}_{md})||^2  + 2\langle  \nabla l (w^{\tau}_{md}), \delta^\tau\rangle + ||\delta^\tau||^2 \right].
\label{eq:diff_params}
\end{align}
We show that in the following way. First, let us combine the update steps \cref{eq:2.2,eq:3.2,eq:3.3}. Performing change of variable we have:
\begin{align}
    w^t_{ag} - w^t &= (1 - \alpha_t)w^{t-1}_{ag} + \alpha_t w^{t-1} - \beta_t \nabla l (w^t _{md}) - [w^{t-1} - \lambda_t \nabla l (w^t _{md})] \nonumber \\
    &= (1 - \alpha_t) (w^{t-1}_{ag} - w^{t-1}) + (\lambda_t - \beta_t) \nabla l (w^t _{md}).
\label{eq:combined_steps}
\end{align}

Following from \cref{eq:combined_steps} and using Lemma 1 stated in~\cite{ghadimi2016accelerated} it is implied that:
\begin{equation}
    w^t_{ag} - w^t = \Gamma^t \sum_{\tau = 1}^t \frac{\lambda_\tau - \beta_\tau}{\Gamma^\tau} \nabla l (w^\tau _{md}).
\end{equation}

Furthermore, we have:
\begin{equation}
    ||w^t_{ag} - w^t ||^2 = \left\| \Gamma^t \sum_{\tau = 1}^t \frac{\lambda_\tau - \beta_\tau}{\Gamma^\tau} \nabla l (w^\tau _{md}) \right\| ^2.
\label{eq:diff_mod}
\end{equation}
From the definition in \cref{eq:def_gamma} we have:
\begin{equation}
    \sum _{\tau = 1} ^t \frac{\alpha_\tau}{\Gamma ^\tau} = \frac{\alpha_1}{\Gamma ^1} + \sum _{\tau = 2} ^t \frac{1}{\Gamma ^\tau} \left( 1 - \frac{\Gamma^\tau}{\Gamma^{\tau-1}} \right) = \frac{1}{\Gamma ^1} + \sum _{\tau = 2} ^t  \left( \frac{1}{\Gamma ^\tau} - \frac{1}{\Gamma^{\tau-1}} \right) = \frac{1}{\Gamma ^t}.
\end{equation}

Inserting that into \cref{eq:diff_mod} we get:
\begin{equation}
    ||w^t_{ag} - w^t ||^2 = \left\| \Gamma^t \sum_{\tau = 1}^t \frac{\alpha_\tau}{\Gamma^\tau} \frac{\lambda_\tau - \beta_\tau}{\alpha^\tau} \nabla l (w^\tau _{md})   \right\| ^2.
\label{eq:jen}
\end{equation} 
Applying Jensen's inequality to \cref{eq:jen} we have:
\begin{equation}
    ||w^t_{ag} - w^t ||^2 \leq  \Gamma^t \sum_{\tau = 1}^t \frac{\alpha_\tau}{\Gamma^\tau} \left\| \frac{\lambda_\tau - \beta_\tau}{\alpha^\tau} \nabla l (w^\tau _{md})   \right\| ^2 
    =  \Gamma^t \sum_{\tau = 1}^t \frac{(\lambda_\tau - \beta_\tau)^2}{\Gamma^\tau \alpha_\tau}  || \nabla l (w^\tau _{md}) ||^2.
\end{equation}
Hence, \cref{eq:diff_params} holds.

Coming back to the main flow of the proof. We combine the previous two inequalities \cref{eq:secondly_l} and \cref{eq:diff_params}. Also, we use the fact that $\Gamma^{t-1}(1-\alpha_t)^2 \leq \Gamma^t$ :
\begin{align}
    l(w^t) \leq &l(w^{t-1})  - \lambda_t \left( 1 - \beta\lambda_t \right)|| \nabla l(w^t_{md}) || ^2 + \frac{\beta\lambda_t ^2}{2} ||\delta^t||^2 - \lambda_t \langle  \nabla l (w^{t-1}) - \beta\lambda_t \nabla l (w^t _{md}), \delta^t\rangle \nonumber \\
    &+ \frac{\beta \Gamma^t}{2} \sum_{\tau=1}^{t}\frac{(\lambda_{\tau} - \beta_{\tau})^2}{\Gamma^\tau \alpha_\tau} \left[ ||\nabla l (w^{\tau}_{md})||^2  + 2\langle  \nabla l (w^{\tau}_{md}), \delta^\tau\rangle + ||\delta^\tau||^2 \right].
\label{eq:for_summation}
\end{align}

Summing up the above inequalities (\cref{eq:for_summation}) up to the rTX iterate, we get:
\begin{align}
    l(w^{rTX}) &\leq l(w^{0}) - \sum _{t=1} ^{rTX}  \lambda_t \left( 1 - \beta\lambda_t \right)|| \nabla l(w^t_{md}) || ^2 - \sum _{t=1} ^{rTX} \lambda_t \langle  \nabla l (w^{t-1}) - \beta\lambda_t \nabla l (w^t _{md}), \delta^t\rangle \nonumber \\
    &\hspace{3mm}+ \sum _{t=1} ^{rTX} \frac{\beta \lambda_t ^2}{2} ||\delta^t||^2 - \frac{\beta}{2} \sum _{t=1} ^{rTX} \Gamma^t \sum _{\tau=1} ^{t} \frac{(\lambda_\tau - \beta_\tau)^2}{\Gamma^\tau \alpha_\tau} \left[ ||\nabla l (w^{\tau}_{md})||^2  + 2\langle  \nabla l (w^{\tau}_{md}), \delta^\tau\rangle + ||\delta^\tau||^2 \right] \nonumber \\
    &= l(w^0) - \sum _{t=1} ^{rTX} \lambda_t C^t || \nabla l (w^t _{md}) ||^2 + \frac{\beta}{2} \sum_{t=1}^{rTX} \lambda_t ^2 \left( 1 + \frac{(\lambda_t - \beta_t)^2}{\alpha_t \Gamma^t \lambda^2 _t} \sum _{\tau=t} ^{rTX} \Gamma^\tau \right) || \delta^t ||^2 - \sum _{t=1} ^{rTX} b_t,
\end{align}
where $b_t = \langle \lambda_t \nabla l (w^{t-1}) - \left[ \beta\lambda_t ^2 +  \frac{\beta(\lambda_t - \beta_t)^2}{ \Gamma^t \alpha_t} \left( \sum _{\tau=t} ^{rTX} \Gamma^\tau \right) \right] \nabla l (w^t _{md}), \delta^t \rangle$. 
Due to the fact that under assumptions 
\cref{eq:bias_gradient,eq:variance_gradient} $\mathbb{E}||\delta^t||^2 \leq \sigma ^2$ and $\{b_t\}$ is a martingale difference, when taking expectation on both sides we obtain:
\begin{equation}
    \sum _{t=1} ^{rTX} \lambda_t C^t \mathbb{E} ||\nabla l (w^t _{md})||^2 \leq l(w^0) - l(w^{rTX}) + \frac{\beta \sigma^2}{2}\sum _{t=1} ^{rTX} \lambda_t^2 \left( 1 + \frac{(\lambda_t - \beta_t)^2}{\alpha_t \Gamma^t \lambda^2 _t} \sum _{\tau=t} ^{{rTX}} \Gamma^\tau \right).
\end{equation}
Using the fact that $l(w^t) \geq l(w^*)$, $\mathbb{E} || \nabla l (w^R _{md}) ||^2 = \frac{\sum_{t=1} ^{rTX} \lambda_t C^t \mathbb{E} || \nabla l (w^t _{md}) ||^2}{\sum _{t=1} ^{rTX} \lambda_t C^t}$, and by dividing both sides by $\sum _{t=1} ^{rTX} \lambda_t C^t$, we obtain:
\begin{equation}
\label{eq:thm3}
    \mathbb{E} || \nabla l (w^R _{md}) ||^2 \leq \frac{1}{\sum _{t=1} ^{rTX} \lambda_t C^t} \left[ l(w^0) - l(w^*) + \frac{\beta\sigma^2}{2} \sum _{t=1} ^T \lambda_t ^2 \left(  1 + \frac{(\lambda_t - \beta_t)^2}{\alpha_t \Gamma^t \lambda^2 _t} \sum _{\tau=t} ^{rTX} \Gamma^\tau\right) \right].
\end{equation}

Hence, we have proven the wanted \cref{eq:theorem3} holds.

For the remainder of the proof for the nonconvex case we specialize the previously obtained result. Let us assume the following:
\begin{equation}
    \alpha_t = \frac{2}{t+1}
\label{eq:alpha}
\end{equation}

\begin{equation}
    \lambda_t \in \left[ \beta_t, \left( 1 + \frac{\alpha_t}{4}\right) \beta_t \right]
\label{eq:lambda}
\end{equation}

\begin{equation}
    \Gamma^t = \frac{2}{t(t+1)}
\label{eq:gamma}
\end{equation}

\begin{equation}
    \beta_t = \min \left\{ \frac{8}{21\beta}, \frac{\Tilde{D}}{\sigma \sqrt{rTX}} \right\} \: \text{for some} \: \Tilde{D} > 0.
\label{eq:beta}
\end{equation}

Now, we want to prove:
\begin{equation}
\label{eq:corrolary3}
    \mathbb{E} ||\nabla l(w^R _{md})||^2 \leq \frac{21\beta (l(w^0) - l(w^*))}{4rTX} + \frac{2\sigma}{\sqrt{rTX}} \left( \frac{l(w^0) - l(w^*)}{\Tilde{D}} + \beta \Tilde{D} \right).
\end{equation}

From definition of \cref{eq:alpha}, \cref{eq:lambda} let us make a claim about $C^t$. For that, from \cref{eq:lambda} we observe $0 \leq \lambda_t - \beta_t \leq \alpha_t \beta_t / 4$. Now we have:
\begin{align}
    C^t &= 1 - \beta \left[ \lambda_t + \frac{(\lambda_t - \beta_t)^2}{2 \alpha_t \Gamma^t \lambda_t} \left( \sum_{\tau=t} ^{rTX} \Gamma^\tau \right) \right] \\
    &\geq 1 - \beta \left[ \left( 1 + \frac{\alpha_t}{4} \right) \beta_t + \frac{\alpha_t^2 \beta_t^2}{16} \frac{1}{t \alpha_t \Gamma^t \beta_t} \right] \\
    &= 1 - \beta_t \beta (1 + \frac{\alpha_t}{4} + \frac{1}{16}) \\
    &\geq  1 - \beta_t \beta \frac{21}{16}.
\label{eq:ct_ineq}
\end{align}
Multiplied by $\lambda_t$ we have $\lambda_t C^t \geq \frac{11\beta_t}{32}$.

Now we make the following claim about $\Gamma^t$. From \cref{eq:gamma}:
\begin{equation}
    \sum _{\tau = t} ^{rTX} \Gamma^\tau = \sum _{\tau = t} ^{rTX} \frac{2}{\tau (\tau+1)} = 2 \sum _{\tau = t} ^{rTX} \left( \frac{1}{\tau} - \frac{1}{\tau + 1} \right) \leq \frac{2}{t}.
\label{eq:gamma_sum}
\end{equation}

From the \cref{eq:lambda}, \cref{eq:beta}, \cref{eq:ct_ineq},  we have:
\begin{equation}
    C^t \geq 1 - \frac{21}{16}\beta\beta_t \geq \frac{1}{2} > 0 \quad \text{and} \quad \lambda_t C^t \geq \frac{\beta_t}{2}.
\end{equation}
Furthermore, from \cref{eq:lambda}, \cref{eq:gamma}, \cref{eq:beta}, and \cref{eq:gamma_sum}, we obtain:
\begin{align}
    \lambda_t ^2 \left[ 1 + \frac{(\lambda_t - \beta_t)^2}{\alpha_t \Gamma^t \lambda_t ^2} \left( \sum _{\tau=t} ^{rTX} \Gamma^\tau \right) \right] &\leq \lambda_t ^2 \left[ 1 + \frac{1}{\alpha_t \Gamma^t \lambda_t ^2} \left( \frac{\alpha_t \beta_t}{4} \right)^2  \frac{2}{t}\right] = \lambda_t^2 + \frac{\beta_t^2}{8} \nonumber \\
    &\leq \left[ \left( 1 + \frac{\alpha_t}{4}  \right)^2 + \frac{1}{8} \right] \beta_t^2 \leq 2 \beta_t^2.
\end{align}

Together with \cref{eq:thm3} it holds that:
\begin{align}
    \mathbb{E}||\nabla l (w^R _{md})||^2 &\leq \frac{2}{\sum _{t=1} ^{rTX} \beta_t} \left( l(w^0) - l(w^*) + \beta\sigma^2\sum _{t=1} ^{rTX} \beta_t ^2 \right) \nonumber \\
    &\leq \frac{2 (l(w^0) - l(w^*))}{rTX \beta_1} + 2\beta\sigma^2 \beta_1 \nonumber \\
    &\leq \frac{2 (l(w^0) - l(w^*))}{rTX} \left\{ \frac{21\beta}{8} + \frac{\sigma \sqrt{rTX}}{\Tilde{D}} \right\} + \frac{2\beta\Tilde{D}\sigma}{\sqrt{rTX}},
\end{align}
which implies:
\begin{equation}
\label{eq:cor3}
    \mathbb{E} ||\nabla l(w^R _{md})||^2 \leq \frac{21\beta (l(w^0) - l(w^*))}{4rTX} + \frac{2\sigma}{\sqrt{rTX}} \left( \frac{l(w^0) - l(w^*)}{\Tilde{D}} + \beta \Tilde{D} \right).
\end{equation}
Hence, we have shown that \cref{eq:cor3} holds. Continuing from that,
minimizing \cref{eq:cor3} with respect to $\Tilde{D}$, the optimal choice is $\Tilde{D} = \sqrt{\frac{l(w^0_{ag}) - l(w^*)}{\beta}}$. Inserting that value for $\Tilde{D}$, \cref{eq:cor3} becomes:
 \begin{equation}
      \mathbb{E} ||\nabla l(w^R _{md})||^2 \leq \frac{21\beta (l(w^0) - l(w^*))}{4rTX} + \frac{4\sigma \sqrt{\beta(l(w^0) - l(w^*)))}}{\sqrt{rTX}}.
 \end{equation}
 Until now we have assumed that $\mathbb{E} || g(w, \xi_t) - \nabla l(w)||^2 = \sigma ^2$ for the ease of the proof. However, if we assume that the gradient is calculated on a batch of size $b$, the variance of the stochastic gradient reduces to $\sigma^2/b$ (see~\cite{wang2019stochastic}). The entire previous results follow with that assumption without loss of generality. Therefore we conclude it holds that:
 \begin{equation}
    \mathbb{E} || \nabla l(w) ||^2 \leq \mathcal{O} \left( \frac{\beta (l(w^0) - l(w^*))}{rTX} + \frac{\sigma\sqrt{\beta(l(w^0) - l(w^*))}}{\sqrt{brTX}} \right).  
\end{equation}
\paragraph{Convex case}
Now, let us consider the case for convex functions.
First, in order to prove \cref{eq:convex_statement}, we want to show that, assuming:
\begin{equation}
    \alpha_t \lambda_t \leq \beta \beta_t^2, \quad \beta_t < \frac{1}{\beta},
\label{eq:3.6}
\end{equation}
\begin{align}
    p_t = \frac{\frac{1}{\Gamma^t} \beta_t (1 - \beta\beta_t)}{\sum _{t=1}^{rTX} \frac{1}{\Gamma^t} \beta_t (1 - \beta\beta_t)},
\label{eq:3.7}
\end{align}
and
\begin{equation}
    \frac{\alpha_1}{\lambda_1 \Gamma^1} \geq \frac{\alpha_2}{\lambda_2 \Gamma^2} \geq \ldots
\label{eq:2.10}
\end{equation}
the following holds:
\begin{equation}
    \mathbb{E}[l(w^R _{ag} - l(w^*))] \leq \frac{\sum_{t=1}^{rTX} \beta_t(1 - \beta\beta_t) \left[ (2\lambda_1)^{-1} ||w^0 - w^* ||^2 + \beta\sigma^2 \sum _{j=1}^t \frac{\beta_j^2}{\Gamma^j} \right]}{\sum_{t=1} ^{rTX} \frac{\beta_t}{\Gamma^t} (1 - \beta\beta_t)}.
\label{eq:3.9_first}
\end{equation}

Starting from the update rule \cref{eq:2.2} and using the convexity of $l(\cdot)$
we have:
\begin{align}
    l(w^t_{md}) - [(1-\alpha_t)l(w^{t-1}_{ag}) + \alpha_t l(w)] &= \alpha_t \left[ l(w^t_{md}) - l(w) \right] + (1-\alpha_t)\left[ l(w^t_{md} - l(w^{t-1}_{ag})) \right] \nonumber \\
    &\leq \alpha_t \langle \nabla l(w^t_{md}), w^t_{md} - w \rangle + (1-\alpha_t) \langle  \nabla l(w^t_{md}), w^t_{md} - w^{t-1}_{ag} \rangle \nonumber \\
    &= \langle \nabla l(w^t_{md}), \alpha_t (w^t_{md} - w) + (1-\alpha_t)(w^t_{md} - w^{t-1}_{ag})  \rangle \nonumber \\
    &= \alpha_t \langle \nabla l(w^t_{md}), w^{t-1} - w\rangle.
    \label{eq:2.21}
\end{align}

Similar to before, we now start with the smoothness \cref{eq:2.1} and use the update step \cref{eq:3.3} to obtain:
\begin{align}
    l(w^t_{ag}) &\leq l(w^t_{md}) + \langle \nabla l (w^t _{md}), w^t _{ag} -  w^t _{md}\rangle + \frac{\beta}{2}||w^t _{ag} - w^t _{md}||^2 \nonumber \\
    &= l(w^t _{md}) - \beta_t || \nabla l (w^t _{md})||^2 + \beta_t \langle \nabla l (w^t _{md}), \delta^t \rangle + \frac{\beta \beta_t^2}{2} || \nabla l (w^t _{md}) + \delta^t||^2.
\end{align}
Inserting \cref{eq:2.21} into the previous inequality, we have:
\begin{align}
     l(w^t_{ag}) &\leq (1 - \alpha_t)l(w^{t-1} _{ag}) + \alpha_t l(w) + \alpha_t \langle \nabla l (w^t _{md}), w^{t-1} - w \rangle \nonumber \\
    &\hspace{3mm}- \beta_t ||\nabla l (w^t _{md})||^2 + \beta_t \langle \nabla l (w^t _{md}), \delta^t \rangle + \frac{\beta \beta_t^2}{2} || \nabla l (w^t _{md}) + \delta^t ||.
\label{eq:for_combination1}
\end{align}
    
From \cref{eq:3.3} we have:
\begin{align}
    || w^{t-1} - w ||^2 &- 2\lambda_t \langle \nabla l (w^t_{md}) + \delta^t, w^{t-1} - w \rangle \nonumber \\
    &+ \lambda_t^2 || \nabla l (w^t_{md}) + \delta^t ||^2 = || w^{t-1} - \lambda_t (\nabla l(w^t_{md}) + \delta^t) - w ||^2 = || w^{t} - w  ||^2.
\end{align}

From the previous equation, we have:
\begin{equation}
\label{eq:for_combination2}
    \alpha_t \langle \nabla l (w^t _{md}) + \delta^t, w^{t-1} - w \rangle = \frac{\alpha_t}{2\lambda_t} \left[ ||w^{t-1} - w||^2 - || w^t - w ||^2 \right] + \frac{\alpha_t \lambda_t}{2} || \nabla l (w^t _{md}) + \delta^t ||^2.
\end{equation}
Combining \cref{eq:for_combination1,eq:for_combination2} and the fact that $|| \nabla l(w^t_{md}) + \delta^t||^2 = || \nabla l(w^t_{md}) ||^2 + || \delta^t ||^2 +2\langle \nabla l(w^t_{md}) , \delta^t \rangle$, we get:
\begin{align}
    l (w^t _{ag}) &\leq (1 - \alpha_t) l(w^{t-1} _{ag}) + \alpha_t l(w) + \frac{\alpha_t}{2\lambda_t} \left[ ||w^{t-1} - w||^2 - || w^t - w ||^2 \right] \nonumber \\
    &- \beta_t \left( 1- \frac{\beta \beta_t}{2} - \frac{\alpha_t\lambda_t}{2\beta_t} \right) || \nabla l (w^t _{md}) ||^2
    + \left( \frac{\beta \beta_t^2 + \alpha_t\lambda_t}{2} \right) || \delta^t ||^2 \nonumber \\
    &+ \langle \delta^t, (\beta_t + \beta \beta_t^2 + \alpha_t\lambda_t)\nabla l (w^t _{md}) + \alpha_t(w - w^{t-1}) \rangle.
\label{eq:previous_ineq}
\end{align}

Due to the fact that $\alpha_1 = \Gamma^1 = 1$ and by \cref{eq:2.10} it holds that:

\begin{equation}
    \sum _{t=1}^{rTX} \frac{\alpha_t}{\lambda_t \Gamma^t} \left[ || w^{t-1} - w ||^2 - || w^t - w ||^2 \right] \leq \frac{\alpha_1 || w^0 - w ||^2}{\lambda_1 \Gamma^1} = \frac{|| w^0 - w ||^2}{\lambda_1}.
    \label{eq:2.25}
\end{equation}

Using Lemma 1 from~\cite{ghadimi2016accelerated}, \cref{eq:2.25} and subtracting $l(w)$ from \cref{eq:previous_ineq}, we obtain:
\begin{align}
    \frac{l(w^{rTX}_{ag}) - l(w)}{\Gamma^{rTX}} &\leq \frac{||w^0 - w||^2}{2\lambda_1} - \sum _{t=1} ^{rTX} \frac{\beta_t}{2 \Gamma^t} \left( 2 - \beta \beta_t - \frac{\alpha_t \lambda_t}{\beta_t} \right) ||\nabla l(w^t _{md})||^2 \nonumber \\
    &+ \sum _{t=1} ^{rTX} \left( \frac{\beta \beta_t^2 + \alpha_t \lambda_t}{2 \Gamma^t} \right) || \delta^t ||^2 + \sum _{t=1} ^{rTX} b_t', 
\end{align}
where $b_t' = \frac{1}{\Gamma^t}\langle \delta^t, (\beta_t + \beta \beta_t^2 + \alpha_t \lambda_t) \nabla l(w^t _{md}) + \alpha_t (w - w^{t-1}) \rangle$. Together with \cref{eq:3.6} the above inequality gives:
\begin{align}
    \frac{l(w^{rTX}_{ag}) - l(w)}{\Gamma^{rTX}} &\leq \frac{||w^0 - w||^2}{2\lambda_1}  - \sum _{t=1} ^{rTX} \frac{\beta_t}{\Gamma^t} (1- \beta \beta_t) ||\nabla l(w^t _{md})||^2 \nonumber \\
    &+ \sum _{t=1} ^{rTX} \frac{\beta\beta_t^2}{\Gamma^t} || \delta^t ||^2 + \sum _{t=1} ^{rTX} b_t'.
\end{align}
Since $\{b_t'\}$ is a martingale difference, $\mathbb{E}||\delta^t||^2 \leq \sigma^2$ and by taking expectation with respect to $\xi_{[rTX]}$, we have:
\begin{equation}
    \frac{1}{\Gamma^{rTX}} \mathbb{E} \left[ l(w^{rTX} _{ag}) - l(w) \right] \leq \frac{||w^0 - w||^2}{2 \lambda_1} - \sum _{t=1} ^{rTX} \frac{\beta_t}{\Gamma^t} (1 - \beta\beta_t) \mathbb{E} ||\nabla l (w^t _{md})||^2 + \sigma^2 \sum _{t=1} ^{rTX} \frac{\beta \beta_t ^2}{\Gamma ^t}.
\label{eq:3.11}
\end{equation}
Now, assume $w=w^*$ and since by definition $l(w^{rTX}_{ag}) \geq l(w^*)$, we obtain:
\begin{equation}
\label{eq:thm3_end}
    \sum _{t=1} ^{rTX} \frac{\beta_t}{\Gamma^t} (1 - \beta\beta_t) \mathbb{E} ||\nabla l (w^t _{md})||^2 \leq \frac{||w^0 - w^*||^2}{2 \lambda_1} + \sigma^2 \sum _{t=1} ^{rTX} \frac{\beta \beta_t ^2}{\Gamma ^t},
\end{equation}
from which, using the definition of $w^R_{md}$, it follows that
\begin{equation}
    \mathbb{E} || \nabla l(w^R _{md}) ||^2 \leq \frac{(2\lambda_1)^{-1} ||w^0 - w^*||^2 + \beta\sigma^2 \sum_{t=1} ^{rTX} \frac{\beta_t^2}{\Gamma^t}}{\sum_{t=1} ^{rTX} \frac{\beta_t}{\Gamma^t} (1 - \beta \beta_t)}.
\label{eq:3.8}
\end{equation}

Also, using \cref{eq:3.6} and \cref{eq:3.11} in \cref{eq:thm3_end}, for $rTX\geq1$ we have:
\begin{equation*}
    \mathbb{E} \left[ l(w^{rTX} _{ag} - l(w^*)) \right] \leq \Gamma^{rTX} \left( \frac{|| w^0 - w ||^2}{2\lambda_1} + \sigma^2 \sum _{t=1} ^{rTX} \frac{\beta \beta_t ^2}{\Gamma ^t}\right),
\end{equation*}
which implies that \cref{eq:3.9_first} holds:
\begin{align}
    \mathbb{E}[l(w^R _{ag} - l(w^*))] &= \sum _{t=1} ^{rTX} \frac{\frac{\beta_t}{\Gamma^t} (1 - \beta\beta_t)}{\sum _{t=1} ^{rTX} \frac{\beta_t}{\Gamma^t} (1 - \beta\beta_t)} \mathbb{E} [l(w^t _{ag} - l(w^*))] \nonumber \\
    &\leq \frac{\sum_{t=1}^{rTX} \beta_t(1 - \beta\beta_t) \left[ (2\lambda_1)^{-1} ||w^0 - w ||^2 + \beta\sigma^2 \sum _{j=1}^t \frac{\beta_j^2}{\Gamma^j} \right]}{\sum_{t=1} ^{rTX} \frac{\beta_t}{\Gamma^t} (1 - \beta \beta_t)}.
\label{eq:3.9}
\end{align}

Now, assuming $\alpha_t$ is set as in \cref{eq:alpha}, $p_t$ is set as in \cref{eq:3.7},
\begin{equation}
    \beta_t = \min \left\{ \frac{1}{2\beta}, \left( \frac{\Tilde{D}^2}{\beta^2\sigma^2 (rTX)^3} \right) ^{1/4} \right\},
\label{eq:3.14}
\end{equation}
and
\begin{equation}
    \lambda_t = \frac{t \beta \beta_t^2}{2},
\label{eq:3.15}
\end{equation}

we want to show that the following inequality holds:
 \begin{equation}  
     \mathbb{E} [l(w^R _{ag}) - l(w^*)] \leq \frac{48 \beta || w^0 -w^* ||^2}{r^2T^2X^2} + \frac{12\sigma}{\sqrt{rTX}} \left( \frac{|| w^0 -w^* ||^2}{\Tilde{D}} + \Tilde{D} \right),
 \label{eq:cor3_convex}
 \end{equation}

for some $\Tilde{D} > 0$.
Note, that \cref{eq:3.14} and \cref{eq:3.15} imply \cref{eq:2.10} and \cref{eq:3.6}.
 
Since, $\Gamma^t = \frac{2}{t(t+1)}$, and by \cref{eq:3.14}, we obtain:
\begin{align}
    \sum _{t=1} ^{rTX} \frac{\beta_t}{\Gamma^t} (1 - \beta \beta_t) &\geq \frac{1}{2} \sum _{t=1} ^{rTX} \frac{\beta_t}{\Gamma^t} = \frac{\beta_1}{2} \sum _{t=1} ^{rTX} \frac{1}{\Gamma^t}  \label{eq:3.18} \\
    \sum _{t=1} ^{rTX} \frac{1}{\Gamma^t} &\geq \sum _{t=1} ^{rTX} \frac{t^2}{2} = \frac{1}{12} rTX(rTX+1)(2rTX + 1) \geq \frac{1}{6} r^3T^3X^3.
\label{eq:3.19}
\end{align}

By $\Gamma^t = \frac{2}{t(t+1)}$, \cref{eq:3.8,eq:3.14,eq:3.15}, we obtain:
\begin{align}
    \mathbb{E} || \nabla l (w^R _{md}) ||^2 &\leq \frac{2}{\beta_1 \sum _{t=1} ^{rTX} \frac{1}{\Gamma^t}} \left( \frac{||w^0 - w^* ||^2}{\beta \beta_1^2} + \beta\sigma^2 \beta_1^2 \sum _{t=1} ^{rTX} \frac{1}{\Gamma^t} \right) \nonumber \\
    &= \frac{2 ||w^0 - w^* ||^2 }{\beta \beta_1^3 \sum _{t=1} ^{rTX} \frac{1}{\Gamma^t}} + 2\beta \sigma^2 \beta_1 \leq \frac{12 ||w^0 - w^* ||^2}{\beta r^3T^3X^3 \beta_1^3} + 2\beta\sigma^2\beta_1 \nonumber \\
    &\leq \frac{96 \beta^2 ||w^0 - w^* ||^2}{r^3 T^3 X^3} + \frac{\beta^{1/2\sigma^{3/2}}}{(rTX)^{3/4}} \left( \frac{12 ||w^0 - w^* ||^2}{\Tilde{D}^{3/2}} + 2\Tilde{D}^{1/2} \right).
\end{align}
 Moreover, by \cref{eq:3.14}, it holds that:
 \begin{equation*}
     1 - \beta\beta_t \leq 1 \quad \text{and} \quad \sum _{j=1} ^t \frac{1}{\Gamma^j} = \frac{1}{2} \sum _{j=1} ^t  j(j+1) \leq \sum _{j=1} ^t j^2 \leq t^3.
 \end{equation*}

 It is implied by \cref{eq:3.9,eq:3.14,eq:3.18,eq:3.19} that:
 \begin{align}
     \mathbb{E} [l(w^R _{ag}) - l(w^*)] &\leq \frac{2}{\sum _{t=1} ^{rTX} \frac{1}{\Gamma^t}} \left[ rTX(2\lambda_1)^{-1} ||w^0 - w^* ||^2 + \beta\sigma^2 \beta_1^2 \sum_{t=1}^{rTX} t^3 \right] \nonumber \\
     &\leq \frac{12||w^0 - w^* ||^2 }{r^2T^2 X^2 \beta \beta_1 ^2} + \frac{12 \beta \sigma^2 \beta_1 ^2}{r^3T^3 X^3}\sum_{t=1}^{rTX} t^3 \nonumber \\
     &\leq \frac{12||w^0 - w^* ||^2 }{r^2T^2X^2 \beta \beta_1 ^2} + 12\beta\sigma^2 \beta_1^2 rTX \nonumber \\
     &\leq \frac{48 \beta || w^0 -w^* ||^2}{r^2T^2X^2} + \frac{12\sigma}{\sqrt{rTX}} \left( \frac{|| w^0 -w^* ||^2}{\Tilde{D}} + \Tilde{D} \right).
 \end{align}

  This shows that \cref{eq:cor3_convex} holds. Minimizing the previous inequality with respect to $\Tilde{D}$, the optimal choice is $\Tilde{D} = ||w^0 - w^* ||$. Hence, it becomes:
 \begin{equation}
      \mathbb{E} [l(w^R _{ag}) - l(w^*)] \leq \frac{48 \beta || w^0 -w^* ||^2}{r^2T^2X^2} + \frac{24|| w^0 -w^* || \sigma}{\sqrt{rTX}}.
 \end{equation}
 As in~\cite{wang2019stochastic}, the variance of the stochastic gradient reduces to $\sigma^2/b$ when estimating with $b$ samples. Therefore we conclude it holds that:
 \begin{equation}
    \mathbb{E} [l(w) - l(w^*)] \leq \mathcal{O} \left( \frac{\beta || w^0 - w^*||^2}{r^2T^2X^2}  + \frac{\sigma || w^0 - w^* ||}{\sqrt{b r TX}}\right).  
\end{equation}
 
\end{proof}

\section{RS2 Generalization Error}
\label{app:generalization_error} 
We proceed with the generalization error bound of RS2 for nonconvex Lipschitz and smooth losses. We start by introducing the assumptions on the function $f:\mathbb{R}^d \times \mathcal{Z}\rightarrow \mathbb{R}^+$ (see Section \ref{sec:theory}) for completeness, and then we proceed with the proof of Theorem \ref{thm:generalization}. 

\paragraph{Assumption. (Smooth and Lipschitz Loss)}\label{ass:smooth}\! There exist constants $\beta_f\geq 0$ and $L_f\geq 0$, such that for all $w,u\in\mbb{R}^d$ and $x\in \mc{X}$, it is true that $\Vert  \nabla_w f( w, z) - \nabla_u  f( u, z) \Vert_2 \leq \beta_f \Vert w- u \Vert_2 $ and $\Vert  f( w, z) -  f( u, z) \Vert_2 \leq L_f \Vert w- u \Vert_2 $.

The next result follows form prior work~\cite{nikolakakis2023select} and it suffices to show that RS2 sampling is data independent and belongs to the set of general mini batch schedules that appear in~\cite[Definition 1]{nikolakakis2023select}. 
\begin{theorem}[Generalization error of standard gradient RS2,~\cite{nikolakakis2023select} Theorem 8]\label{thm:generalization_app}
    Let the function $f$ be nonconvex, $L_f$-Lipschitz and $\beta_f$-smooth. Then the generalization error of the standard gradient RS2 algorithm with a decreasing step-size $\eta_t\leq C/t$ (for $C<1/\beta_f$), is bounded as: 
\begin{align}
\small
      |\gen (f, \mc{D} , \mathrm{RS2} ) |\leq \frac{1}{N} \cdot 2 C e^{C\beta_f} L_f^2 {(r \cdot T \cdot X)}^{C\beta_f }  \min\Big\{  1+ \frac{1}{C \beta_f} , \log(e\cdot r\cdot T\cdot X) \Big\}.\label{eq:gen_bound_app}
\end{align}
\end{theorem} 
\begin{proof}
Let $\{k^j_1,\ldots, k^j_b   \}\subset \{1,2,\ldots,N\}^{b}$ be the set of indices for mini batch selection at each gradient step $j\in\{1,\ldots , rTX \}$. We select the mini batch through the choice of indices $\{k^j_1,\ldots, k^j_b   \}$ as follows. For sampling with replacement in line 4 of Algorithm~\ref{alg:general_rs2} at round $j$ selects a subset of indices $\{k_1^j, \ldots, k_{rN}^j\}$. These indices are sampled independently from any other round $i \in \{1, \ldots, X\}$, i.e., the same indices can be sampled in consecutive rounds, hence with replacement. Note, that at round $j$ sampled indices in the set are unique. The parameters are then updated using a deterministic batch schedule iterating through the sampled subset of indices resulting in $rT$ gradient updates. On the contrary, RS2 without replacement can be seen as traversing the full dataset in a deterministic round-Robin fashion. That is, the model parameters are updated by sequentially selecting indices $\{k_1^j, \ldots, k_{b}^j\}$. After iterating over the full dataset, i.e., after $T = N / b$ gradient updates we shuffle the full dataset array and repeat the procedure (e.g., Algorithm~\ref{alg:minibatch_sgd_no_momentum}). The algorithm early stops after $rTX$ gradient updates. Thus the selection rule is non-adaptive and data-independent and it belongs to the set of the general batch schedules~\cite[Definition 1]{nikolakakis2023select}. As a consequence,~\cite[Lemma 2]{nikolakakis2023select} and the growth recursion~\cite[Lemma 3 (Growth Recursion)]{nikolakakis2023select} holds verbatim for RS2 with standard gradient training with batch size $b$. Then, we solve the recursion identically to~\cite[Proof of Theorem 8]{nikolakakis2023select}) for $rTX$ total number of gradient steps. The solution of the recursion gives (the on-average stability and thus) the generalization error bound of RS2, as appears in the theorem, and completes the proof.
\end{proof}

\end{document}